%% file: ODE-minimax.tex
\documentclass[english]{article}
\usepackage[T1]{fontenc}
\usepackage[latin9]{inputenc}
\usepackage{geometry}
\usepackage[unicode=true,
 bookmarks=false,
 breaklinks=false,
 pdfborder={0 0 1},
 colorlinks=false]{hyperref}
\hypersetup{colorlinks,
linkcolor=red,
anchorcolor=blue,
citecolor=blue,
urlcolor=blue}
\usepackage{bm}
\usepackage{amsmath}
\usepackage{mathtools}
\usepackage{amssymb}
\usepackage{amsthm}
\usepackage{comment}
\usepackage{natbib}
\usepackage{booktabs}
\usepackage{graphicx}
\usepackage{algorithm}
\usepackage[noend]{algorithmic}
\usepackage{float}
\usepackage{multirow}
\usepackage{dsfont}
\usepackage{tcolorbox}
\usepackage{tabulary}
\usepackage{colortbl}
\usepackage{color}
\definecolor{gen}{RGB}{0,0,200}
\definecolor{cc}{RGB}{231,117,0}

\allowdisplaybreaks

\newcommand{\ind}{\mathds{1}}
\newcommand{\defn}{\coloneqq}
\newcommand{\defnrev}{\eqqcolon}
\newcommand{\var}{\mathsf{Var}}

\newcommand{\cB}{\mathcal{B}}

\newcommand{\cE}{\mathcal{E}}
\newcommand{\cF}{\mathcal{F}}

\newcommand{\cH}{\mathcal{H}}

\newcommand{\cN}{\mathcal{N}}

\newcommand{\bE}{\mathbb{E}}
\newcommand{\bP}{\mathbb{P}}
\newcommand{\bR}{\mathbb{R}}

\newcommand{\bZ}{\mathbb{Z}}

\newcommand{\KL}{\mathsf{KL}}
\newcommand{\diff}{\,\mathrm{d}}

\newcommand{\numpf}[2]{\overset{(\mathrm{#1})}{#2}}
\newcommand{\ol}{\overline}
\newcommand{\wh}{\widehat}
\newcommand{\wt}{\widetilde}

\newcommand{\sde}{\mathsf{sde}}
\newcommand{\ode}{\mathsf{ode}}

\newcommand{\veps}{\varepsilon}
\newcommand{\vphi}{\varphi}
\newcommand{\TV}{\mathsf{TV}}
\newcommand{\setc}{\mathrm{c}}

\newcommand{\score}{\mathsf{sc}}
\newcommand{\disteq}{\overset{\mathrm d}=}

\newcommand{\jcb}{\mathsf{jcb}}

\newcommand{\gd}{\nabla}

\newcommand{\typ}{\mathsf{typ}}
\newcommand{\mymid}{\mid}
\newcommand{\psd}{\succcurlyeq}

\title{Minimax Optimality of the Probability Flow ODE \\ for Diffusion Models\footnotetext{Corresponding author: Gen Li.}}

\author{Changxiao Cai\thanks{Department of Industrial and Operations Engineering, University of Michigan, Ann Arbor, USA; Email: \href{mailto:cxcai@umich.edu}{cxcai@umich.edu}.}
\and 
Gen Li\thanks{Department of Statistics, The Chinese University of Hong Kong, Hong Kong; Email: \href{genli@cuhk.edu.hk}{genli@cuhk.edu.hk}.}}

\begin{document}

\theoremstyle{plain} 
\newtheorem{lemma}{\bf Lemma} 
\newtheorem{proposition}{\bf Proposition}
\newtheorem{theorem}{\bf Theorem}
\newtheorem{corollary}{\bf Corollary} 
\newtheorem{claim}{\bf Claim}

\theoremstyle{remark}
\newtheorem{assumption}{\bf Assumption} 
\newtheorem{definition}{\bf Definition} 
\newtheorem{condition}{\bf Condition}
\newtheorem{property}{\bf Property} 
\newtheorem{example}{\bf Example}
\newtheorem{fact}{\bf Fact}
\newtheorem{remark}{\bf Remark}

\maketitle

\input{abstract}

\medskip


\tableofcontents{}

\input{intro}

\input{related-work}
\input{notation}
\input{problem}
\input{results}

\input{analysis}
\input{discussion}

\section*{Acknowledgements}

Gen Li is supported in part by the Chinese University of Hong Kong Direct Grant for Research.

\bibliographystyle{apalike}
\bibliography{bibfileDF}

\appendix
\section{Proof of theorems}\label{sec:proof_of_theorem}
\input{proof-convergence}

\input{proof-discretization}
\input{proof-score}
\input{proof-jacobian}

\input{proof-early-stopping}

\section{Proof of lemmas}\label{sec:proof_of_lemmas}
\input{proof-learning-rate}

\input{proof-lemma}
\input{aux-lemma}

\end{document}

%% file: abstract.tex
\begin{abstract}
Score-based diffusion models have become a foundational paradigm for modern generative modeling, demonstrating exceptional capability in generating samples from complex high-dimensional distributions. Despite the dominant adoption of probability flow ODE-based samplers in practice due to their superior sampling efficiency and precision, rigorous statistical guarantees for these methods have remained elusive in the literature. This work develops the first end-to-end theoretical framework for deterministic ODE-based samplers that establishes near-minimax optimal guarantees under mild assumptions on target data distributions. Specifically, focusing on subgaussian distributions with $\beta$-H\"older smooth densities for $\beta\leq 2$, we propose a smooth regularized score estimator that simultaneously controls both the $L^2$ score error and the associated mean Jacobian error. Leveraging this estimator within a refined convergence analysis of the ODE-based sampling process, we demonstrate that the resulting sampler achieves the minimax rate in total variation distance, modulo logarithmic factors. Notably, our theory comprehensively accounts for all sources of error in the sampling process and does not require strong structural conditions such as density lower bounds or Lipschitz/smooth scores on target distributions, thereby covering a broad range of practical data distributions.
\end{abstract}

%% file: intro.tex

\section{Introduction}\label{sec:introduction}

Score-based generative models (SGMs), more commonly known as diffusion models \citep{sohl2015deep,song2019generative}, have rapidly emerged as a cornerstone of contemporary generative modeling---a task seeking to create new samples that mirror the training data in distribution.
Their exceptional performance spans diverse domains, from image synthesis \citep{dhariwal2021diffusion,rombach2022high} and natural language processing \citep{austin2021structured,li2022diffusion} to medical imaging \citep{song2021solving,chung2022score} and protein modeling \citep{trippe2022diffusion,gruver2024protein}.
We refer readers to \citet{yang2023diffusion} for a comprehensive survey of recent methodological advances and practical applications in diffusion models, and to \citet{chen2024overview,tang2024score} for an overview of theoretical advances.

At the core of diffusion models are two complementary processes: a {forward} process that progressively corrupts a sample from the target data distribution with Gaussian noise, and a {reverse} process that aims to reverse the forward process and systematically transforms pure noise back into samples matching the target distribution.
%
The realization of such a reverse process relies critically on accurate estimation of the (Stein) \emph{score function}---the gradient of the logarithm of the density---of the distribution of the forward process. This estimation step, referred to as \emph{score matching} \citep{hyvarinen2005estimation,vincent2011connection,song2019generative}, has proven fundamental to the successes of diffusion models.
Historically, drawing on classical results regarding the existence of reverse-time stochastic differential equations (SDEs) \citep{anderson1982reverse,haussmann1986time}, the reverse processes have often been implemented through the SDE framework, as exemplified by the Denoising Diffusion Probabilistic Model (DDPM) \citep{ho2020denoising}. Yet an alternative formulation based on ordinary differential equations (ODE)---often referred to as the  \emph{probability flow ODEs} \citep{song2020score}---has gained growing prominence for enabling a deterministic sampling process while preserving the same marginal distributions as the reverse SDE. In particular, these ODE-based samplers, including those closely related to the Denoising Diffusion Implicit Models (DDIMs) \citep{song2020denoising}, offer two key compelling advantages:
\begin{itemize}
	\item \emph{Fast sampling speed.} Empirically, ODE-based samplers often require substantially fewer iteration steps than their SDE counterparts to produce high-fidelity samples. By discretizing the probability flow ODE through numerical methods such as forward Euler or exponential integrators \citep{lu2022dpm,zhang2022fast,lu2022dpm+,zhao2024unipc}, they can deliver comparable or superior quality in as few as 10 steps, whereas SDE-based samplers typically need hundreds or even thousands of iterations.
	\item \emph{Deterministic dynamic.} Since each update is a deterministic function of the previous iterate, ODE-based samplers allow for precise control and reproducibility over the sampling path. This critical feature is desirable for image generation and editing tasks that require fine-grained control of the structure and content \citep{meng2021sdedit,cao2023masactrl,huang2024diffusion}.
\end{itemize}
These benefits have made ODE-based samplers the dominant choice in practical generative modeling, outperforming SDE-based approaches in both speed and flexibility \citep{ramesh2022hierarchical,rombach2022high}.

\paragraph*{Existing theoretical understanding and limitations.}
Motivated by the striking empirical performance of diffusion models, a flurry of recent theoretical works have sought to establish rigorous guarantees for their performance. 
Existing theory, however, typically separates the two essential phases of diffusion modeling: (1) score estimation given training data sampled from the target distribution, and (2) sample generation based on these estimates. 
The majority of prior works focuses on the sampling phase, with both SDE-based samplers \citep{,lee2022convergence,chen2022sampling,chen2023improved,benton2023linear,li2024d} and ODE-based samplers \citep{chen2023probability,li2023towards} known to converge polynomially fast toward distributions whose total variation (TV) distance to the target scales proportionally to the score estimation error. 
These results suggest that with accurate score estimates available, diffusion models can, in principle, learn any distribution without relying on restrictive conditions like log-concavity or functional inequalities---requirements that often limit the scope of classical sampling algorithms like Langevin dynamics.

From a statistical viewpoint, diffusion models provide an algorithmic framework for sampling from high-dimensional distributions. Given data collected from a target distribution, they first estimate the score functions along a forward process with carefully chosen learning rates. Subsequently, they use these score estimates to iteratively generate new samples from the target distribution.
However, existing convergence results fail to provide ``end-to-end'' sampling guarantees. They largely assume the availability of accurate score estimations but do not fully investigate whether such estimates can be obtained in practice, let alone characterize the score estimation error. 

This leads to a fundamental question: Given training data drawn from a target distribution, can score-based diffusion models achieve the information-theoretic limit of sampling?

\subsection{End-to-end sampling guarantees}
Several recent works have begun to bridge this gap for the SDE-based samplers (e.g., DDPM), deriving \emph{end-to-end} sampling guarantees under various distributional settings.
Examples include Gaussian mixture models (GMMs) \citep{shah2023learning,chen2024learning,gatmiry2024learning}, subgaussian distributions \citep{wibisono2024optimal,cole2024score,li2024good}, and low-dimensional structures \citep{chen2023score,tang2024adaptivity}, where researchers have developed specialized score estimators and leveraged existing convergence results to derive sample complexity guarantees.

Notably, some studies have established that DDPM can achieve (near)-minimax optimal performance for target distributions with smooth densities.
For instance, the seminal work \citet{oko2023diffusion} showed that DDPM with score estimates learned using empirical risk minimization via neural networks attains minimax-optimal rates in both TV and 1-Wasserstein distances (up to some logarithmic factors), assuming the target density is Besov-smooth, bounded below by a constant, and supported over a compact set.
Under a similar setting (except for H\"older smooth densities),
\citet{dou2024optimal} established the sharp minimax rate of score estimation, achieved by a combination of kernel smoothing and kernel-based regularized score estimators. This further leads to the sharp minimax optimal performance of DDPM, though the approach therein requires knowledge of the density lower bound.
\cite{wibisono2024optimal} removed the bounded density assumption, focusing instead on subgaussian distributions with Lipschitz/$\beta$-H\"older smooth score functions with $\beta \leq 1$. Drawing insights from empirical Bayes methodology, the authors introduced a kernel-based score estimator that achieves the minimax optimal score estimation. By integrating this estimator with the convergence guarantees established in \citet{chen2022sampling}, they derived a sample complexity bound for the entire sampling pipeline.
Further relaxing the Lipschitz score requirement, \citet{zhang2024minimax} demonstrated that DDPM equipped with a kernel-based truncated score estimator obtains the (near-)minimax rate in TV distance for subgaussian distributions with $\beta$-Soblev-smooth densities with $\beta \leq 2$.
\citet{holk2024statistical} investigated stochastic samplers in reflected diffusion models for constrained generative modeling, and established the (near-)minimax optimality in TV distance for Soblev-smooth densities when using a neural network-based score estimator.

In contrast to these remarkable advances for SDE-based samplers, end-to-end statistical guarantees of ODE-based samplers remain lacking in the literature, despite their popularity in practice.
A straightforward attempt might be to simply reuse the existing DDPM-oriented score estimators for ODE-based samplers. Unfortunately, significant obstacles arise in both algorithmic feasibility and theoretical analysis, leaving open the question of whether ODE-based methods can achieve the same level of optimality enjoyed by DDPM.
\begin{itemize}
	\item \emph{Nonsmooth score estimators.}
Existing score estimators typically employ a ``hard thresholding'' procedure, resulting in discontinuous or nonsmooth score estimators \citep{wibisono2024optimal,zhang2024minimax,dou2024optimal}. While these $L^2$-accurate score estimators can yield rate-optimal sampling when coupled with stochastic samplers, their optimality may not extend to ODE-based samplers. 
 Indeed, as observed in \citet{li2024sharp}, an $L^2$ estimation guarantee alone can be insufficient for deterministic samplers to obtain non-trivial TV distance bounds  due to their deterministic update rule (in fact, the TV distance can remain as high as $1$). 
This limitation does not affect stochastic samplers but presents a significant challenge for ODE-based methods.
To overcome this issue for ODE-based samplers, a key requirement for the score estimator is that the estimation error in terms of the corresponding Jacobian matrix be controlled as well.
 This imposes sort of smoothness requirements that the existing score estimators in the literature may violate, potentially undermining the minimax optimality for ODE-based samplers. 
	\item \emph{Convergence analysis hurdles.} On the theoretical side, existing analysis for statistical guarantees of stochastic samplers relies on the Girsanov theorem, which controls the Kullback-Leibler (KL) divergence between the forward process and the backward process with estimated score functions \citep{oko2023diffusion,zhang2024minimax,dou2024optimal,holk2024statistical}. 
However, this approach fails for deterministic ODE-based samplers where the KL divergence can be infinite---rendering Girsanov-based techniques inapplicable and necessitating novel analytical frameworks.
Moreover, the minimax optimal guarantees established therein rest on the idealized assumption that the backward SDEs with estimated scores can be solved accurately. This analysis overlooks the discretization errors introduced by practical sampling implementations, without rigorously justifying whether or not such errors would compromise the minimax rate.
	A comprehensive theoretical framework of score-based samplers must incorporate these discretization effects to accurately capture the end-to-end sampling performance.
\end{itemize}

Therefore, this naturally raises the following question:
\begin{center}
	\emph{Can the probability flow ODE-based sampler achieve the fundamental statistical limit of sampling?}
\end{center}

\subsection{This paper}
In this work, we provide an affirmative answer by developing the first end-to-end theoretical analysis that rigorously establishes minimax-optimal performance guarantees for ODE-based samplers.

Focusing on $d$-dimensional subgaussian target distributions with $\beta$-H\"older smooth densities for $\beta\leq 2$ (see Section~\ref{sub:assumptions} for rigorous definitions), we propose a smooth kernel-based regularized score estimator that simultaneously controls the $L^2$ score error and the associated Jacobian error.
Building upon a refined convergence guarantee that characterizes the discretization error of the ODE-based sampling scheme, we show that the TV distance between the sample distribution and the target distribution---using our proposed smooth score estimator---achieves the rate
\begin{align*}
	n^{-\frac{\beta}{d+2\beta}} \quad \text{(up to log factors)}
\end{align*}
which matches the minimax rate for density estimation (up to some logarithmic factor) \citep{stone1980optimal,stone1982optimal, tsybakov2009introduction}. 
Since the lower bound for density estimation forms a valid lower bound for sampling, our result demonstrates that the probability flow ODE-based samplers can attain (near)-minimax optimal performance. 

Notably, this result holds without restrictive assumptions---such as log-concavity, constant density lower bounds, or Lipschitz scores---on the target distributions. 
Moreover, our convergence analysis directly controls the TV distance, a significant departure from previous DDPM analyses that rely on the Girsanov theorem.
Our work thus validates both the practical effectiveness and fundamental statistical optimality of ODE-based samplers, providing the first rigorous, end-to-end minimax optimal guarantee that simultaneously quantifies both discretization and score errors for diffusion models.

%% file: related-work.tex

\subsection{Other related work}

\paragraph*{Score estimation.}

One of the earliest works on theoretical guarantees for score estimation in diffusion models is by \citet{block2020generative}, which provided an error bound in terms of the Rademacher complexity of the score approximation function class.
Subsequent research has focused on neural network-based approaches under diverse distributional assumptions.
For instance, \citet{chen2023score} provided a sample complexity bound by assuming the support of the target distribution lies in a low-dimensional linear space, and \citet{de2022convergence} examined the generalization error in terms of the Wasserstein distance when the target distribution is supported on a low-dimensional manifold.
Assuming the target is smooth with respect to the volume of the low-dimensional manifold, \citet{tang2024adaptivity} established the minimax optimality under the Wassertein distance.
\citet{cole2024score} explored score estimation for subgaussian target distributions whose log-relative density (with respect to the standard Gaussian measure) can be locally approximated by a neural network.
Focusing on a mixture of two Gaussians, \citet{cui2023analysis} established the score estimation error using a two-layer neural network, and \citet{song2021solving} derived sample complexity guarantees for scores learned via deep neural networks in the context of graphical models.
Beyond using neural networks for score estimation, \citet{scarvelis2023closed} studied smoothing the closed-form score of the empirical distribution of the training data, and \citet{li2024good} investigated the empirical kernel density estimator and analyzed the sample complexity under Gaussian or bounded-support settings. 
Recent works \citep{gatmiry2024learning,chen2024learning} focused on GMMs and introduced a piecewise polynomial regression approach for score estimation, providing end-to-end guarantees for learning GMMs.

\paragraph*{Convergence theory.}
Assuming access to reliable score estimates, early convergence results for SDE-based samplers (with respect to TV distance) 
either replied on $L^\infty$-accurate estimates \citep{de2021diffusion,albergo2023stochastic}, or exhibited exponential dependence \citep{de2022convergence,block2020generative}.
A breakthrough came from \citet{lee2022convergence}, which provided the first polynomial iteration complexity assuming only $L^2$-accurate score estimates and a log-Sobelev inequality on the target distribution. Subsequent works relaxed this functional inequality condition by requiring Lipschitz scores along the forward process \citep{chen2022sampling,lee2023convergence} or bounded support/moment conditions on the target distribution \citep{chen2023improved}.
Notably, under the minimal assumption---only $L^2$-accurate score estimates and a finite moment of the target distribution--- \citet{benton2023linear} and \citet{li2024d} established state-of-the-art iteration complexities of $\wt{O}(d/\veps^2)$ in KL divergence and $\wt{O}(d/\veps)$ in TV distance, respectively.
For ODE-based samplers, \citet{chen2023restoration} provided the first convergence guarantee, albeit without explicit polynomial dependencies and requiring exact score estimates. \citet{chen2023probability} incorporated stochastic corrector steps into the ODE-based sampler and attained improved convergence bounds. Assuming access to accurate Jacobian estimates of the scores, \citet{li2024sharp} established an iteration complexity of $\wt O(d/\veps)$ in TV distance, 
In pursuit of improved computational efficiency, various acceleration strategies have been proposed based on higher-order discretization of the reverse process \citep{li2024provable,li2024accelerating,li2024improved,wu2024stochastic,huang2024reverse,huang2024convergence,taheri2025regularization} or by leveraging low-dimensional data structures \citep{li2024adapting,tang2024adaptivity,huang2024denoising,potaptchik2024linear,liang2025low}.

%% file: notation.tex

\subsection{Notation and organization}
\label{subsec:Notations}


\paragraph*{Notation.} 
For $a,b\in\bR$, we denote $a \vee b \defn \max\{a,b\}$ and $a\wedge b \defn \min\{a,b\}$.
For positive integer $N>0$, let $[N]\defn\{1,\cdots,N\}$. 
For matrix $A$, we use $\|A\|$ to denote the spectral norm.
For vector-valued function $f$, we denote by $J_f=\frac{\partial f}{\partial x}$ its Jacobian matrix.
For random vector $X$, we use $p_X$ to refer to its distribution and probability density function interchangeably for simplicity of notation.
For random vectors $X, Y$,
let $\mathsf{KL}(X\,\|\,Y) \equiv \mathsf{KL}(p_X\,\|\,p_Y) =\int p_X(x)\log\frac{p_X(x)}{p_Y(x)} \diff x$ and $\TV(X,Y)\equiv\TV(p_X,p_Y)=\frac12\int |p_X(x) - p_Y(x)|\diff x$ denote the KL divergence and the TV distance, respectively.
For probability distributions $p,q$, we denote their convolution by $p\ast q$.
In addition, let $\mathds{1}\{ \cdot \}$ denote the indicator function. 

For two functions $f(n),g(n) > 0$, we use $f(n)\lesssim g(n)$ or $f(n)=O\big(g(n)\big)$ to mean $f(n)\leq Cg(n)$ for some absolute constant $C>0$. 
Similarly, we write $f(n)\gtrsim g(n)$ or $f(n)=\Omega\big(g(n)\big)$ when $f(n)\geq C'g(n)$ for some absolute constant $C'>0$.
We denote $f(n)\asymp g(n)$ or $f(n)=\Theta\big(g(n)\big)$ when $Cf(n)\leq g(n)\leq C'f(n)$ for some absolute constants $C' > C > 0$. The notations $\wt{O}(\cdot)$, $\wt{\Omega}(\cdot)$, and $\wt{\Theta}(\cdot)$ denote the respective bounds up to some logarithmic factors. In addition, we use $f(n) = o(g(n))$ to denote $\limsup_{n \to \infty} f(n)/g(n) = 0$. And we use $\mathsf{poly}(n)$ to denote a polynomial function of $n$ where the specific degree may vary across different contexts.

\paragraph*{Organization.} The paper is structured as follows. Section~\ref{sec:problem_formulation} provides background on SGMs and describes our problem formulation. In Section~\ref{sec:results}, we present our ODE-based sampler and its theoretical guarantees. The analysis is provided in Section~\ref{sec:analysis}, with comprehensive proofs and technical lemmas postponed to the appendix. Finally, Section~\ref{sec:discussion} concludes the paper with a discussion of future research directions.

%% file: problem.tex
\section{Problem formulation}\label{sec:problem_formulation}

\subsection{Preliminaries}\label{sub:preliminaries}
In this section, we provide a brief introduction to score-based diffusion models.

\paragraph*{Forward process.}
The forward process begins with a sample $X_0\sim p_0^\star$ distributed according to some initial distribution $p_0^\star$ (typically chosen to match or closely approximate the target distribution $p^\star$) and progressively adds Gaussian noise via the following Markov chain:
\begin{align}\label{eq:forward}
	X_k = \sqrt{\alpha_k}\,X_{k-1} + \sqrt{1-\alpha_k}\,W_k,\quad k=1,2,\dots,K.
\end{align}
Here, $\alpha_1,\dots,\alpha_K\in(0,1)$ represent the learning rates and $W_1,\dots,W_K\overset{\mathrm{i.i.d.}}{\sim}\cN(0,I_d)$ are $d$-dimensional standard Gaussian random vectors independent of $(X_k)$. Throughout the paper, we use $p^\star$ to denote the target distribution. 

It is important to emphasize that the forward process serves exclusively as a theoretical framework for analysis and is not implemented in practice. In this work, we focus on a specific choice of the initial distribution:
\begin{align}\label{eq:tau}
	p_0^\star=p^\star \ast \cN(0,\tau I_d)
\end{align}
for some parameter $\tau\geq 0$, that is, the convolution of the target distribution $p^\star$ with a Gaussian distribution $\cN(0,\tau I_d)$. This encompasses the case where we begin directly with the target distribution ($\tau=0$) and the cases where we initialize the process with its smoothed approximation ($\tau>0$).

Defining the cumulative noise parameters
\begin{align}\label{eq:ol-alpha-defn}
	\ol\alpha_0 \defn 1 \qquad \text{and} \qquad \ol\alpha_k \defn \prod_{i=1}^k\alpha_i,\quad k=1,2,\dots,K,
\end{align}
one can verify that each $X_k$ can be written as
\begin{align}\label{eq:DDIM-dist}
	X_k = \sqrt{\ol\alpha_k}\,X_0 + \sqrt{1-\ol\alpha_k}\, W_k' = \sqrt{\ol\alpha_k}\big(X_0+\sqrt{(1-\ol\alpha_k)/\ol\alpha_k}\, W_k'\big),
\end{align}
for some $d$-dimensional standard Gaussian random vector $W_k'\sim\cN(0,I_d)$ independent of $X_0$. In particular, the parameter $\ol\alpha_k$ and initial distribution $p_0^\star$ fully determines the distribution of $X_k$.
Notice that if the covariance of $X_0$ is equal to $I_d$, then the covariance of $X_k$ remains $I_d$ for all $k$. Hence, the forward process~\eqref{eq:forward} is often referred to as \emph{variance-preserving} \citep{song2020score}. In particular, when $\ol\alpha_K$ is sufficiently small, the distribution of $X_K$ is approximately normal.

The discrete forward process \eqref{eq:forward} can be interpreted as a time-dicretization of a continuous process $(X_t^\sde)_{t\in[0,T]}$, which is the solution to the following SDE:
\begin{align}
	\diff X_t^\sde = -\frac12 \beta_t X_t^\sde \diff t + \sqrt{\beta_t} \diff B_t, \quad X_0^\sde\sim p_0^\star; \qquad 0\leq t \leq T \label{eq:forward-SDE}
\end{align}
where $\beta_t:[0,T]\to \bR$ is a drift coefficient function and $(B_t)_{t\in[0,T]}$ is a standard Brownian motion in $\bR^d$.

\paragraph*{ODE-based sampling process.}
Reversing the forward process \eqref{eq:forward} in time ideally transforms pure Gaussian noise into a sample from the initial distribution $p_0^\star$, thereby accomplishing generative modeling.
The realization of such a reverse process relies on knowledge of the score functions $\{s_{X_{k}}(\cdot)\}_{k=1}^K$ of the distributions along the forward process $(X_k)_{k=1}^K$, defined as follows:
\begin{definition}[Score function]
	For random vector $X\in\bR^d$ with density $p_X$, the score function $s_{X}:\bR^d\to\bR^d$ is defined as the gradient (with respect to $x$) of log-density of $X$:
\begin{align}
	s_{X}(x) \defn \gd\log p_{X}(x).
\end{align}
\end{definition}

In practical applications, the score functions of the forward process cannot be directly computed since the target density $p^\star$ is typically unavailable in explicit form or analytically intractable. To overcome this fundamental challenge, one can use a training dataset $\{X^{(i)}\}_{i=1}^n$ collected from the target distribution $p^\star$ to construct a score estimator $\{\wh{s}_{X_{k}}(\cdot)\}_{k=1}^K$ for the true scores $\{{s}_{X_{k}}(\cdot)\}_{k=1}^K$. 
%

Given a score estimator $\{\wh{s}_{X_{k}}(\cdot)\}_{k=1}^K$, an ODE-based sampling process adopts the update rule:
\begin{align}\label{eq:DDIM}
	Y_{k-1} = \frac{1}{\sqrt{\alpha_k}} \bigg(Y_k + \frac{1-\alpha_k}{2}\wh{s}_{X_k}(Y_k) \bigg),\qquad k = K,\dots,1,
\end{align}
where the learning rates $(\alpha_k)_{k=1}^K$ are the same as in the forward process \eqref{eq:forward}. Since $p_{X_K}$ converges to $\cN(0,I_d)$ as $\ol\alpha_K$ approaches zero (see \eqref{eq:DDIM-dist}), we choose the initialization as $Y_K\sim \cN(0,I_d)$.


\paragraph*{Probability flow ODE.}
The sampling scheme in \eqref{eq:DDIM} is fundamentally grounded in the probability flow ODE for diffusion models.
Indeed, the continuous process $(X_t^\sde)_{t\in[0,T]}$ defined by the forward SDE \eqref{eq:forward-SDE}, which can be interpreted as the continuum limit of the forward process $(X_k)_{k=0}^K$ in \eqref{eq:forward}, admits a reverse-time process $(Y^\ode_t)_{t\in[0,T]}$ \citep{song2020score}. In specific, there exists a probability flow ODE such that when initialized at $Y_0^\ode\sim p_{X_T^\sde}$, the process $(Y^\ode_t)$ driven by this ODE dynamic shares the same marginal distributions as  $(X^\sde_{T-t})$, that is, $Y_t^\ode \disteq X^\sde_{T-t}$ for all $t\in[0,T]$. The probability flow ODE admits the expression:
\begin{align}\label{eq:ODE}
	\diff Y_t^\ode = \frac12 \beta_{T-t}\big(Y_t^\ode+s_{X_{T-t}^\sde}(Y_t^\ode) \big)\diff t, \quad Y_0^\ode\sim p_{X_T^\sde}; \qquad 0\leq t \leq T.
\end{align}
where $s_{X_t^\sde}=\gd \log p_{X_t^\sde}$ is the score function of $X_t^\sde$ in the forward SDE \eqref{eq:forward-SDE}.
Therefore, the ODE-based sampler \eqref{eq:DDIM} can be viewed as a time-discretization of $(Y^\ode_t)_{t\in[0,T]}$ governed by the probability flow ODE \eqref{eq:ODE}.

\paragraph*{Score estimation reduction.}\label{par:score_estimation_reduction_}
To estimate the score functions 
$\{s_{X_{k}}(\cdot)\}_{k=1}^K$ 
of the forward process $(X_k)_{k=1}^K$ in \eqref{eq:forward}, it is often more convenient to consider the following SDE:
\begin{align}\label{eq:SDE-VE}
	\diff Z_t = \diff B_t, \quad Z_0 \sim p^\star; \quad t \geq 0
\end{align}
where $(B_t)_{t\geq0}$ is a standard Brownian motion in $\bR^d$. The solution $(Z_t)_{t\geq 0}$, commonly referred to as \emph{variance exploding} process \citep{song2020score}, satisfies that for any $t > 0$:
\begin{align}\label{eq:SDE-VE-dist}
	Z_t \disteq Z_0 + \sqrt{t}\, W_t''
\end{align} 
for some $d$-dimensional standard Gaussian random vector $W_t''\sim\cN(0,I_d)$ independent of $Z_0$. The name ``variance exploding'' originates from the fact that the variance $Z_t$ increases as $t$ grows.

Given that the initial distribution of the forward process $(X_k)_{k=0}^K$ in \eqref{eq:forward} is chosen to be $p_0^\star=p^\star \ast \cN(0,\tau I_d)$ for some $\tau\geq0$ in \eqref{eq:tau}, combining the distributional characterization of $(Z_t)_{t>0}$ in \eqref{eq:SDE-VE-dist} with that of the forward process $(X_k)_{k=1}^K$ in \eqref{eq:DDIM-dist} shows
\begin{align}\label{eq:DDIM-VE}
	X_k \disteq \sqrt{\ol\alpha_k}\,Z_{t_k}\quad \text{with}\quad t_k\defn\frac{1-\ol\alpha_k}{\ol\alpha_k} + \tau,\qquad k=1,\dots,K. 
\end{align}
This distributional equivalence further implies the relationships between their score functions and associated Jacobian matrices:
\begin{align}\label{eq:score-relationship}
	s_{X_k}(x)=\frac1{\sqrt{\ol\alpha_k}}s_{Z_{t_k}}(x/\sqrt{\ol\alpha_k}) \quad \text{and} \quad J_{s_{X_k}}(x)= \frac1{\ol\alpha_k}J_{s_{Z_{t_k}}}(x/\sqrt{\ol\alpha_k}),\qquad k=1,\dots,K. 
\end{align}
Therefore, score estimation for the forward process $(X_k)_{k=1}^K$ can be reduced to estimating the score functions $\{s_{Z_t}(\cdot)\}_{t>0}$ of the variance-exploding process $(Z_t)_{t>0}$. The proof of the above relationship \eqref{eq:score-relationship} can be found in Lemma~\ref{lemma:score-error-X-Z} in Appendix~\ref{sec:auxiliary_lemmas}.

\subsection{Assumptions}\label{sub:assumptions}

In this section, we introduce assumptions on the target distribution $p^\star$.

\paragraph*{Subgaussian target distribution.}
First, we assume the target distribution $p^\star$ is subgaussian. Formally, we adopt the definition from \citet{vershynin2018high} as follows.
\begin{assumption}[Subgaussian target]\label{assume:subgaussian}
We assume $p^\star$ is $\sigma$-subgaussian distribution over $\bR^d$, that is, for any unit vector $\theta\in\bR^d$ with $\|\theta\|_2=1$,
\begin{align}
 \bE\bigl[\exp\bigl((X^\top \theta / \sigma)^2\bigr)\bigr] \leq 2,\quad X\sim p^\star.
\end{align}
\end{assumption}

The subgaussian condition is relatively mild in the sense that it encompasses any distribution with bounded support. Moreover, it can be viewed as a generalization of a log-Sobolev condition, which has a long history of use in sampling history \citep{durmus2019high,vempala2019rapid,raginsky2017non,ma2019sampling,yang2022convergence,lee2022convergence}. Indeed, the subgaussian property can be derived from the log-Sobolev inequality but not vice versa \citep{vershynin2018high}.
\begin{remark}
Parallel to the subgaussian assumption, another line of works assumes that the target density is bounded below by a constant over a compact support (e.g., \citet{oko2023diffusion,dou2024optimal}). 
Although this condition accommodates many cases of interest and simplifies analysis, it can exclude multi-modal distributions frequently found in real-world applications---where diffusion models have shown particular advantages over traditional methods like Langevin dynamics. Indeed, a constant lower-bounded density on a compact support implies certain functional inequalities, where Langevin dynamics are provably effective. Therefore, we choose to focus on the subgaussian setting, which more closely reflects many practical scenarios.
\end{remark}
\paragraph*{H\"older smooth target density.}
Next, we assume that the density of the target distribution $p^\star$ belongs to the H\"older class with smooth parameter $\beta>0$, as defined in \citet{tsybakov2009introduction}. 
Let $\lfloor\beta\rfloor$ denote the largest integer that is strictly smaller than $\beta$.
For vector $x=(x_1,\dots,x_d)\in\bR^d$ and multi-index $s=(s_1,\dots,s_d)\in\bZ_{\geq 0}^d$, 
we write $|s|\defn\sum_{i=1}^d s_i$ and $x^s=x_1^{s_1}\dots x_d^{s_d}$. 
The partial derivative with respect to the index $s$ is denoted by
$\partial^s \defn \frac{\partial^{s_1+\dots+s_d}}{\partial x_1^{s_1}\dots\partial x_d^{s_d}}$.
\begin{assumption}[H\"older smooth target]\label{assume:smooth}
There exist $\beta,L>0$ such that the density of the target distribution 
$p^\star$ is $\lfloor\beta\rfloor$-times differentiable and
\begin{align}
	\max_{|s|\leq \lfloor\beta\rfloor}\bigl|\partial^s p^\star(x)- \partial^s p^\star(x')\bigr|\leq L \left\|x-x'\right\|^{\beta-\lfloor\beta\rfloor}_2,\quad \forall x,x'\in\bR^d.
\end{align}
\end{assumption}
This smoothness assumption is standard in nonparametric settings and accommodates a broad range of practical distributions.


%% file: results.tex
\section{Results}\label{sec:results}

In this section, we first present the score estimator and probability flow ODE sampler, followed by the statistical guarantees on the sampling performance.

\subsection{Algorithm}

We sequentially introduce the three critical components of our probability flow ODE-based sampler: the learning rate schedule, score estimator, and sampling procedure. The complete procedure is summarized in Algorithm~\ref{alg:DDIM}.

\paragraph*{Learning rates.}
Let us start with the learning rates that define the forward process $(X_k)_{k=1}^K$ and implement the reverse process $(Y_k)_{k=1}^K$. We employ a learning rate schedule $(\alpha_k)_{k=1}^K$ similar to that proposed in \citet{li2024provable}. In light of $\ol\alpha_k\defn \prod_{k=1}^K \alpha_k$ and the distribution characterization in \eqref{eq:DDIM-dist} that $X_k = \sqrt{\ol\alpha_k}\,X_0 + \sqrt{1-\ol\alpha_k}\, W_k'$ with $W_k'\sim\cN(0,I_d)$ independent of $X_0$, it suffices to specify the learning rates through the following recursive relationship on the cumulative noise parameters $(\ol\alpha_k)_{k=1}^K$:
\begin{align}
\label{eq:learning-rate}
\ol\alpha_1 = 1-\frac{1}{K^{c_0}},
\qquad\text{and}\qquad
\ol{\alpha}_{k} = \ol\alpha_{k-1}-\frac{c_1\log K}{K}\ol\alpha_{k-1}(1-\ol\alpha_{k-1}), \quad k = 2,\dots,K,
\end{align}
where $c_0,c_1 > 0$ are some large constants obeying $c_1 \geq 5c_0$.

As shown in Lemma~\ref{lemma:step-size}, the chosen learning rates ensure that
\begin{align*}
	\ol\alpha_K\leq \frac{1}{K^{c_1/4}}\qquad \text{and}\qquad 1-\alpha_k < \frac{1-\alpha_k}{1-\ol\alpha_k}\asymp \frac{1-\alpha_k}{1-\ol\alpha_{k-1}} = \frac{c_1\log K}{K}, \quad k=2,\dots,K.
\end{align*}

\begin{remark}

Our results hold for a broader class of learning rate schedules, provided that they satisfy:
\begin{align*}
1-\overline{\alpha}_{1} = \frac1{\mathsf{poly}(K)},
 \qquad \overline{\alpha}_{K} = \frac1{\mathsf{poly}(K)} 
\qquad \text{and} \qquad  \frac{1-\alpha_k}{1-\overline{\alpha}_{k}} &\asymp \frac{\log K}{K}, \quad k \geq 2.
\end{align*} 
These conditions ensure that $1-\overline{\alpha}_{1}$ is sufficiently small (so that $p_{X_1}$ closely approximates $p^\star_0$ and $p^\star$) and that $\overline{\alpha}_{K}$ is small (so that $p_{X_K}$ is nearly Gaussian), which is achieved by the chosen rate $(1-\alpha_k)/(1-\overline{\alpha}_{k})\asymp \log K/{K}$.
\end{remark}

\paragraph*{Score estimator.}

Given the forward process $(X_k)_{k=1}^K$ determined by the above learning rates $(\alpha_k)_{k=1}^K$, our next step is use the i.i.d.~training data $\{X^{(i)}\}_{i=1}^n$ sampled from the target distribution $p^\star$ to estimate the score functions $\{s_{X_k}(\cdot)\}_{k=1}^K$ of the forward process.
In light of the reduction argument in Section~\ref{par:score_estimation_reduction_}, it is equivalent to estimating the scores of the variance-exploding process $(Z_t)_{t>0}$ defined in \eqref{eq:SDE-VE}.

To this end, for each $t>0$, we first construct a Gaussian kernel-based density estimator $\wh{p}_{t}(\cdot):\bR^d\to\bR$ for $Z_t$:
\begin{align}\label{eq:p_hat} 
	\wh{p}_{t}(x) &\defn \frac1{n} \sum_{i=1}^n \vphi_{t}(X^{(i)}-x),
\end{align}
where $\vphi_t(\cdot):\bR^d \to \bR^d$ denotes the probability density function of $\cN(0, tI_d)$, i.e.,
\begin{align}\label{eq:phi}
	\vphi_t(x) \defn \frac1{(2\pi t)^{d/2}} \exp\bigl(-\|x\|_2^2/(2t)\bigr).
\end{align}
Using this density estimator, we then define the score estimator $\wh{s}_{t}(\cdot):\bR^d\to\bR^d$ for $Z_t$ as follows:
\begin{align}\label{eq:score-estimate}
		\wh{s}_{t}(x) \defn \frac{\gd\wh{p}_t(x)}{\wh{p}_{t}(x)} \psi\bigl(\wh{p}_{t}(x);\eta_t\bigr)
	\qquad \text{with} \qquad \eta_t \defn \frac{\log n}{n(2\pi t)^{d/2}}.
\end{align}
Here, $\psi(\cdot\,;\eta):[0,\infty)\to\bR$ is a soft-thresholding function with threshold $\eta$, defined as
\begin{align}\label{eq:soft-thre-defn}
	\psi(x;\eta)\defn
	\begin{cases}
		1, & x \geq \eta;\\
		0, & x \leq \eta/2;\\
		\Bigl[1+\exp\Bigl(\frac{1-2(2x/\eta-1)}{(2x/\eta-1)(2-2x/\eta)}\Bigr)\Bigr]^{-1},  &  x \in(\eta/2, \eta).
	\end{cases} 
\end{align}

We now present the score estimator $\{\wh{s}_{X_k}(\cdot)\}_{k=1}^K$ for the forward process $(X_k)_{k=1}^K$.
Given the relationship between $s_{X_k}$ and $s_{Z_t}$ shown in \eqref{eq:score-relationship}, we leverage the estimator in \eqref{eq:score-estimate} to construct $\wh{s}_{X_k}$ as follows:
\begin{align}\label{eq:score-estimator-X}
	\wh{s}_{X_k}(x) \defn \frac{1}{\sqrt{\ol\alpha_k}} \wh{s}_{t_k} \big(x/\sqrt{\ol\alpha_k} \big) \quad \text{with} \quad t_k\defn\frac{1-\ol\alpha_k}{\ol\alpha_k}+\tau,\qquad k=1,\dots,K,
\end{align}
where $(\ol\alpha_k)_{k=1}^K$ are chosen in the learning rate schedule \eqref{eq:learning-rate} and $\tau$ is defined in \eqref{eq:tau}.

In a nutshell, we propose a smooth regularized score estimator $\wh s_t$ based on a Gaussian kernel-based density estimator $\wh p_t$. Starting with the plug-in estimator $\wh{p}_{t}/\gd\wh{p}_t$, we introduce a soft-thresholding procedure $\psi(\wh p_t;\eta_t)$ with threshold $\eta_t$ that depends on the sample size $n$ and time parameter $t$.

In the regions where the estimated density $\wh p_t$ (and consequently the true density $p_{Z_t}$) is low (with respect to the threshold $\eta_t$), 
the plug-in estimator $\wh{p}_{t}/\gd\wh{p}_t$ becomes inherently unstable due to small denominators, a challenge compounded by limited training data in these sparse regions.
%
To mitigate this issue, we set the score estimator $\wh s_t$ to zero in these low-density areas. Conversely, in the high-density regimes, the plug-in estimator $\wh{p}_{t}/\gd\wh{p}_t$ is sufficiently reliable so we use it directly. In the transition regions between these extremes, we incorporate a ``bump'' function $\psi$ that ensures the resulting score estimator maintains the necessary smoothness properties.

To provide some intuition, we note that the plug-in estimator $\wh{p}_{t}/\gd\wh{p}_t$ represents the score function of the empirical distribution of the training data after being diffused by noise. If used directly in the reverse process, it would effectively transform the noise back into the empirical distribution. In other words, the diffusion model would merely memorize the training data and fail to generalize. This necessity for regularization has been recognized in prior works on score estimation for DDPM \citep{scarvelis2023closed,wibisono2024optimal,zhang2024minimax,dou2024optimal} and is closely connected to approaches in the empirical Bayes literature \citep{jiang2009general,saha2020nonparametric}.
However, as discussed in Section~\ref{sec:introduction}, existing score estimators often exhibit discontinuity and nonsmoothness, and their efficiency could deteriorate when employed in the ODE-based deterministic samplers. Our soft-thresholding strategy serves as the critical step to ensure the mean Jacobian error remains sufficiently small in addition to the $L^2$ score error, thereby leading to robust convergence guarantees.
%


\paragraph*{Sampling process.}
\begin{algorithm}[t]
\caption{Probability flow ODE-based sampler}
\label{alg:DDIM}
\begin{algorithmic}[1]
\STATE{\textbf{Input:} training data $\{X^{(i)}\}_{i=1}^n$ sampled from the target distribution $p^\star$.} 
\STATE{Set the learning rates $\{\alpha_k\}_{k=1}^K$ according to \eqref{eq:learning-rate}.}
\STATE{Construct the score estimator $\{\wh{s}_{X_k}(\cdot)\}_{k=1}^K$ according to \eqref{eq:score-estimator-X}.}
\STATE{Initialize $Y_K\sim\cN(0,I_d)$.}
\FOR{$k=K,K-1,\dots,2$}
\STATE{$Y_{k-1} \gets \frac1{\sqrt{\alpha_k}}\big(Y_{k}+\frac{1-\alpha_k}2 \wh{s}_{X_k}(Y_{k}) \big)$.}
\ENDFOR
\STATE{\textbf{Output:} sample $Y=Y_1/\sqrt{\alpha_1}$.}
\end{algorithmic}
\end{algorithm}

Equipped with the score estimators $\{\wh{s}_{X_k}(\cdot)\}_{k=1}^K$ from \eqref{eq:score-estimator-X} and learning rate schedule $(\alpha_k)_{k=1}^K$ from \eqref{eq:learning-rate}, the probability flow ODE-based sampler iteratively generates samples as follows. Initializing $Y_K\sim\cN(0,I_d)$, we compute $Y_{k-1}$ based on the deterministic mapping in \eqref{eq:DDIM}:
\begin{align*}
Y_{k-1} = \frac1{\sqrt{\alpha_k}}\biggl(Y_{k}+\frac{1-\alpha_k}2 \wh{s}_{X_k}(Y_{k}) \bigg),\quad k=K,\dots,2,
\end{align*}
and output $Y=Y_1/\sqrt{\alpha_1}$ as the generated sample.
It is worth noting that the sampling process is completely deterministic dynamic (except for the random initialization of $Y_K$), that is, $(Y_{K-1},\dots,Y_1)$ is purely deterministic given $Y_K$.

\subsection{Theoretical guarantees}

In this section, we present the end-to-end sampling guarantee of the proposed ODE-based sampler in Theorem~\ref{thm:TV}. A proof outline is provided in Section~\ref{sec:analysis}.
\begin{theorem}\label{thm:TV}
Suppose the target distribution $p^\star$ satisfies Assumption~\ref{assume:subgaussian} and Assumption \ref{assume:smooth} with $\beta \leq 2$. If we use the score estimator in \eqref{eq:score-estimator-X} with $\tau=n^{-\frac{2}{d+2\beta}}$ and the learning rates in \eqref{eq:learning-rate} with $c_0\geq 2/\beta$ and $c_1 = 5c_0\vee 12$, then the output $Y=Y_1/\sqrt{\alpha_1}$ of Algorithm~\ref{alg:DDIM} satisfies
	\begin{align}
		\bE\big[\TV(p_{Y},\,p^\star)\big] \leq C\,n^{-\frac{\beta}{d+2\beta}} (\log n)^{\frac{d+1}{2}}\log K,
	\end{align}
	provided the iteration number satisfies $K\geq n^{\frac{\beta}{d+2\beta}}(\log K)^3$. Here, $C = C(d,\beta,L,\sigma)>0$ is a constant that depends only on the dimension $d$, smoothness parameters $(\beta,L)$, and subgaussian parameter $\sigma$.
The expectation is taken over the training data $X^{(i)}\overset{\mathrm{i.i.d.}}{\sim} p^\star$, $1\leq i \leq n$.
\end{theorem}

In summary, Theorem~\ref{thm:TV} provides a rigorous end-to-end statistical guarantee for ODE-based samplers under mild assumptions on the target distributions.
Let us discuss the implications of this theorem.
\begin{itemize}
	\item \emph{(Near-)minimax optimal sampling.}
Since the minimax rate of density estimation for $\beta$-H\"older smooth densities scales as $n^{-\beta/(d+2\beta)}$ \citep{stone1980optimal,stone1982optimal, tsybakov2009introduction} and density estimation forms a valid lower bound for sampling, 
our result establishes the minimax optimality of the probability flow ODE-based sampler up to some logarithmic factors. As far as we know, this is the first result that reveals the (near-)minimax optimality of the deterministic score-based sampler.

\item \emph{End-to-end theoretical guarantees.} Unlike prior works that isolate either the discretization error in the practical sampling process \citep{oko2023diffusion,zhang2024minimax,dou2024optimal} or score estimation error \citep{chen2023improved,chen2022sampling,benton2023linear, li2024d,li2024sharp}, our analysis rigorously tracks all sources of errors. This end-to-end guarantee enables a clearer understanding of how imperfect score estimates and discretized sampling updates jointly influence the final sampling quality, yielding a rigorous and comprehensive characterization of the sampling error.

\item \emph{Jacobian estimation guarantees.} As a critical step for establishing the above guarantee, our analysis quantifies the mean Jacobian error of our proposed score estimator \eqref{eq:score-estimate}, extending beyond the standard $L^2$ score estimation guarantees \citep{oko2023diffusion,wibisono2024optimal,zhang2024minimax,dou2024optimal}. This characterization ensures that the score estimator not only achieves $L^2$ accuracy but also possesses the necessary smoothness properties. Such smoothness is fundamental for controlling the discretization errors in the deterministic sampling process and ultimately guarantees the sampling performance established in our main theorem.

\item \emph{Mild target distribution assumptions.} The developed theory does not impose strong structural conditions such as lower-bounded densities or Lipschitz/smooth scores on the target distributions required in prior works (e.g., \citet{oko2023diffusion,wibisono2024optimal,dou2024optimal}). 
As a result, our analysis framework applies to a broader class of real-world distributions---including ones that are multimodal or lie on low-dimensional manifolds. This is particularly beneficial for modern applications where those restrictive conditions often fail.

\item \emph{No need of early stopping.} The established sampling guarantee holds as long as the iteration number $K$ exceeds $n^{\frac{\beta}{d+2\beta}}$ up to some log factors. This result aligns with the findings of \citet{dou2024optimal} where early stopping techniques are not needed to achieve the (near-)minimax optimal performance.

\item \emph{Analysis framework for ODE-based samplers.} The theoretical framework we establish provides a foundation for analyzing the performance guarantees of various ODE-based samplers. For instance, it can be extended to accommodate higher-order ODE solvers or alternative score estimation strategies, leading to improved computational efficiency or statistical accuracy.

\end{itemize}

As a final remark, we note that Theorem~\ref{thm:TV} guarantees optimal convergence rates only for ``rough'' densities with $\beta\leq 2$. According to classical density estimation theory, kernels with negative components are necessary when estimating smoother densities \citep{tsybakov2009introduction}, which renders our Gaussian kernel-based score estimator \eqref{eq:score-estimate} suboptimal when the forward process is close to the target distribution (i.e., $\alpha_1$ vanishingly small). Indeed, the optimal estimation of H\"older scores with higher smoothness remains an open problem \citep{wibisono2024optimal}.
Since our primary objective is to demonstrate the achievability of ODE-based samplers attaining near-minimax optimal performance through appropriately regularized score estimation and a suitably chosen learning rate schedule, we focus on the case $0<\beta\leq 2$ for clarity of presentation. In this regime, we can balance the score estimation error against the bias introduced by early stopping---a trade-off sufficient to attain the optimal sampling rate. Nevertheless, we expect the algorithmic insights presented here can be extended to the $\beta>2$ case by replacing the Gaussian kernel with polynomial kernels \citep{zhang2024minimax,dou2024optimal} and employing suitable regularization to control the Jacobian errors. We leave a thorough investigation of this direction to future work.

%% file: analysis.tex

\section{Analysis}
\label{sec:analysis}

In this section, we provide the proof strategy for establishing Theorem~\ref{thm:TV}, structured into three main steps.

\subsection{Preparation}
Before the technical expositions, we first collect several preliminary facts that will be helpful for the proof.

\paragraph*{Properties of learning rates.}
We begin by presenting the following lemma that summarizes important properties of the learning rates chosen in \eqref{eq:learning-rate}. 
The proof can be found in Appendix~\ref{sec:pf-lemma:step-size}.
\begin{lemma}
\label{lemma:step-size}
The learning rates $(\alpha_k)_{k=1}^K$ chosen in \eqref{eq:learning-rate} satisfy that for all $k\geq 2$:
\begin{subequations}\label{eq:step-size}
	\begin{align}
	1-\alpha_k &< \frac{1-\alpha_k}{1-\ol\alpha_k}\leq \frac{c_1\log K}{K}\label{eq:alpha-t-lb}, \\ 
	1 &< \frac{1-\ol\alpha_k}{1-\ol\alpha_{k-1}} \leq 1+\frac{c_1\log K}{K}, \label{eq:learning-rate-3}\\
	\frac{c_1\log K}{K} &< \frac{1-\alpha_k}{\alpha_k-\ol\alpha_k} \leq \frac{2c_1\log K}{K}, \label{eq:learning-rate-4}
	\end{align}
	where $c_1$ is defined in \eqref{eq:learning-rate}. 
	Further, $\ol\alpha_K$ satisfies
	\begin{align}
	\ol\alpha_K  \leq K^{-c_1/4}. \label{eq:learning-rate-6}
	\end{align}
\end{subequations}
\end{lemma}

\paragraph*{Distance between $p_{X_K}$ and $p_{Y_K}$.}

Recall that the initializer of our sampler is chosen $Y_{K}\sim \cN(0,I_d)$ and that the forward process satisfies $X_{K}=\sqrt{\ol \alpha_K}X_{0} + \sqrt{1-\ol\alpha_K} W_K'$ with $W_K'\sim \cN(0,I_d)$ independent of $X_{0}$. 
In recognition of $\ol\alpha_K  \leq K^{-c_1/4}$ shown in Lemma~\ref{lemma:step-size}, the following lemma shows that the distributions of $X_{K}$ and $Y_K$ are sufficiently close.
The proof is deferred to Appendix~\ref{sub:proof_of_lemma_ref_lemma_mixing}.
\begin{lemma}\label{lemma:mixing}
Suppose the learning rates are chosen according to \eqref{eq:learning-rate}. Then the TV distance between $p_{X_{K}}$ and $p_{Y_{K}}$ satisfies
	\begin{align}
	\TV\big(p_{X_{K}}\,\|\,p_{Y_{K}}\big) \leq \frac{1}{2K^{c_1/8}}\sqrt{\bE_{X_0\sim p_0^\star}\big[\|X_0\|_2^2\big]}.
	\label{eq:KL-step-T}
\end{align}
\end{lemma}
As a remark, the error bound depends on the initial distribution of $X_0\sim p^\star_0$ in the forward process \eqref{eq:forward} only through its second moment.

\paragraph*{Auxiliary quantities and notation.}

For clarity of presentation, we introduce the following streamlined notation that will facilitate clearer exposition throughout our analysis.

\begin{itemize}
	\item 
To begin with, recall the variance-exploding process $(Z_t)_{t>0}$ introduce in \eqref{eq:SDE-VE-dist}.
For each $t>0$, we denote the density, score function, and associated Jacobian matrix of $Z_t$ by
\begin{align}\label{eq:Z-t-abbrev}
	p_t(x)\defn p_{Z_t}(x),\qquad s_t(x)\defn s_{Z_t}(x),\qquad \text{and}\qquad J_t(x) \defn J_{s_t}(x).
\end{align}
By Tweedie's formula \citep{efron2011tweedie}, the score function $s_{t}(x)$ takes the form
\begin{align}\label{eq:score-expression}
	s_{t}(x) = \frac1t \bE\big[Z_0-Z_t\mid Z_t = x\big].
\end{align}
The Jacobian matrix 
$J_{t}(x)$
can be expressed as
\begin{align}
	J_{t}(x) 
	= -\frac1t I_d + \frac{1}{t^2} \bE\big[(Z_0-Z_t)(Z_0-Z_t)^\top \mid Z_t = x\big] - s_{t}(x)s_{t}(x)^\top. \label{eq:Jacobian}
\end{align}

\item 
Next, recall our density estimator $\wh{p}_t$ for $p_t$ in \eqref{eq:p_hat}, where $\vphi_t$ denotes the density of $\cN(0,tI_d)$. Also, the i.i.d.\ training data $\{X^{(i)}\}_{i=1}^n$ are distributed identically to $Z_0\sim p^\star$. For each $t>0$, we define the vector-valued functions $\wh{g}_t(\cdot),g_t(\cdot):\bR^d\to\bR^d$ as
\begin{subequations}\label{eq:g}
\begin{align}
	\wh{g}_{t}(x) &\defn \gd \wh{p}_t(x) = \frac1{nt}\sum_{i=1}^n (X^{(i)}-x)\vphi_{t}(X^{(i)}-x)\label{eq:g-hat};\\
	g_t(x) &\defn \bE\big[\wh{g}_{t}(x)\big] = \frac1t\bE \big[(Z_0-x)\vphi_t(Z_0-x) \big]. \label{eq:g_t}
\end{align}
Notice that 
\begin{align}
	\frac{g_t(x)}{p_{t}(x)} & = \frac1{t} \int_{\bR^d} (y-x) \frac{1}{(2\pi t)^{d/2}}\exp\bigl(-{\|y-x\|_2^2}/({2t})\bigr) \frac{p_0(y)}{p_{t}(x)} \diff y \notag \\
	& = \frac1{t} \int_{y\in\bR^d} (y-x) p_{Z_t \mid Z_0}(y \mid x)\frac{p_0(y)}{p_{t}(x)} \diff y \notag\\ 
	& = \frac1{t} \int_{\bR^d} (y-x) p_{Z_0 \mid Z_t}(y \mid x) \diff y \notag \\
	&  = \frac1{t} \bE\big[Z_0-x \mid Z_t = x\big] = s_{t}(x)\label{eq:g-p-exp},
\end{align}
where the second line holds as $Z_t \mid Z_0=y \sim \cN(y,tI_d)$. 
\end{subequations}
In addition, we define the matrix-valued functions $\wh H_t(\cdot),H_t(\cdot):\bR^d \to \bR^{d\times d}$ for each $t>0$:
\begin{subequations}\label{eq:H}
\begin{align}
	\wh{H}_{t}(x)  &\defn \frac1{nt^2} \sum_{i=1}^n (X^{(i)}-x)(X^{(i)}-x)^\top \vphi_{t}(X^{(i)}-x) \label{eq:Phi-hat},\\
	H_t(x) & \defn \bE\big[\wh{H}_t(x)\big] = \frac1{t^2}\bE\big[(Z_0-x)(Z_0-x)^\top \vphi_t(Z_0-x)\big].\label{eq:Phi-def}
\end{align}
Similar to \eqref{eq:g-p-exp}, we also have
\begin{align}
	\frac{H_t(x)}{p_{t}(x)} 
	& = \frac1{t^2} \int_{y\in\bR^d} (y-x)(y-x)^\top \frac{1}{(2\pi t)^{d/2}}\exp\bigl(-{\|y-x\|_2^2}/{(2t)}\bigr) \frac{p_0(y)}{p_{t}(x)} \diff y \notag \\
	& = \frac1{t^2} \int_{y\in\bR^d} (y-x)(y-x)^\top p_{Z_0 \mid Z_t}(y \mid x) \diff y \notag \\
	&  = \frac1{t^2} \bE\big[(Z_0-Z_t)(Z_0-Z_t)^\top \mid Z_t = x\big]. \label{eq:Phi-p-exp}
\end{align} 
\end{subequations}
Combining \eqref{eq:g-p-exp} and \eqref{eq:Phi-p-exp} with the expression of the Jacobian $J_{t}(x)$ in \eqref{eq:Jacobian}, we can alternatively express it as
\begin{align}
	J_{t}(x) & = -\frac1t I_d + \frac{H_t(x)}{p_{t}(x)} - s_{t}(x)s_{t} = -\frac1t I_d + \frac{H_t(x)}{p_{t}(x)} - \frac{g_{t}(x)g_{t}(x)^\top}{p_t^2(x)}\label{eq:Jacobian-expression}.
\end{align}

\item 
To proceed, for some absolute constant $C_\cE>0$ large enough, we define the event
\begin{align}
	\cE_t(x) \defn \Bigg\{\big|\wh{p}_t(x)-p_t(x)\big| \leq C_\cE \biggl( \frac{\log n}{n(2\pi t)^{d/2}}+ \sqrt{\frac{p_{t}(x)\log n}{n(2\pi t)^{d/2}}} \biggr) \Bigg\} \label{eq:event-E},
\end{align}
for all $x\in\bR^d$ and $t>0$.
Finally, for each $t>0$, we denote the set
\begin{align}\label{eq:rho-t}
	\cF_t\defn\big\{x\in\bR^d\colon p_t(x)\geq c_\eta \eta_t \big\},
\end{align}
with $c_\eta \defn 4C_\cE \vee 2$ and threshold $\eta_t$ as defined in \eqref{eq:score-estimate}.
In short, the set $\cF_t$ represents the region where the density $p_{t}(x)$ of $Z_t$ is not vanishingly small. Note that $Z_t$ is $\sqrt{\sigma^2+t}$-subgaussian given  $Z_0\sim p^\star$ is $\sigma$-subgaussian by Assumption~\ref{assume:subgaussian}.
\end{itemize}
\paragraph*{Properties of the score estimator $\wh{s}_t$ for $Z_t$.}

Finally, let us examine the score estimator $\wh{s}_t(x)$ for $Z_t$ defined in \eqref{eq:score-estimate} and its associated Jacobian $J_{\wh{s}_t}(x)$, which exhibit different forms depending on the value of $\wh{p}_t(x)$.
\begin{itemize}
	\item For $\wh{p}_t(x) \geq \eta_t$, the score estimator satisfies $\wh{s}_t(x) = \gd\wh{p}_t(x)/\wh{p}_t(x)$ with Jacobian $J_{\wh{s}_t}(x)$ given by
\begin{subequations}\label{eq:Jacobian-score-est-expression}
	\begin{align}
	J_{\wh{s}_t}(x)
	 = -\frac1t I_d + \frac{\wh{H}_t(x)}{\wh{p}_t(x)} - \wh{s}_t(x)\wh{s}_t(x)^\top
	  = -\frac1t I_d + \frac{\wh{H}_t(x)}{\wh{p}_t(x)} - \frac{\gd\wh{p}_t(x)\gd\wh{p}_t(x)^\top}{\wh{p}_t^2(x)} \label{eq:J-1}.
\end{align}
\item For $\wh{p}_t(x) \leq \eta_t / 2$, both the score estimator and its Jacobian vanish, i.e., 
\begin{align}\label{eq:J-2}
	\wh{s}_t(x) = 0\qquad \text{ and } \qquad J_{\wh{s}_t}(x)=0.
\end{align}
\item For $\eta_t/2 < \wh{p}_t(x) < \eta_t$, the score estimator becomes $\wh{s}_t(x) = \gd\wh{p}_t(x)/\wh{p}_t(x) \psi(\wh{p}_t(x);\eta_t)$ with Jacobian given by
\begin{align}
	J_{\wh{s}_t}(x) 
	& = \biggl(-\frac1t I_d + \frac{\wh{H}_t(x)}{\wh{p}_t(x)} - \frac{\gd\wh{p}_t(x)\gd\wh{p}_t(x)^\top}{\wh{p}_t^2(x)}\biggr) \psi\bigl(\wh{p}_{t}(x);\eta_t\bigr)  + \frac{\gd\wh{p}_t(x)\gd\wh{p}_t(x)^\top}{\wh{p}_t(x)}\psi'\bigl(\wh{p}_{t}(x);\eta_t\bigr). \label{eq:J-3}
\end{align}
\end{subequations}
\end{itemize}

\subsection{Step 1: Convergence guarantee}

With the above preparation in place, let us begin the detailed analysis. We first present the convergence guarantee for the ODE-based sampling process \eqref{eq:DDIM} that relates the sampling error of the generated output to the quality of an arbitrary score estimator.


Let us begin by introducing several quantities that measure the accuracy of a given score estimator.
For any score estimator $\{\wh s_{X_k}(\cdot)\}_{k=1}^K$, define the $L^2$ score error and the mean Jacobian error for each $k\in[K]$:
\begin{subequations}\label{eq:error-score-jab-defn}
\begin{align}
	\veps^2_{\score,k}\defn \bE\Big[ \big\| \wh{s}_{X_k}(X_k)-s_{X_k}(X_k) \big\|_2^2 \Big]\quad \text{and} \quad \veps_{\jcb,k}\defn \bE\Big[ \big\| J_{\wh{s}_{X_k}}(X_k)-J_{s_{X_k}}(X_k) \big\| \Big],
\end{align}
where the expectation is over both the training data $\{X^{(i)}\}_{i=1}^n$ and $X_k$ in the forward process \eqref{eq:forward}. We then define the average errors over the $K$ iterations as
\begin{align}
	\veps_\score \defn \sqrt{\frac1K\sum_{k=1}^K(1-\ol\alpha_k)\veps^2_{\score,k}} \qquad \text{and} \qquad \veps_\jcb \defn  \frac1K\sum_{k=1}^K(1-\ol\alpha_k)\veps_{\jcb,k},
\end{align}
\end{subequations}
where $(\ol\alpha_k)_{k=1}^K$ are determined in the learning rate schedule \eqref{eq:learning-rate}.
With the notation in place, we now state the following theorem that captures the sampling quality of the ODE-based sampler \eqref{eq:DDIM} given an arbitrary score estimator.
The proof can be found in Appendix~\ref{sub:proof_of_theorem_ref_thm_convergence}.
\begin{theorem}\label{thm:convergence}
Suppose that the number of iterations satisfies $K\gtrsim d^2(\log K)^5$ and 
$K^{c_2}\geq \bE_{X_0\sim p_0^\star}[\|X_0\|_2^2]$ for some absolute constant $c_2>0$ and that the learning rates are chosen according to \eqref{eq:learning-rate} with $c_1/8-c_2/2\geq 1$.
Then for any score estimator $\{\wh s_{X_k}(\cdot)\}_{k=1}^K$ satisfying $J_{\wh{s}_{X_k}}(x)+(1-\ol\alpha_k)^{-1}I_d \psd0$ for all $x\in\bR^d$ and $k\in[K]$, the last iterate $Y_1$ in the sampling process \eqref{eq:DDIM} satisfies
	\begin{align}\label{eq:convergence}
		\bE\big[\TV(p_{X_1}, p_{Y_1})\big] &\lesssim \frac{d (\log K)^4}{K}+ \veps_\score\sqrt{d}\,(\log K)^{3/2}+\veps_{\jcb}d\log K.
	\end{align}	
Here, the expectation is taken over the training data $\{X^{(i)}\}_{i=1}^n$.
\end{theorem}
In essence, Theorem~\ref{thm:convergence} bounds the TV distance between the distribution of $Y_1$ generated by the sampling process \eqref{eq:DDIM} and that of ${X_1}$ in the forward process \eqref{eq:forward}. The first term corresponds to the discretization error as the sampling process \eqref{eq:DDIM} can be interpreted as the time-discreization of the probability flow ODE \eqref{eq:ODE}; the second term reflects the quality of score matching in terms of the $L^2$ score error and mean Jacobian error. 

The convergence rate $\wt O(d/K)$ achieved here is state-of-the-art among all standard score-based samplers (including SDE-based samplers) \citep{benton2023linear,li2024d,li2024sharp}.
If the iteration number satisfies $K=\wt \Omega(\sqrt{d}/\veps_\score + 1/\veps_\jcb)$, then the sampling error scales as 
$$\wt O\big(\veps_\score\sqrt{d}+\veps_{\jcb}d\big).$$

It is crucial to highlight that our convergence theory is established only under a finite second-moment condition of the initial distribution $p^\star_0$ of the forward process, and does not require smoothness or functional inequality assumptions.
The condition $K^{c_2}\geq \bE[\|X_0\|_2^2]$ requires the second moment to be at most polynomially large in the iteration number $K$, which accommodates extremely large second moments as $c_2$ can be arbitrarily large and $K$ typically increases polynomially with dimension $d$ and sample size $n$. We formulate the condition in terms of $K$ to express our convergence guarantees more concisely. In addition, the assumption $J_{\wh{s}_{X_k}}+(1-\ol\alpha_k)^{-1}I_d \psd0$ guarantees the invertibility of the deterministic mapping in update rule \eqref{eq:DDIM}, serving as an essential component in the analysis for the TV distance bound that aims to control the density ratio between $X_1$ and $Y_1$.

\begin{remark}
Our convergence theory refines the framework developed in \citet{li2024sharp}, which directly controls the TV distance instead of resorting to technical tools such as the Girsanov theorem. 
The definitions of $\veps_\score$ and $\veps_\jcb$ in \eqref{eq:error-score-jab-defn} include extra factors $1-\ol\alpha_k$ compared to the result in \citet[Theorem 1]{li2024sharp}. Furthermore, the additional assumption $J_{\wh{s}_{X_k}}+(1-\ol\alpha_k)^{-1}I_d \psd0$ ensures the mapping in the sampling process is invertible. These critical refinements yield the improved error bound \eqref{eq:convergence} and are fundamental to obtaining the minimax rate.
Similarly, to establish the minimax optimality for stochastic samplers using convergence guarantees developed in \citet{li2024d,li2024provable}, the error bounds therein would require analogous refinements.
\end{remark}

\subsection{Step 2: Score estimation guarantee}

Equipped with Theorem~\ref{thm:convergence}, we now turn to bounding the estimation errors---$\veps_\score$ and $\veps_\jcb$ defined in \eqref{eq:error-score-jab-defn}---of the proposed score estimator $\{\wh{s}_{X_k}(\cdot)\}_{k=1}^K$ in \eqref{eq:score-estimator-X}.

Our main result in this direction is the following theorem, whose proof can be found in Appendix~\ref{sub:proof_of_theorem_ref_thm_score_error_x}.
\begin{theorem}\label{thm:score-error-X}
Assume the target distribution $p^\star$ satisfies Assumptions~\ref{assume:subgaussian} and the learning rates are chosen according to \eqref{eq:learning-rate}. 
Then the score errors $\veps_\score$ and $\veps_\jcb$ (cf.\ \eqref{eq:error-score-jab-defn}) of the score estimator $\{\wh{s}_{X_k}(\cdot)\}_{k=1}^K$ constructed in \eqref{eq:score-estimator-X} satisfy
\begin{subequations}\label{eq:score-Jab-error-X}
\begin{align}
\veps_\score^2 
& \lesssim \frac{C_d'}{n}\bigg\{1 + \sigma^d \bigg(\tau^{-d/2} \wedge \frac{K^{c_0d/2}}{d\log K} \bigg)\bigg\}(\log n)^{d/2+1}  \label{eq:score-error-X};\\ 
\veps_\jcb &
\lesssim \sqrt{\frac{C_d'}n} \biggl\{ 1 + \sigma^{d/2} \bigg(\tau^{-d/4} \wedge \frac{K^{c_0d/4}}{d\log K}\bigg)  \biggr\} (\log n)^{d/4+1} + \frac{C_d'\sigma^d}n\bigg(\tau^{-d/2}\wedge \frac{K^{c_0d/2}}{d\log K}\bigg)(\log n)^{d/2+1}
  \label{eq:Jacobian-error-X},
\end{align}
\end{subequations}
where $C_d' = (4\sqrt{2}/\sqrt{\pi})^d$.
In addition, for all $x\in\bR^d$ and $k\in[K]$, the Jacobian $J_{\wh{s}_{X_k}}(x)$ of the score estimator $\wh{s}_{X_k}$ satisfies
\begin{align}
	J_{\wh{s}_{X_k}}(x)+\frac{1}{1-\ol\alpha_k}I_d \psd 0. \label{eq:PSD-cond}
\end{align}
\end{theorem}

In a nutshell, Theorem~\ref{thm:score-error-X} bounds both the $L^2$ score error $\veps_\score$ and the mean Jacobian error $\veps_\jcb$ in terms of the iteration number $K$, training sample size $n$, dimension $d$, and subgaussian parameter $\sigma$ of the target distribution $p^\star$. In addition, it ensures the positive semidefinite requirement on the Jacobian $J_{\wh{s}_{X_k}}+(1-\ol\alpha_k)^{-1}I_d$ needed in Theorem~\ref{thm:convergence}.

We note that these score estimation guarantees depend on the target distribution solely through the subgaussian property (Assumption~\ref{assume:subgaussian}) and are independent of the H\"older smooth condition (Assumption~\ref{assume:smooth}). Indeed, since each distribution $p_{X_k}$ in the forward process is obtained by convolving the target distribution $p^\star$ with a Gaussian distribution, it naturally acquires the smoothness properties. Even if the score function of $p^\star$ may not be well-defined or grow unbounded, we can still faithfully estimate the sufficiently regular score functions of $p_{X_k}$ for $k\geq 1$.

\paragraph*{Roadmap for the proof of Theorem~\ref{thm:score-error-X}.}
Now, we discuss the proof strategy for Theorem~\ref{thm:score-error-X}. 
 Recall the reduction argument in Section~\ref{sub:preliminaries}.
 The distributional relationship $X_k \disteq \sqrt{\ol\alpha_k}\,Z_{t_k}$ with $t_k= ({1-\ol\alpha_k})/{\ol\alpha_k}+\tau$ allows us to focus on establishing estimation guarantees of the score estimators $\wh{s}_t$ (constructed in \eqref{eq:score-estimate}) for $Z_t$ in the variance exploding process \eqref{eq:SDE-VE}.

We first bound the $L^2$ estimation error of $\wh{s}_t$ for any $t>0$ in the following proposition. The proof is provided in Appendix~\ref{sub:proof_of_lemma_ref_lemma_score-error}.
\begin{proposition}\label{thm:score-error}
Assume the target distribution $p^\star$ satisfies Assumption~\ref{assume:subgaussian}. For any $t>0$,
	\begin{align}
\bE\Big[ \big\| \wh{s}_{t}(Z_t)-s_{t}(Z_t) \big\|_2^2 \Big] \leq 
\frac{C_d}{n}\biggl(\frac{1}{t} + \frac{\sigma^d}{t^{d/2+1}}\biggr)(\log n)^{d/2+1}, \label{eq:score-error}
\end{align}
where $C_d = (4/\sqrt{\pi})^d$ and the expectation is taken over the training data $\{X^{(i)}\}_{i=1}^n$ and $Z_t$.
\end{proposition}

In addition, the distance between the Jacobian matrix of the estimator $\wh s_t$ and that of the true score $s_t$ is also controlled by the following proposition. The detailed proof is deferred to Appendix~\ref{sub:proof_of_lemma_ref_lemma_Jacobian-error}.
\begin{proposition}\label{thm:Jacobian-error}
Assume the target distribution $p^\star$ satisfies Assumptions~\ref{assume:subgaussian}. For any $t>0$,
\begin{align}
\bE\Big[ \big\| J_{\wh{s}_{t}}(Z_t)-J_{t}(Z_t) \big\| \Big]
& \lesssim \sqrt{\frac{C_d}{n}} \bigg(\frac1t + \frac{\sigma^{d/2}}{t^{d/4+1}}\bigg) (\log n )^{d/4+1} + \frac{C_d}n\bigg(\frac1t+\frac{\sigma^d}{t^{d/2+1}}\bigg) (\log n)^{d/2+2} \label{eq:Jacobian-error}
\end{align}
where $C_d = (4/\sqrt{\pi})^d$ and the expectation is taken over the training data $\{X^{(i)}\}_{i=1}^n$ and $Z_t$.
\end{proposition}

Propositions \ref{thm:score-error} and \ref{thm:Jacobian-error} characterize the score estimation errors in terms of the training sample size $n$, time parameter $t$, dimension $d$, and subgaussian parameter $\sigma$ of the target distribution $p^\star$.
Both the $L^2$ score error and the mean Jacobian error decrease as the time parameter $t$ grows. Intuitively, for large $t$, the distribution of $Z_t$ approaches a  Gaussian $\cN(0,tI_d)$, effectively simplifying the score estimation.
While prior works \citep{wibisono2024optimal,zhang2024minimax,dou2024optimal} have derived similar $L^2$ score estimation guarantees using hard-thresholding kernel-based estimators, our results demonstrates that a kernel-based score estimator with soft-thresholding not only achieves comparable $L^2$ error rates but also maintains a similar rate for the crucial mean Jacobian error---a property essential to theoretical guarantees for deterministic ODE-based samplers in Theorem~\ref{thm:convergence}.

Having established Propositions~\ref{thm:score-error}--\ref{thm:Jacobian-error} for the variance-exploding process $(Z_t)_{t>0}$, we can now translate the above score error bounds to derive the score estimation guarantees for the forward process $(X_k)_{k=1}^K$.
Recall the fundamental relationships between score functions and their Jacobians established in \eqref{eq:score-relationship}:
\begin{align*}
	s_{X_k}(x)=\frac1{\sqrt{\ol\alpha_k}}s_{{t_k}}(x/\sqrt{\ol\alpha_k}) \qquad \text{and} \qquad J_{s_{X_k}}(x)= \frac1{\ol\alpha_k}J_{t_k}(x/\sqrt{\ol\alpha_k}).
\end{align*}
Leveraging this with our construction of the score estimator $\wh{s}_{X_k}= \wh{s}_{t_k}(x/\sqrt{\ol\alpha_k})/\sqrt{\ol\alpha_k}$ in \eqref{eq:score-estimator-X} allows us to establish a direct correspondence between $\bE\big[ \| J_{\wh{s}_{X_k}}(X_k)-J_{s_{X_k}}(X_k) \| \big]$ and $\bE\big[ \| J_{\wh{s}_{t_k}}(Z_t)-J_{t_k}(Z_t) \| \big]$. And the $L^2$ score error transfers through an analogous mechanism. Combined with the learning-rate schedule \eqref{eq:learning-rate}, these results immediately yield the desired error bounds \eqref{eq:score-Jab-error-X}, thereby completing the proof of Theorem~\ref{thm:score-error-X}.
\begin{remark}
Prior works on DDPM combine $L^2$ score estimation guarantees with Girsanov's theorem to control the KL divergence between the target distribution and that of the continuous reverse SDE using score estimates \citep{,oko2023diffusion,zhang2024minimax,dou2024optimal}. This approach introduces an integral term $\int_{1}^{T} t^{-1}\diff t$ in the sampling error bound, where $T$ is the end time in the exploding process $(Z_t)_{t>0}$. To prevent this integral from diverging, previous analyses either imposed an upper limit on $T$ \citep{zhang2024minimax} or used the lower bound assumption on the density to derive a faster error rate $t^{-2}$ \citep{dou2024optimal}. 
In contrast, our fine-grained convergence analysis in Theorem~\ref{thm:convergence} and careful choice of learning rates show that the factor $t^{-1}$ in the score estimation error does not cause divergence for arbitrarily large $T$. This also highlights the importance of the convergence analysis in establishing theoretical guarantees for diffusion models.
\end{remark}

It is worth emphasizing that our score estimation guarantees rely exclusively on Assumption~\ref{assume:subgaussian}---that the target distribution $p^\star$ is subgaussian. As discussed in Section~\ref{sec:results}, a key direction for future research would be investigating whether the smoothness property can be leveraged to replace Gaussian kernels with polynomial kernels, potentially achieving optimal score and Jacobian estimation for sufficiently small $t$ (without requiring the density lower bound assumption).

\paragraph*{Roadmap for the proof of Proposition~\ref{thm:Jacobian-error}.} To conclude this section, we sketch the main ideas behind controlling the Jacobian error of $\wh{s}_t$ in Proposition~\ref{thm:Jacobian-error}. A parallel argument can be used to establish the $L^2$ score error in Proposition~\ref{thm:score-error}.

Recall the sets $\cE_t(x)$ and $\cF_t$ from \eqref{eq:event-E} and \eqref{eq:rho-t}, respectively. Let us use them to decompose 
\begin{align*}
& \bE\Big[\big\| J_{\wh{s}_{t}}(Z_t)-J_t(Z_t) \big\|\Big] = \int_{\bR^d} \bE\Big[ \big\| J_{\wh{s}_t}(x)-J_t(x) \big\| \Big] p_{t}(x) \diff x \notag\\
& \qquad = \underbrace{\int_{\cF_t} \bE\Big[ \big\| J_{\wh{s}_t}(x)-J_t(x) \big\| \ind\big\{\cE_t(x)\big\} \Big] p_{t}(x) \diff x}_{\defn \xi_1}
+ \underbrace{\int_{\cF_t^\setc} \bE\Big[ \big\| J_{\wh{s}_t}(x)-J_t(x) \big\|\ind\big\{\cE_t(x)\big\} \Big] p_{t}(x) \diff x}_{\defn \xi_2} \notag\\
& \qquad \quad + \underbrace{\int_{\bR^d} \bE\Big[ \big\| J_{\wh{s}_t}(x)-J_t(x) \big\| \ind\big\{\cE_t^\setc(x)\big\} \Big] p_{t}(x) \diff x}_{\defn \xi_3}.
\end{align*}
In what follows, we handle the three terms  $\xi_1$, $\xi_2$, and $\xi_3$ separately.
\begin{itemize}
	\item For any $x\in\cF_t$ (high-density region under $p_t$) such that $p_t(x)\geq c_\eta \eta_t \geq 2\eta_t$ (see \eqref{eq:rho-t}), on the typical event $\cE_t(x)$, we can show that the density estimator $\wh{p}_t(x)$ satisfies $\wh{p}_t(x) \geq p_t(x)/2 \geq \eta_t$. By the expressions of $J_{\wh{s}_t}$ and $J_{t}$ in \eqref{eq:J-1} and \eqref{eq:Jacobian-expression}, respectively, we can bound the Jacobian error by decomposing it:
\begin{align}
	\big\|J_{\wh{s}_t}(x) - J_{t}(x)\big\| 
	& \lesssim {\frac{\|\wh{H}_t(x)-H_t(x)\|}{{p}_t(x)}}
	+ {\frac{|\wh{p}_t(x)-p_t(x)|}{{p}_t(x)}\frac{\|H_t(x)\|}{p_t(x)}}
	+ {\big\|\wh{s}_t(x)-s_t(x)\big\|_2 \big\|\wh{s}_t(x)+s_t(x)\big\|_2} \notag\\
	& \lesssim {\frac{\|\wh{H}_t(x)-H_t(x)\|}{{p}_t(x)}}
	+ {\frac{|\wh{p}_t(x)-p_t(x)|}{{p}_t(x)}\frac{\|H_t(x)\|}{p_t(x)}} + \big\|\wh{s}_t(x)-s_t(x)\big\|_2^2
	 \notag \\
	&\quad + \frac{\|\wh{g}_t(x)-g_t(x)\|_2}{p_t(x)}\|s_t(x)\|_2   + \frac{|\wh{p}_t(x)-p_{t}(x)|}{p_t(x)}  \|s_{t}(x)\|_2^2. \label{eq:jac-decomp} 
\end{align}
Here, we recall the definitions of $\wh{g}_t\defn \gd \wh{p}_t$ and $\wh{H}_t$ defined in \eqref{eq:g-hat} and \eqref{eq:Phi-hat}, respectively.
This suggests we need to control the expected errors of $\wh{p}_t$, $\wh{g}_t$, and $\wh{H}_t$, which are established in the following lemma. The proof is provided in Appendix~\ref{sub:proof_of_lemma_ref_lemma_mse-density}.
\begin{lemma}\label{lemma:MSE-density}
Recall the definitions of $g_t$ and $H_t$ in \eqref{eq:g-hat} and \eqref{eq:Phi-def}, respectively.
For all $x\in\bR^d$, one has
\begin{subequations}\label{eq:MSE}
	\begin{align}
		\bE\Big[ \big|\wh{p}_t(x)-p_t(x)\big|^2 \Big] & 
	\leq \frac{p_{t}(x)}{n(2\pi t)^{d/2}} \label{eq:MSE-pdf};\\
	\bE\Big[ \big\|\wh{g}_t(x)-g_t(x)\big\|_2^2 \Big] & \leq \frac{p_{t}(x)}{n(2\pi t)^{d/2}t} \label{eq:MSE-g},\\
	\bE\Big[\big\|\wh{H}_t(x)-H_t(x)\big\|\Big] &\lesssim \frac1{n(2\pi t)^{d/2}t} + \frac1t\sqrt{\frac{p_{t}(x)}{n(2\pi t)^{d/2}}}
	 \label{eq:Phi-error-exp-UB}.
\end{align}
\end{subequations}
Moreover, the event $\cE_t(x)$ defined in \eqref{eq:event-E} obeys
	\begin{align}
		\bP\big(\cE_t^\setc(x) \big) \lesssim n^{-10}. \label{eq:density-concen-ineq}
	\end{align}
\end{lemma}
\begin{remark}
Since the target density is allowed to be vanishingly small $p^\star(x)$, the term $p_t(x)$ appearing on the numerator on the right-hand side of the above bounds is essential to obtaining sharp bounds.
\end{remark}

In addition, both $\|s_t(x)\|_2$ and  $\|H_t(x)\|$ associated with the true score are well-controlled on the high-density set $\cF_t$. This is formalized by the following lemma; with proof postponed to Appendix~\ref{sub:proof_of_lemma_ref_lemma_score_jacobian_ub_prob_high}.
\begin{lemma}\label{lemma:score-Jacobian-UB-prob-high}
	For any $x\in\cF_t$, one has
	\begin{subequations}\label{eq:score-Jacobian-UB-prob-high}
		\begin{align}
		\|s_{t}(x)\|_2 &\lesssim \sqrt{\frac{\log n}{t}} \label{eq:score-UB-prob-high},\\
		\big\|J_{t}(x)\big\| &\lesssim \frac{\log n}{t} \label{eq:Jacobian-UB-prob-high},\\
		\big\|H_t(x)\big\| &\lesssim \frac{p_{t}(x)}{t}\log n  \label{eq:Sigma-UB-prob-high}.
	\end{align}
	\end{subequations}
\end{lemma}


Combining Lemmas~\ref{lemma:MSE-density}--\ref{lemma:score-Jacobian-UB-prob-high} with the fact that the volume of the set $\cF_t$ is relatively small due to the subgaussian property of $p_t$ and our choice of the threshold $\eta_t$ in \eqref{eq:score-estimate}, we can use \eqref{eq:jac-decomp} to show that
\begin{align*}
	\xi_1 \lesssim \sqrt{\frac{C_d}{n}} \bigg(\frac1t + \frac{\sigma^{d/2}}{t^{d/4+1}}\bigg) (\log n )^{d/4+1}
\end{align*}
where $C_d = (4/\sqrt{\pi})^d$.
\item We proceed to consider the second term $\xi_2$. By construction, the Jacobian of the score estimator satisfies that $J_{\wh{s}_{t}}(x) = 0$ when $\wh{p}_{t}(x) \leq \eta_t / 2$ (see \eqref{eq:J-2}). This allows us to bound
\begin{align*}
	\xi_2 \leq {\int_{\cF_t^\setc} \bE\Big[ \| J_{\wh{s}_t}(x) \| \ind\big\{\wh p_{t}(x) \geq \eta_t/2 \big\} \Big] p_{t}(x) \diff x}+{\int_{\cF_t^\setc} \big\|J_{t}(x) \big\| \,p_{t}(x) \diff x}.
\end{align*}

Next, when $\wh p_{t}\geq \eta_t/2$, the proposed smooth score estimator $\wh{s}_t(x)$, which incorporates the soft-thresholding function $\psi(\wh p_{t};\eta_t)$, enjoys the desired smoothness property (as characterized in \eqref{eq:J-1} and \eqref{eq:J-3}). Consequently, the spectral norm of its Jacobian $\|J_{\wh{s}_t}\|$ remains well controlled. Hence, one can exploit the subgaussian property of $p_t$ to upper bound the integral over the low-density region under $p_t$.
Putting these two observations together, we can derive
\begin{align*}
	\xi_2 \leq \frac{C_d}{n} \bigg(1+\frac{\sigma^d}{t^{d/2+1}}\bigg) (\log n)^{d/2+2}.
\end{align*}

\item Finally, it can be shown that $\|J_{\wh{s}_t}\|$ and $\|J_{{s}_t}\|$ do not blow up significantly. Combining this with $\bP(\cE_t^\setc)\lesssim K^{-10}$ in \eqref{eq:density-concen-ineq}, we can show that the last term $\xi_3$ is negligible. 

\item Summing the respective bounds for $\xi_1$, $\xi_2$, and $\xi_3$ completes the proof of Proposition~\ref{thm:Jacobian-error}.
\end{itemize}

\subsection{Step 3: Bias in the initial distribution $p_0^\star$}
Once Theorem~\ref{thm:score-error-X} is combined with Theorem~\ref{thm:convergence}, we immediately obtain a bound on $\TV(p_{X_1},p_{Y_1})$, or equivalently, $\TV(p_{X_1/\sqrt{\alpha_1}},p_{Y})$ since $Y=Y_1/\sqrt{\alpha_1}$ is our generated sample. The remaining step is to control the TV distance between $p_{X_1/\sqrt{\alpha_1}}$ and the target distribution $p^\star$. 
Recall that $X_1/\sqrt{\alpha_1}$ can be written as $X_0+\sqrt{1-\alpha_1/\alpha_1}\,W_1'$ for some standard Gaussian random vector $W_1'\sim\cN(0,I_d)$ independent of $X_0\sim p_0^\star = p^\star \ast \cN(0,\tau I_d)$.
Because the distribution of ${X_1/\sqrt{\alpha_1}}$ can be arbitrarily close to the initial distribution $p_0^\star$ as the iteration number $K$ increases and $\alpha_1$ approaches zero, the TV distance fundamentally depends on the bias in $p_0^\star$ with respect to the target distribution $p^\star$. The following proposition formalizes this relationship, with the complete proof presented in Appendix \ref{sub:proof_of_theorem_ref_thm_early_stopping}.
\begin{proposition}\label{thm:early-stopping}
Assume the target distribution $p^\star$ satisfies Assumptions~\ref{assume:subgaussian}--\ref{assume:smooth}. If $0<1-\alpha_1+\tau<1$, then
	\begin{align}
		\TV(p_{X_1/\sqrt{\alpha_1}},\,p^\star) \leq C_2 (1-\alpha_1+\tau)^{\frac \beta2 \wedge 1} \log^{\frac{d}{2}}\big(1/{(1-\alpha_1+\tau)}\big) \label{eq:early-claim}
	\end{align}
for some constant $C_2>0$ depending only on $d,\beta,L$, and $\sigma$.
\end{proposition}

Proposition~\ref{thm:early-stopping} rigorously justifies that, if $1-\alpha_1$ and $\tau$ are sufficiently small, the distribution of $X_1/\sqrt{\alpha_1}$ is extremely close to the target distribution $p^\star$.
In addition, when $\beta \leq 2$, the error bound in \eqref{eq:early-claim} becomes $\wt O\big((1-\alpha_1+\tau)^{\beta/2}\big)$, which allows us to achieve the near-optimal rate in Theorem~\ref{thm:TV}.
As a remark, Proposition~\ref{thm:early-stopping} is the only component in the proof of Theorem~\ref{thm:TV} that requires the H\"older smoothness condition on the target distribution $p^\star$.


\subsection{Proof of Theorem~\ref{thm:TV}}\label{sub:proof_of_theorem_ref_thm_tv}

Putting all the pieces together in Steps 1--3, we are now ready to prove Theorem~\ref{thm:TV}.
\begin{itemize}
	\item 
Substituting our selected $\tau=n^{-2/(d+2\beta)}$ into Theorem~\ref{thm:score-error-X}, we can characterize the estimation errors $\veps_\score$ and $\veps_\jcb$ (defined in \eqref{eq:error-score-jab-defn}) of our proposed score estimator $\{\wh{s}_{X_k}(\cdot)\}_{k=1}^K$ in \eqref{eq:score-estimator-X}:
\begin{subequations}\label{eq:first-step-sc-jcb-error}
\begin{align}
 \veps_\score^2 & \lesssim \frac{C_d'}n\bigg\{1 + \sigma^d\bigg( n^{\frac{d}{d+2\beta}}\wedge \frac{K^{\frac{c_0d}2}}{d\log K} \bigg)\bigg\}(\log n)^{\frac d2+1} \lesssim C'_d \sigma^d n^{-\frac{2\beta}{d+2\beta}}(\log n)^{\frac d2+1}, \label{eq:first-step-error1}
\end{align}
and
\begin{align}
	\veps_\jcb & \lesssim \sqrt{\frac{C_d'}n} \biggl\{ 1 + \sigma^{\frac d2} \bigg(n^{\frac{d}{2(d+2\beta)}} \wedge \frac{K^{\frac{c_0d}4}}{d\log K}\bigg)  \biggr\} (\log n)^{\frac d4+1} + \frac{C_d'\sigma^d}n\bigg(n^{\frac{d}{d+2\beta}}\wedge \frac{K^{\frac{c_0d}2}}{d\log K}\bigg)(\log n)^{\frac d2+1} \notag \\ 
	& \lesssim \sqrt{C_d'}\,\sigma^{\frac d2}  n^{\frac{d}{2(d+2\beta)}}  (\log n)^{\frac d4+1},\label{eq:first-step-error2}
\end{align}
\end{subequations}
for $n$ large enough.
\item
Next, in order to apply Theorem~\ref{thm:convergence}, let us verify its conditions.
First, the subgaussian property of $p^\star$ (Assumption~\ref{assume:subgaussian}) implies that $p_0^\star=p^\star\ast \cN(0,\tau I_d)$ is $\sqrt{\sigma^2+\tau}$-subgaussian and $\bE[\|X_0\|_2^2]\leq d(\sigma^2 + \tau)\leq d(\sigma^2 + n^{-2/(d+2\beta)})$. Hence, the requirement on the iteration number $K\geq n^{\beta/(d+2\beta)}(\log K)^3$ ensures both $K\gtrsim d^2 (\log K)^5$ and $K\geq \bE[\|X_0\|_2^2]$  (simply setting $c_2=1$) for $n$ large enough. Also, $c_1=5c_0 \vee 12$ chosen in the learning rate schedule satisfies $c_1/8-c_2/2\geq 1$, and the positive semidefinite requirement $J_{\wh{s}_{X_k}}+(1-\ol\alpha_k)^{-1}I_d$ has been verified in \eqref{eq:PSD-cond} of Theorem~\ref{thm:score-error-X}.
Therefore, we can invoke Theorem~\ref{thm:convergence} to obtain
\begin{align}
	&\bE\big[\TV(p_{X_1/\sqrt{\alpha_1}}, p_{Y_1/\sqrt{\alpha_1}})\big] = \bE\big[\TV(p_{X_1}, p_{Y_1})\big] \notag \\
	&\qquad \lesssim \frac{d\,(\log K)^4}{K}+ \sqrt{d}\,\veps_{\score}(\log K)^{\frac32} +  d\,\veps_{\jcb}\log K \notag \\ 
	&\qquad\lesssim d n^{-\frac{\beta}{d+2\beta}} \log K + \sqrt{C'_d d} \,\sigma^{\frac d2}n^{-\frac{\beta}{d+2\beta}} (\log n)^{\frac d4+\frac12}(\log K)^{\frac32}  + \sqrt{C'_d} \,d \sigma^{\frac d2}n^{-\frac{\beta}{d+2\beta}} (\log n)^{\frac d4+1}\log K \notag \\ 
	&\qquad \leq C_1 n^{-\frac{\beta}{d+2\beta}} (\log n)^{\frac d4+1}\log K, \label{eq:first-step-error}
\end{align}
for some constant $C_1>0$ depending only  on $d,\beta$, and $\sigma$. Here, the third line holds due to $K\geq n^{\beta/(d+2\beta)}(\log K)^3$ and \eqref{eq:first-step-sc-jcb-error}.
\item 
In addition, since $K\geq n^{{\beta}/{(d+2\beta)}}$, $1-\alpha_1=K^{-c_0}$ with $c_0\geq 2/\beta$, and $\tau = n^{-2/(d+2\beta)}$, one has
\begin{align*}
	n^{-\frac{2}{d+2\beta}} < 1-\alpha_1+\tau\leq 2n^{-\frac{2}{d+2\beta}} < 1
\end{align*}
for $n$ large enough. Applying Proposition~\ref{thm:early-stopping} with $\beta\leq 2$ yields
\begin{align}
	\TV(p_{X_1/\sqrt{\alpha_1}},\,p^\star)
&\leq C_2 \big(2n^{-\frac{2}{d+2\beta}}\big)^{\frac\beta2} \big(\log (n^{\frac{2}{d+2\beta}})\big)^{\frac d2}
	\leq C_2  2^{\frac \beta 2} \big({2}/({d+2\beta})\big)^{\frac d2} n^{-\frac{\beta}{d+2\beta}} (\log n)^{\frac{d}{2}} \notag \\ 
	&\lesssim C_2  n^{-\frac{\beta}{d+2\beta}} (\log n)^{\frac{d}{2}} \label{eq:early-stopping-error}
\end{align}
where the last step holds as $x^{-x}\leq 2$ for any $x>0$
 and $\beta\leq 2$.

\item
Finally, applying the triangle inequality to \eqref{eq:first-step-error} and \eqref{eq:early-stopping-error} finishes the proof of Theorem~\ref{thm:TV}.
\end{itemize}

%% file: discussion.tex

\section{Discussion}\label{sec:discussion}

In this paper, we have made progress towards understanding the theoretical guarantees of probability flow ODE-based sampler. We have developed an end-to-end performance guarantee for ODE-based samplers in the context of learning subgaussian distributions with H\"older smooth densities. 
By combining a smooth regularized score estimator with a refined convergence analysis of the ODE-based sampling process, our result is the first to prove the near-minimax optimality of ODE-based samplers. Our analysis framework not only addresses the challenges of ODE-based samplers but also offers promising insights that could be extended to other variants of score-based generative models.

Our study opens several compelling directions for future research. First, our current sampling error bound suffers from the curse of dimensionality and is likely suboptimal with respect to the dependence on the logarithmic factors. Future work should investigate refining this dependence on the dimension $d$, possibly by exploiting intrinsic low-dimensional structures underlying the target data. Second, since practical implementations typically employ neural network-based score functions through empirical risk minimization, exploring the Jacobian estimation guarantees and developing rigorous end-to-end performance guarantees for such samplers represents a critical theoretical challenge.
   Third, extending our theoretical framework to encompass broader distribution classes would significantly enhance applicability. Beyond the smoother densities ($\beta>2$) previously discussed, it would also be important to consider heavy-tailed and non-smooth densities that frequently arise in many real-world applications, which requires novel theoretical approaches to handle their distinct analytical challenges.

%% file: proof-convergence.tex
\subsection{Analysis for convergence guarantee (proof of Theorem~\ref{thm:convergence})}\label{sub:proof_of_theorem_ref_thm_convergence}

We prove Theorem~\ref{thm:convergence} using the framework developed in \citet{li2024sharp}, which directly controls the TV distance between $p_{X_1}$ and $p_{Y_1}$.

For ease of presentation, we begin by introducing some notation. Recall the sampling update rule \eqref{eq:DDIM}. For each $k\in[K]$ and $x\in\bR^d$, define
\begin{align}
	\phi_k^\star(x)\defn x + \frac{1-\alpha_k}2 s_{X_k}(x)\quad\text{and}\quad \phi_k(x)\defn x + \frac{1-\alpha_k}2 \wh{s}_{X_k}(x) \label{eq:phi-phi-star-defn}
\end{align}
This allows us to write the update rule $Y_{k-1}=\phi_k(Y_k)/\sqrt{\alpha_k}$.
We also let
\begin{align}
	\theta_k(x) \defn -\frac{\log p_{X_k}(x)}{d\log K} \vee \big(2c_2+c_0\big) \label{eqn:choice-y}.
\end{align}
In addition, we define the pointwise score error and Jacobian error as
\begin{align}
	\veps^2_{\score,k}(x)\defn \big\| \wh{s}_{X_k}(x)-s_{X_k}(x) \big\|_2^2 \quad \text{and} \quad \veps_{\jcb,k}(x)\defn \big\| J_{\wh{s}_{X_k}}(x)-J_{s_{X_k}}(x) \big\|. \label{eq:score-error-pointwise}
\end{align}
Finally, for two functions $f(K),g(K) > 0$, we use $f(K)\ll g(K)$ to mean $f(K)\leq cg(K)$ for some absolute constant $c>0$ that is sufficiently small.

Now, we begin the proof by recalling that the TV distance can be written as
\begin{align}
\mathsf{TV}\big(p_{X_1},p_{Y_1}\big) &= \int_{p_{X_1}(y) > p_{Y_1}(y)}\big(p_{X_1}(y) - p_{Y_1}(y)\big)\diff  y = \int_{p_{X_1}(y) > p_{Y_1}(y)}\bigg(\frac{p_{X_1}(y)}{p_{Y_1}(y)} - 1\bigg)p_{Y_1}(y)\diff  y
	\label{eqn:ode-tv-123}.
\end{align}
Hence, to bound this TV distance, it is natural to focus on controlling the ratio ${p_{X_1}(y)}/{p_{Y_1}(y)}$.
For each $k=2,3,\dots,K$, the ratio can be expressed as
\begin{align}
\frac{p_{Y_{k-1}}(y_{k-1})}{p_{X_{k-1}}(y_{k-1})} & =\frac{p_{\sqrt{\alpha_k}Y_{k-1}}(\sqrt{\alpha_k}y_{k-1})}{p_{\sqrt{\alpha_k}X_{k-1}}(\sqrt{\alpha_k}y_{k-1})}\notag\\
 & =\frac{p_{\sqrt{\alpha_k}Y_{k-1}}(\sqrt{\alpha_k}y_{k-1})}{p_{Y_k}(y_k)}\cdot\bigg(\frac{p_{\sqrt{\alpha_k}X_{k-1}}(\sqrt{\alpha_k}y_{k-1})}{p_{X_k}(y_k)}\bigg)^{-1}\cdot\frac{p_{Y_k}(y_k)}{p_{X_k}(y_k)} \notag\\ 
 &=\frac{p_{\phi_k(Y_{k})}(\phi_k(y_k))}{p_{Y_k}(y_k)}\cdot\bigg(\frac{p_{\sqrt{\alpha_k}X_{k-1}}(\phi_k(y_k))}{p_{X_k}(y_k)}\bigg)^{-1} \cdot\frac{p_{Y_k}(y_k)}{p_{X_k}(y_k)} \qquad \forall y_{k-1},y_k\in\bR^d, \label{eq:recursion}
\end{align}
where the last line uses the notation \eqref{eq:phi-phi-star-defn}.
As an important remark, the assumption $J_{\wh{s}_{X_k}}+(1-\ol\alpha_k)^{-1}I_d \psd0$ ensures that the mapping $\phi_k$ is invertible.

The next step is to control the term
\begin{align*}
	\frac{p_{\phi_k(Y_{k})}(\phi_k(y_k))}{p_{Y_k}(y_k)}\cdot\bigg(\frac{p_{\sqrt{\alpha_k}X_{k-1}}(\phi_k(y_k))}{p_{X_k}(y_k)}\bigg)^{-1}
\end{align*}
for typical $y_k\in\bR^d$. We accomplish this via the following critical lemma, which yields the more refined convergence results in Theorem~\ref{thm:convergence} than those in \citet[Theorem 1]{li2024sharp}. The proof is deferred to the end of this section.
\begin{lemma}\label{lemma:density-ratio}
Suppose $K\gtrsim d\log^3 K$.
For each $2\leq k\leq K$, there exists a function $\zeta_k(\cdot):\bR^d\to\bR$ such that for any $x \in \bR^d$ satisfies $\theta_k(x)\lesssim 1$, 
$
{\sqrt{\theta_k(x)d(1-\ol\alpha_k)}\varepsilon_{\score,k}(x)\log^{3/2}K}\ll{K}
$ 
and
$
d(1-\ol\alpha_k)\varepsilon_{\jcb,k}(x)\log K  \ll {K},
$
one has
\begin{align}
&\frac{p_{\phi_k(Y_k)}(\phi_k(x))}{p_{Y_k}(x)}\bigg(\frac{p_{\sqrt{\alpha_k}X_{k-1}}\big(\phi_k(x)\big)}{p_{X_k}(x)}\bigg)^{-1}  \notag\\ 
 & \qquad=1+\zeta_k(x)+O\bigg(\Big\|\frac{\partial\phi_k^{\star}(x)}{\partial x}-I_d\Big\|_{\mathrm{F}}^{2}\bigg) \notag \\ 
 &\qquad\quad +O\bigg(\frac{\sqrt{d(1-\ol\alpha_k)}\,\varepsilon_{\score,k}(x)\log^{3/2}K}{K}+ \frac{d(1-\ol\alpha_k) \varepsilon_{\jcb,k}(x)\log K}{K}+\frac{d\log^{3}K}{K^2}\bigg).
 \label{eq:yt}
\end{align}
Moreover, $\zeta_k$ satisfies $\zeta_k(x) \le 0$ and
\begin{align}
\bE\Big[\big|\zeta_{k}(X_k)\big|\Big]
	\lesssim
	\bE\left[\Big\|\frac{\partial\phi_{k}^{\star}}{\partial x}(X_k)-I_d\Big\|_{\mathrm{F}}^{2}\right]+
	\frac{d\log^{3}K}{K^{2}},  \label{eq:yt2}
\end{align} 
provided that $K \gtrsim d^{2}\log^{5}K$.
\end{lemma}

Using Lemma~\ref{lemma:density-ratio} in place of \citet[Lemma 5]{li2024sharp}, one can repeat the argument for \citet[Theorem 1]{li2024sharp}  (see \citet[Section 5]{li2024sharp}) to complete the proof of Theorem~\ref{thm:convergence}. For the sake of clarity, we provide a high-level proof sketch below. 

Consider a ``typical'' sequence $(y_K,y_{K-1},\dots,y_1)$ generated by the sampling process \eqref{eq:DDIM} from an initialization $y_K$. From \eqref{eq:yt} in Lemma~\ref{lemma:density-ratio} and the recursion \eqref{eq:recursion}, one sees that for each $2\leq k\leq K$, typical points $y_{k-1},y_k$ satisfy
\begin{align*}
	\frac{p_{Y_{k-1}}(y_{k-1})}{p_{X_{k-1}}(y_{k-1})} \approx \frac{p_{Y_k}(y_k)}{p_{X_k}(y_k)},
\end{align*}
by chaining these ratios, one obtains
\begin{align}
\frac{p_{X_1}(y_{1})}{p_{Y_1}(y_{1})} & = \bigg\{ 1+O\bigg(\frac{d\log^{4}K}{K} + \sum_{k=2}^K|\zeta_k(y_k)| + \sum_{k=2}^K\Big\|\frac{\partial \phi^{\star}_k(y_k)}{\partial x} - I_d\Big\|_{\mathrm{F}}^2+S_K(y_K)\bigg)\bigg\} \frac{p_{X_K}(y_K)}{p_{Y_K}(y_K)},	\label{eq:pt-qt-equiv-ODE-St-taui}
\end{align}
where
\begin{align*}
	S_{K}(y_K) &\coloneqq \frac{\log K}{K}\sum_{k=2}^K \Big(\sqrt{d(1-\ol\alpha_k)\log K}\,\varepsilon_{\score, k}(y_k)+d(1-\ol\alpha_k)\varepsilon_{\jcb, k}(y_k)\Big).
\end{align*}
In addition, one can show that
\begin{align}
	\sum_{k=2}^K\bE\bigg[\Big\|\frac{\partial\phi_{k}^{\star}}{\partial x}(X_k)-I_d\Big\|_{\mathrm{F}}^{2}\bigg] \lesssim \frac{d\log^2 K}{K}. \label{eq:phi-star-gd-frob}
\end{align}

Returning to the integral expression for the TV distance in \eqref{eqn:ode-tv-123}, we can split the domain of integration into some ``typical'' set $\cH$ and its complement.
For typical point $y_1$, we can apply \eqref{eq:pt-qt-equiv-ODE-St-taui} to obtain
\begin{align*}
	& \int_{y_1\in\cH}\bigg(\frac{p_{X_1}(y_1)}{p_{Y_1}(y_1)} - 1\bigg)p_{Y_1}(y_1)\diff  y_1\\ 
	& \qquad \numpf{i}{=} \int_{y_K\in\cH'}O\bigg(\frac{d\log^{4}K}{K} + \sum_{k=2}^K|\zeta_k(y_k)| + \sum_{k=2}^K\Big\|\frac{\partial \phi^{\star}_k(y_k)}{\partial x} - I_d\Big\|_{\mathrm{F}}^2+S_{K}(y_{K})\bigg) \frac{p_{X_K}(y_{K})}{p_{Y_K}(y_{K})} p_{Y_K}(y_{K}) \diff y_K\\ 
	&\qquad\quad + \int_{y_K\in\cH'} \bigg(\frac{p_{X_K}(y_{K})}{p_{Y_K}(y_{K})}-1\bigg)p_{Y_K}(y_{K}) \diff y_K \nonumber\\
	& \qquad \,\lesssim \TV(p_{X_K},p_{Y_K}) + \frac{d\log^{4}K}{K} +  \int_{y_K\in\bR^d}\bigg(\sum_{k=2}^K|\zeta_k(y_k)| + \sum_{k=2}^K\Big\|\frac{\partial \phi^{\star}_k(y_k)}{\partial x} - I_d\Big\|_{\mathrm{F}}^2+S_{K}(y_{K})\bigg) p_{X_K}(y_{K}) \diff y_K \nonumber\\
	&\qquad \numpf{ii}{\lesssim}
	\frac{1}{K^{c_1/8-c_2/2}} +\frac{d\log^4 K}{K}
	+ \sum_{k=2}^K\bE\Big[\big|\zeta_{k}(X_k)\big|\Big]
	+ \sum_{k=2}^K\bE\bigg[\Big\|\frac{\partial\phi_{k}^{\star}}{\partial x}(X_k)-I_d\Big\|_{\mathrm{F}}^{2}\bigg]\\ 
	& \qquad \quad +\frac{\sqrt{d}\,\log^{3/2} K}{K}\sum_{k=2}^K\sqrt{(1-\ol\alpha_k)}\,\bE\big[\varepsilon_{\score, k}(X_k)\big] + \frac{d\log K}{K}\sum_{k=2}^K(1-\ol\alpha_k)\bE\big[\varepsilon_{\jcb, k}(X_k)\big] \\
	& \qquad \numpf{iii}{\lesssim} \frac{d\log^4 K}{K} +\frac{\sqrt{d}\,\log^{3/2} K}{K}\sum_{k=1}^K\sqrt{(1-\ol\alpha_k)}\,\bE\big[\varepsilon_{\score, k}(X_k)\big] + \frac{d\log K}{K}\sum_{k=1}^K(1-\ol\alpha_k)\bE\big[\varepsilon_{\jcb, k}(X_k)\big],
\end{align*}
where (i) is true since the randomness of $Y_1$ is due to $Y_K$ and we choose set $\cH'$ such that $\{y_K\in\cH'\}=\{y_1\in\cH\}$; (ii) uses Lemma~\ref{lemma:mixing} and the condition $\bE[\|X_0\|_2^2]\leq K^{c_2}$;
(iii) uses the condition $c_1/8-c_2/2\geq 1$ as well as \eqref{eq:yt2} and \eqref{eq:phi-star-gd-frob}. We note that the above expectations are taken over the randomness of $(X_k)$.

Meanwhile, one can show that the contribution from remaining atypical points is negligible. Putting these two observations together and taking the expectation over the score estimator $\{\wh{s}_{X_k}\}_{k=1}^K$, we find that
\begin{align*}
\bE\big[\TV(p_{X_1}, p_{Y_1})\big] & \lesssim \frac{d\log^4 K}{K} +\sqrt{d}\,\log^{3/2} K\sqrt{\frac{1}{K}\sum_{k=1}^K(1-\ol\alpha_k)\varepsilon_{\score, k}^2} + \frac{d\log K}{K}\sum_{k=2}^K(1-\ol\alpha_k)\varepsilon_{\jcb, k} \\ 
& = \frac{d\log^4 K}{K} +\sqrt{d}\,\varepsilon_{\score}\log^{3/2} K + d\varepsilon_{\jcb}\log K,
\end{align*}
where the first step applies Jensen's inequality, and the last line follows from the definitions of $\varepsilon_{\score, k}$ and $\varepsilon_{\jcb, k}$ in \eqref{eq:error-score-jab-defn}. This concludes the proof sketch of Theorem~\ref{thm:convergence}.

\paragraph*{Proof of Lemma~\ref{lemma:density-ratio}.}
Fix an arbitrary $x\in\bR^d$ such that $\theta_k(x) \lesssim 1$. Recalling the definition of $\phi_k$ in \eqref{eq:phi-phi-star-defn}, we further define
\begin{align}
	u &\coloneqq x - \phi_k(x) = x - \phi_k^\star(x) - \big(\phi_k(x) - \phi_k^\star(x)\big) = -\frac{1-\alpha_k}2 s_{X_k}(x)- \big(\phi_k(x) - \phi_k^\star(x)\big) \notag\\
	&= \frac{1-\alpha_k}{2(1-\overline{\alpha}_k)}\bE\big[X_k - \sqrt{\overline{\alpha}_k}X_0 \mid X_k=x\big]-\frac{1-\alpha_k}{2}\big(\wh{s}_{X_k}(x) - s_{X_k}(x)\big)
	\label{eq:defn-u-Lemma-main-ODE}.
\end{align}
Notice that $\sqrt{\alpha_k} X_{k-1}$ satisfies
\begin{equation}
	\sqrt{\alpha_k} X_{k-1} = \sqrt{\alpha_k} \big(\sqrt{\overline{\alpha}_{k-1}}X_0+ \sqrt{1-\overline{\alpha}_{k-1}}\,W_{k-1}'\big) 
	= \sqrt{\overline{\alpha}_k}X_0+ \sqrt{\alpha_k-\overline{\alpha}_k}\,W_{k-1}'
	\notag
\end{equation}
where $W_{k-1}'\sim \cN(0,I_d)$ is some $d$-dimensional standard Gaussian random vector independent of $X_0$.
This allows us to express
\begin{align}
&\frac{p_{\sqrt{\alpha_k}X_{k-1}}\big(\phi_k(x)\big)}{p_{X_k}(x)}    
=\frac{1}{p_{X_k}(x)}\int_{x_{0}}p_{X_{0}}(x_{0}) p_{\sqrt{\alpha_k-\overline{\alpha}_k}W}\big(\phi_k(x)-\sqrt{\overline{\alpha}_k}x_{0}\big)\mathrm{d}x_{0}  \notag\\
& \quad =\frac{1}{p_{X_k}(x)}\int_{x_{0}}p_{X_{0}}(x_{0}){\big(2\pi(\alpha_k-\overline{\alpha}_k)\big)^{-d/2}}\exp\biggl(-\frac{\big\|\phi_k(x)-\sqrt{\overline{\alpha}_k}x_{0}\big\|_{2}^{2}}{2(\alpha_k-\overline{\alpha}_k)}\bigg)\mathrm{d}x_{0}\notag\\
 &\quad  \overset{\mathrm{(i)}}{=}\frac{1}{p_{X_k}(x)}\int_{x_{0}}p_{X_{0}}(x_{0}){\big(2\pi(\alpha_k-\overline{\alpha}_k)\big)^{-d/2}}\exp\biggl(-\frac{\big\| x-\sqrt{\overline{\alpha}_k}x_{0}\big\|_{2}^{2}}{2(1-\overline{\alpha}_k)}\bigg)\notag\\
 & \qquad\qquad\qquad\qquad \cdot\exp\biggl(-\frac{(1-\alpha_k)\big\| x-\sqrt{\overline{\alpha}_k}x_{0}\big\|_{2}^{2}}{2(\alpha_k-\overline{\alpha}_k)(1-\overline{\alpha}_k)}-\frac{\|u\|_{2}^{2}-2u^{\top}\big(x-\sqrt{\overline{\alpha}_k}x_{0}\big)}{2(\alpha_k-\overline{\alpha}_k)}\bigg)\mathrm{d}x_{0}\notag\\
 & \quad \overset{\mathrm{(ii)}}{=}\biggl(\frac{1-\overline{\alpha}_k}{\alpha_k-\overline{\alpha}_k}\biggr)^{d/2}
 \int_{x_{0}}p_{X_{0}\mid X_k}(x_{0}\mid x)\exp\biggl(-\frac{(1-\alpha_k)\big\| x-\sqrt{\overline{\alpha}_k}x_{0}\big\|_{2}^{2}}{2(\alpha_k-\overline{\alpha}_k)(1-\overline{\alpha}_k)}-\frac{\|u\|_{2}^{2}-2u^{\top}(x-\sqrt{\overline{\alpha}_k}x_{0})}{2(\alpha_k-\overline{\alpha}_k)}\bigg)\mathrm{d}x_{0} 
	\label{eqn:fei},
\end{align}
where (i) uses \eqref{eq:defn-u-Lemma-main-ODE} and 
(ii) holds due to the Bayes rule that
\begin{align}
\notag
p_{X_{0}\mymid X_k}(x_{0}\,|\,x)=\frac{p_{X_{0}}(x_{0})}{p_{X_k}(x)}p_{X_k\mymid X_{0}}(x\,|\,x_{0})=\frac{p_{X_{0}}(x_{0})}{p_{X_k}(x)}\big(2\pi(1-\overline{\alpha}_k)\big)^{-d/2}\exp\biggl(-\frac{\| x-\sqrt{\overline{\alpha}_k}x_{0}\|_{2}^{2}}{2(1-\overline{\alpha}_k)}\bigg).
\end{align}

This suggests we focus on controlling \eqref{eqn:fei}. Towards this, let us define the set
\begin{equation}
	\mathcal{E}^{\typ}_c
	\coloneqq\Big\{ x_0\in\bR^d:\big\| x-\sqrt{\overline{\alpha}_k}x_{0}\big\|_{2}\leq 5c\sqrt{\theta_k(x) d(1-\overline{\alpha}_k)\log K}\Big\}	 
	\label{eq:defn-Ec-typical-set}
\end{equation}
for any integer $c\geq 2$. As shown in \citet[Lemma 1]{li2024sharp}, the set $\mathcal{E}^{\typ}_c$ satisfies
\begin{align}
\label{eqn:brahms}
	\mathbb{P}\big\{X_0 \notin \mathcal{E}^{\typ}_c \mid X_k = x\big\} &\leq \exp\bigl(-c^2\theta_k(x) d\log K\bigr)
\end{align}
for all $c\geq 2$,
\begin{align}
	\mathbb{E}\Big[\big\|X_k -  \sqrt{\overline{\alpha}_k}X_{0} \big\|_{2}\,\big|\,X_k=x\Big] &\lesssim \sqrt{\theta_k(x) d(1-\ol\alpha_k)\log K}, \label{eqn:brahms2}
\end{align}
and the vector $u$ defined in \eqref{eq:defn-u-Lemma-main-ODE} satisfies
\begin{align}
\|u\|_{2} & \le \frac{1-\alpha_k}{2}\varepsilon_{\score,k}(x) + \frac{1-\alpha_k}{2(1-\overline{\alpha}_k)}\mathbb{E}\Big[\big\|X_k -  \sqrt{\overline{\alpha}_k}X_{0} \big\|_{2}\,\big|\,X_k=x\Big] \notag\\
	& \leq \frac{1-\alpha_k}{2}\varepsilon_{\score,k}(x) + \frac{6(1-\alpha_k)}{1-\overline{\alpha}_k} \sqrt{\theta_k(x) d(1 - \overline{\alpha}_k)\log K}.
	\label{eqn:johannes}
\end{align}
For any $x_0 \in \mathcal{E}^{\typ}_c$, one can use \eqref{eqn:brahms}, \eqref{eqn:johannes}, and \eqref{eq:step-size} from Lemma~\ref{lemma:step-size} to bound
\begin{subequations}
	\label{eq:long-UB-123}
\begin{align}
\frac{(1-\alpha_k)\big\| x-\sqrt{\overline{\alpha}_k}x_{0}\big\|_{2}^{2}}{2(\alpha_k-\overline{\alpha}_k)(1-\overline{\alpha}_k)} & \leq \frac{25c^2}{2} \frac{(1-\alpha_k) \theta_k(x)d\log K}{\alpha_k-\overline{\alpha}_k}\leq25c^{2}\frac{c_{1}\theta_k(x)d\log^{2}K}{K};\label{eq:long-UB-1}\\
\frac{\|u\|_{2}^{2}}{2(\alpha_k-\overline{\alpha}_k)} & \leq\frac{(1-\alpha_k)^{2}}{4(\alpha_k-\overline{\alpha}_k)}\varepsilon_{\score,k}^{2}(x)+\frac{36(1-\alpha_k)^{2}\theta_k(x)d\log K}{(\alpha_k-\overline{\alpha}_k)(1-\overline{\alpha}_k)} \notag\\
& \leq \frac14\bigg(\frac{1-\alpha_k}{\alpha_k-\overline{\alpha}_k}\bigg)^2 \alpha_k(1-\ol\alpha_{k})\varepsilon_{\score,k}^{2}(x)+\frac{36(1-\alpha_k)^{2}\theta_k(x)d\log K}{(\alpha_k-\overline{\alpha}_k)(1-\overline{\alpha}_k)} \notag\\
 & \leq \frac{c_{1}^{2}(1-\ol\alpha_k)\varepsilon_{\score,k}^{2}(x)\log^{2}K}{K^2}+72\frac{c_{1}^{2}\theta_k(x)d\log^{3}K}{K^2},\label{eq:long-UB-2}\\
\bigg|\frac{u^{\top}(x-\sqrt{\overline{\alpha}_k}x_{0})}{\alpha_k-\overline{\alpha}_k}\bigg| & \leq\frac{\|u\|_{2}\big\| x-\sqrt{\overline{\alpha}_k}x_{0}\big\|_{2}}{\alpha_k-\overline{\alpha}_k}\nonumber\\
 & \leq\frac{5c(1-\alpha_k)}{2(\alpha_k-\overline{\alpha}_k)}\varepsilon_{\score,k}(x)\sqrt{\theta_k(x)d(1-\overline{\alpha}_k)\log K}+\frac{30c(1-\alpha_k)\theta_k(x)d\log K}{\alpha_k-\overline{\alpha}_k}\label{eq:long-UB-3}\\
 & \leq5c\frac{c_{1}\sqrt{\theta_k(x)d(1-\overline{\alpha}_{k})}\,\varepsilon_{\score,k}(x)\log^{3/2}K}{K}+60c\frac{c_{1}\theta_k(x)d\log^{2}K}{K}. \label{eq:long-UB-4}
\end{align}
\end{subequations}
As a result, the following holds for any $x_0\in \mathcal{E}^{\typ}_c$ with $c\geq 2$:
\begin{align}
-\frac{(1-\alpha_k)\big\| x-\sqrt{\overline{\alpha}_k}x_{0}\big\|_{2}^{2}}{2(\alpha_k-\overline{\alpha}_k)(1-\overline{\alpha}_k)}-\frac{\|u\|_{2}^{2}}{2(\alpha_k-\overline{\alpha}_k)}+\frac{u^{\top}\big(x-\sqrt{\overline{\alpha}_k}x_{0}\big)}{\alpha_k-\overline{\alpha}_k} \leq \frac{u^{\top}\big(x-\sqrt{\overline{\alpha}_k}x_{0}\big)}{\alpha_k-\overline{\alpha}_k}
	\le c\theta_k(x)d, \label{eq:multi-term-UB-456}
\end{align}
as long as
\[
\frac{10c_{1}\sqrt{1-\ol\alpha_k}\,\varepsilon_{\score,k}(x)\log^{\frac{3}{2}}K}{K}\leq\sqrt{\theta_k(x)d}\qquad\text{and}\qquad K\geq 120c_{1}\log^{2}K,
\]
which is satisfied by our condition that $c_1\sqrt{1-\ol\alpha_k}\,\varepsilon_{\score,k}(x)\log^{3/2} K\ll K\sqrt{\theta_{k}(x)d}$ and $K\gtrsim d\log^3 K$.

With the above preparation in place, let us begin to bound the integral over the typical set $\cE^\typ_2$. Recalling the definition of $u$ in \eqref{eq:defn-u-Lemma-main-ODE}, we first write
\begin{align}
& \int_{x_0\in\cE^\typ_2}p_{X_{0}\mid X_k}(x_{0}\mid x)\exp\biggl(-\frac{(1-\alpha_k)\big\| x-\sqrt{\overline{\alpha}_k}x_{0}\big\|_{2}^{2}}{2(\alpha_k-\overline{\alpha}_k)(1-\overline{\alpha}_k)}-\frac{\|u\|_{2}^{2}-2u^{\top}(x-\sqrt{\overline{\alpha}_k}x_{0})}{2(\alpha_k-\overline{\alpha}_k)}\bigg)\mathrm{d}x_{0} \notag\\
&\quad = \int_{x_0\in\cE^\typ_2}p_{X_{0}\mid X_k}(x_{0}\mid x)\exp\biggl(-\frac{1-\alpha_k}{2(\alpha_k-\overline{\alpha}_k)}\bigl(\wh{s}_{X_k}(x) - s_{X_k}(x)\bigr)^\top\bigl(x-\sqrt{\overline{\alpha}_k}x_{0}\bigr)-\frac{\|u\|_{2}^2}{2(\alpha_k-\overline{\alpha}_k)}\bigg) \notag\\
& \qquad\quad \cdot 
\exp\biggl(\frac{1-\alpha_k}{2(1-\overline{\alpha}_k)(\alpha_k-\overline{\alpha}_k)}\Bigl((x-\sqrt{\overline{\alpha}_k}x_{0})^\top\bE\big[X_k - \sqrt{\overline{\alpha}_k}X_0 \mid X_k=x\big]-\big\| x-\sqrt{\overline{\alpha}_k}x_{0}\big\|_{2}^{2}\Bigr)\bigg) \mathrm{d}x_{0}. \label{eq:conv-1}
\end{align}
%
For any $x_{0}\in\mathcal{E}_{2}^{\typ}$, using the bounds in \eqref{eq:score-error-pointwise}, \eqref{eq:defn-Ec-typical-set}, and \eqref{eq:long-UB-2} yields
\begin{align*}
	& \frac{1-\alpha_k}{\alpha_k-\overline{\alpha}_k}\bigl\|\wh{s}_{X_k}(x) - s_{X_k}(x)\bigr\|_{2}\bigl\|x-\sqrt{\overline{\alpha}_k}x_{0}\bigr\|_{2}+\frac{\|u\|_{2}^{2}}{\alpha_k-\overline{\alpha}_k} \\ 
	& \qquad\leq \frac{1-\alpha_k}{\alpha_k-\overline{\alpha}_k} \varepsilon_{\score,k}(x) \cdot 10\sqrt{\theta_k(x) d(1-\overline{\alpha}_k)\log K} + \frac{c_{1}^{2}(1-\ol\alpha_k)\varepsilon_{\score,k}^{2}(x)\log^{2}K}{K^2}+72\frac{c_{1}^{2}\theta_k(x)d\log^{3}K}{K^2}\\
	& \qquad \numpf{i}{\lesssim} \frac{c_1\sqrt{\theta_k(x)d(1-\ol\alpha_k)}\,\veps_{\score,k}(x)\log^{3/2}K}{K}+\frac{c_1^2(1-\ol\alpha_k)\varepsilon_{\score,k}^{2}(x)\log^{2}K}{K^{2}}+\frac{c_1^2\theta_k(x)d\log^{3}K}{K^{2}} \\ 
	& \qquad \numpf{ii}{\asymp} \frac{\sqrt{\theta_k(x)d(1-\ol\alpha_k)}\,\veps_{\score,k}(x)\log^{3/2}K}{K} +\frac{\theta_k(x)d\log^{3}K}{K^{2}} \\ 
	& \qquad \lesssim \frac{\sqrt{d(1-\ol\alpha_k)}\,\veps_{\score,k}(x)\log^{3/2}K}{K} +\frac{d\log^{3}K}{K^{2}} \lesssim 1,
\end{align*}
where (i) uses \eqref{eq:learning-rate-4}; (ii) follows from $c_1\asymp 1$ and the condition that $c_1\sqrt{1-\ol\alpha_k}\,\varepsilon_{\score,k}(x)\log^{3/2} K\ll K\sqrt{\theta_{k}(x)d}$; (iii) is true since $\theta_{k}(x)\lesssim1$ and $K\gtrsim d\log^3 K$.
It follows that
\begin{align}
 & \exp\biggl(-\frac{1-\alpha_{k}}{2(\alpha_k-\overline{\alpha}_k)}\bigl(\wh{s}_{X_k}(x) - s_{X_k}(x)\bigr)^\top\bigl(x-\sqrt{\overline{\alpha}_k}x_{0}\bigr)-\frac{\|u\|_{2}^{2}}{2(\alpha_k-\overline{\alpha}_k)}\bigg) \notag\\
 & \qquad=1+O\biggl(\frac{\sqrt{d(1-\ol\alpha_k)}\,\varepsilon_{\score,k}(x)\log^{3/2}K}{K}+\frac{d\log^{3}K}{K^{2}}\biggr).
	\label{eq:exp-usquare-approx}
\end{align}
Plugging this into \eqref{eq:conv-1} allows us to bound the integral over the typical set $\cE^\typ_2$:
\begin{align}
	&\int_{x_0\in\cE^\typ_2}p_{X_{0}\mid X_k}(x_{0}\mid x)\exp\biggl(-\frac{(1-\alpha_k)\big\| x-\sqrt{\overline{\alpha}_k}x_{0}\big\|_{2}^{2}}{2(\alpha_k-\overline{\alpha}_k)(1-\overline{\alpha}_k)}-\frac{\|u\|_{2}^{2}-2u^{\top}(x-\sqrt{\overline{\alpha}_k}x_{0})}{2(\alpha_k-\overline{\alpha}_k)}\bigg)\mathrm{d}x_{0} \notag \\ 
	& \qquad = \bigg\{1+O\biggl(\frac{\sqrt{d(1-\ol\alpha_k)}\,\varepsilon_{\score,k}(x)\log^{3/2}K}{K}+\frac{d\log^{3}K}{K^{2}}\biggr) \bigg\} \notag\\ 
	& \qquad \quad \cdot \int_{x_0\in\cE^\typ_2}p_{X_{0}\mid X_k}(x_{0}\mid x)\exp\biggl(\frac{(1-\alpha_k)\bigl((x-\sqrt{\overline{\alpha}_k}x_{0})^\top\bE\big[X_k - \sqrt{\overline{\alpha}_k}X_0 \mid X_k=x\big]-\big\| x-\sqrt{\overline{\alpha}_k}x_{0}\big\|_{2}^{2}\bigr)}{2(1-\overline{\alpha}_k)(\alpha_k-\overline{\alpha}_k)}\bigg) \mathrm{d}x_{0}. \label{eq:typical-1}
\end{align}

As for the integral over the complement set of $\cE^\typ_2$, we can derive
\begin{align}
 & \int_{x_{0}\notin\cE^\typ_2}p_{X_{0}\mymid X_k}(x_{0}\mymid x)\exp\biggl(-\frac{(1-\alpha_k)\big\| x-\sqrt{\overline{\alpha}_k}x_{0}\big\|_{2}^{2}}{2(\alpha_k-\overline{\alpha}_k)(1-\overline{\alpha}_k)}-\frac{\|u\|_{2}^{2}-2u^{\top}\big(x-\sqrt{\overline{\alpha}_k}x_{0}\big)}{2(\alpha_k-\overline{\alpha}_k)}\bigg)\mathrm{d}x_{0}\notag\\
 & \qquad \numpf{i}{\leq}\sum_{c=3}^{\infty}\int_{x_{0}\in\mathcal{E}_{c}^{\typ}\backslash\mathcal{E}_{c-1}^{\typ}}p_{X_{0}\mymid X_k}(x_{0}\mymid x)\exp\bigl(c\theta_k(x)d\bigr)\mathrm{d}x_{0} \notag\\
  & \qquad \numpf{ii}{\leq}\sum_{c=3}^{\infty}\exp\bigl(-c^{2}\theta_k(x)d\log K\bigr)\exp\bigl(c\theta_k(x)d\bigr)  \leq\sum_{c=3}^{\infty}\exp\biggl(-\frac{1}{2}c^{2}\theta_k(x)d\log K\biggr) \notag \\
 & \qquad \leq \exp\bigl(-\theta_k(x)d\log K\bigr) \lesssim \frac{d\log^3 K}{K^2}
	\label{eq:exp-UB-135702}
\end{align}
where (i) arises from \eqref{eq:multi-term-UB-456}; (ii) uses \eqref{eqn:brahms}; the last line holds as due our choice of $\theta_k$ in \eqref{eqn:choice-y}. 

With \eqref{eq:typical-1} and \eqref{eq:exp-UB-135702} in place, combining them with \eqref{eqn:fei} leads to
%
\begin{align}
 & \frac{p_{\sqrt{\alpha_k}X_{k-1}}\big(\phi_k(x)\big)}{p_{X_k}(x)} 
 \numpf{i}{=}  \biggl(1+O\Big(\frac{d\log^2K}{K^2}\Big)\biggr)\exp\biggl(\frac{d(1-\alpha_k)}{2(\alpha_k-\overline{\alpha}_k)}\biggr) \notag\\
 & \qquad\qquad \cdot \int_{x_{0}\in\mathcal{E}_{2}^{\typ}}p_{X_{0}\mid X_k}(x_{0}\mid x)\exp\biggl(-\frac{(1-\alpha_k)\big\| x-\sqrt{\overline{\alpha}_k}x_{0}\big\|_{2}^{2}}{2(\alpha_k-\overline{\alpha}_k)(1-\overline{\alpha}_k)}-\frac{\|u\|_{2}^{2}-2u^{\top}(x-\sqrt{\overline{\alpha}_k}x_{0})}{2(\alpha_k-\overline{\alpha}_k)}\bigg)\mathrm{d}x_{0} \notag\\
 & \qquad\qquad + O(1)\int_{x_{0}\notin\mathcal{E}_{2}^{\typ}}p_{X_{0}\mid X_k}(x_{0}\mid x)\exp\biggl(-\frac{(1-\alpha_k)\big\| x-\sqrt{\overline{\alpha}_k}x_{0}\big\|_{2}^{2}}{2(\alpha_k-\overline{\alpha}_k)(1-\overline{\alpha}_k)}-\frac{\|u\|_{2}^{2}-2u^{\top}(x-\sqrt{\overline{\alpha}_k}x_{0})}{2(\alpha_k-\overline{\alpha}_k)}\bigg)\mathrm{d}x_{0}   \notag\\
 & \quad \numpf{ii}{=} \bigg\{1+O\bigg(\frac{d\log^{3}K}{K^2}+\frac{\sqrt{d(1-\ol\alpha_k)}\,\varepsilon_{\score,k}(x)\log^{3/2}K}{K}\bigg)\bigg\} \exp\biggl(\frac{d(1-\alpha_k)}{2(\alpha_k-\overline{\alpha}_k)}\biggr) \notag\\ 
 & \qquad \cdot \int_{x_{0}\in\mathcal{E}_{2}^{\typ}}p_{X_{0}\mymid X_k}(x_{0}\mymid x)\exp\biggl(\frac{(1-\alpha_k)\bigl((x-\sqrt{\overline{\alpha}_k}x_{0})^\top\bE[X_k - \sqrt{\overline{\alpha}_k}X_0 \mid X_k=x]-\| x-\sqrt{\overline{\alpha}_k}x_{0}\|_{2}^{2}\bigr)}{2(1-\overline{\alpha}_k)(\alpha_k-\overline{\alpha}_k)}\bigg)\mathrm{d}x_{0}\notag\\
  & \qquad + O\biggl(\frac{d\log^3 K}{K^2}\biggr) \notag\\
 & \quad \numpf{iii}{=}O\bigg(\frac{d\log^{3}K}{K^2}+\frac{\sqrt{d(1-\ol\alpha_k)}\,\veps_{\score,k}\log^{3/2}K}{K}\bigg) +\exp\biggl(\frac{d(1-\alpha_k)}{2(\alpha_k-\overline{\alpha}_k)}\biggr)\notag\\
 & \qquad \cdot \int_{x_{0}\in\mathcal{E}_{2}^{\typ}}p_{X_{0}\mymid X_k}(x_{0}\mymid x)\exp\biggl(\frac{(1-\alpha_k)\bigl((x-\sqrt{\overline{\alpha}_k}x_{0})^\top\bE[X_k - \sqrt{\overline{\alpha}_k}X_0 \mid X_k=x]-\| x-\sqrt{\overline{\alpha}_k}x_{0}\|_{2}^{2}\bigr)}{2(1-\overline{\alpha}_k)(\alpha_k-\overline{\alpha}_k)}\bigg)\mathrm{d}x_{0}\notag\\
 & \quad \numpf{iv}{=}O\bigg(\frac{d\log^{3}K}{K^2}+\frac{\sqrt{d(1-\ol\alpha_k)}\,\veps_{\score,k}\log^{3/2}K}{K}\bigg) + \exp\biggl(\frac{d(1-\alpha_k)}{2(\alpha_k-\overline{\alpha}_k)}\biggr)\notag\\
 & \qquad \cdot \int_{x_0\in\bR^d} p_{X_{0}\mymid X_k}(x_{0}\mymid x)\exp\biggl(\frac{(1-\alpha_k)\bigl((x-\sqrt{\overline{\alpha}_k}x_{0})^\top\bE[X_k - \sqrt{\overline{\alpha}_k}X_0 \mid X_k=x]-\| x-\sqrt{\overline{\alpha}_k}x_{0}\|_{2}^{2}\bigr)}{2(1-\overline{\alpha}_k)(\alpha_k-\overline{\alpha}_k)}\bigg)\mathrm{d}x_{0},
	\label{eq:refined-p-ratio-X0-Xt-135}
\end{align}
Here, (i) holds as one can use \eqref{eq:learning-rate} and \eqref{eq:step-size} to derive
\begin{align*}
\biggl(\frac{1-\overline{\alpha}_k}{\alpha_k-\overline{\alpha}_k}\biggr)^{d/2} &= \biggl(1+\frac{1-\alpha_k}{\alpha_k-\ol\alpha_{k}}\bigg)^{d/2} = \biggl(1+O\Big(\frac{d\log^2K}{K^2}\Big)\biggr)\exp\bigg(\frac{d(1-\alpha_k)}{2(\alpha_k-\ol\alpha_{k})}\bigg)  \lesssim 1;
\end{align*}
(ii) applies the bounds \eqref{eq:typical-1} and \eqref{eq:exp-UB-135702};
(iii) holds as the integral in (ii) is $O(1)$ since
for any $x\in\cE_2^\typ$ with $\theta_t(x)\lesssim 1$, we can use Lemma~\ref{lemma:step-size}, \eqref{eq:defn-Ec-typical-set}, and \eqref{eqn:brahms2} to obtain
\begin{align*}
  \frac{(1-\alpha_k)\big\|x-\sqrt{\overline{\alpha}_k}x_{0}\big\|_{2}^{2}}{(1-\overline{\alpha}_k)(\alpha_k-\overline{\alpha}_k)}
  &\lesssim\frac{(1-\alpha_k)d\log K}{\alpha_k-\overline{\alpha}_k}\lesssim\frac{d\log^2 K}{K}=o(1); \\ 
 \frac{(1-\alpha_k)\big\|x-\sqrt{\overline{\alpha}_k}x_{0}\big\|_{2}\big\|\bE[X_k - \sqrt{\overline{\alpha}_k}X_0 \mid X_k=x]\big\|_2}{(1-\overline{\alpha}_k)(\alpha_k-\overline{\alpha}_k)}
  &\lesssim\frac{(1-\alpha_k)d\log K}{\alpha_k-\overline{\alpha}_k}\lesssim\frac{d\log^2 K}{K}=o(1);
\end{align*}
(iv) is true because we can use a similar argument as in \eqref{eq:exp-UB-135702} to obtain
\begin{align*}
& \int_{x_{0}\notin\cE^\typ_2}p_{X_{0}\mymid X_k}(x_{0}\mymid x) \exp\biggl(\frac{(1-\alpha_k)\big[(x-\sqrt{\overline{\alpha}_k}x_{0})^{\top}\mathbb{E}[x-\sqrt{\overline{\alpha}_k}X_{0}\mid X_k=x]-\| x-\sqrt{\overline{\alpha}_k}x_{0}\|_{2}^{2}\big]}{2(\alpha_k-\overline{\alpha}_k)(1-\overline{\alpha}_k)}\bigg)\mathrm{d}x_{0} \notag\\
&\qquad \lesssim \exp\bigl(-\theta_k(x)d\log K\bigr) \lesssim \frac{d\log^3 K}{K^2}.
\end{align*}

With \eqref{eq:refined-p-ratio-X0-Xt-135} in place, we can repeat the remaining steps in the proof of \citet[Lemma 5]{li2024sharp} to obtain the desired bound:
\begin{align*}
 & \frac{p_{\phi_k(Y_k)}(\phi_k(x))}{p_{Y_k}(x)}\bigg(\frac{p_{\sqrt{\alpha_k}X_{k-1}}\big(\phi_k(x)\big)}{p_{X_k}(x)}\bigg)^{-1}\\
 & \qquad=1+\zeta_k(x)+O\bigg(\bigg\|\frac{\partial\phi_k^{\star}(x)}{\partial x}-I_d\bigg\|_{\mathrm{F}}^{2}+\frac{\sqrt{d(1-\ol\alpha_k)}\,\varepsilon_{\score,k}(x)\log^{3/2}K}{K}+\frac{d(1-\alpha_k) \varepsilon_{\jcb,k}(x)}{K}+\frac{d\log^{3}K}{K^{2}}\bigg) \\ 
 & \qquad=1+\zeta_k(x)+O\bigg(\bigg\|\frac{\partial\phi_k^{\star}(x)}{\partial x}-I_d\bigg\|_{\mathrm{F}}^{2}+\frac{\sqrt{d(1-\ol\alpha_k)}\,\varepsilon_{\score,k}(x)\log^{3/2}K}{K}+\frac{d(1-\ol\alpha_k) \varepsilon_{\jcb,k}(x)\log K }{K}+\frac{d\log^{3}K}{K^{2}}\bigg).
\end{align*}
Here, the last step arises from
\begin{align*}
	(1-\alpha_k)\veps_{\jcb,k}(x) =\frac{1-\alpha_k}{1-\ol\alpha_k}(1-\ol\alpha_k)\veps_{\jcb,k}(x) \lesssim \frac{\log K}{K}(1-\ol\alpha_k)\veps_{\jcb,k}(x)
\end{align*}
where we use \eqref{eq:alpha-t-lb} from Lemma~\ref{lemma:step-size}.

%% file: proof-discretization.tex
\subsection{Analysis for score matching (proof of Theorem~\ref{thm:score-error-X})}\label{sub:proof_of_theorem_ref_thm_score_error_x}

Before delving into the details, we remind readers of several notation. For $Z_t$ introduced in \eqref{eq:SDE-VE}, we denote its score function as $s_t(x)\defn s_{Z_t}(x)$ and the associated Jacobian matrix as $J_t(x) \defn J_{s_t}(x)$ (see \eqref{eq:Z-t-abbrev}). In addition, we recall $t_k \defn (1-\ol\alpha_k)/\ol\alpha_k+\tau$ in \eqref{eq:score-estimator-X}.

Now, we begin with Claim \eqref{eq:score-error-X}. By Theorem~\ref{thm:score-error}, one can derive
\begin{align*}
	\veps_{\score,k}^2 & = \bE\Big[ \big\| \wh{s}_{X_k}(X_k)-s_{X_k}(X_k) \big\|_2^2 \Big]  \\
	&\numpf{i}{=} \bE\Big[ \big\| \wh{s}_{t_k}(X_k/\sqrt{\ol\alpha_k})/\sqrt{\ol\alpha_k}-s_{{t_k}}(X_k/\sqrt{\ol\alpha_k})/\sqrt{\ol\alpha_k} \big\|_2^2 \Big] \nonumber \\ 
	&\numpf{ii}{=} \frac1{\ol\alpha_k} \bE\Big[ \big\| \wh{s}_{t_k}(Z_{t_k})-s_{{t_k}}(Z_{t_k}) \big\|_2^2 \Big] \nonumber \\
	& \numpf{iii}{\lesssim} \frac{C_d}{n} \frac{1}{\ol\alpha_k t_k} \Bigl( 1 + {\sigma^d}{ t_k^{-d/2}} \Bigr)(\log n)^{d/2+1} \\
	& \numpf{iv}{\leq} \frac{C_d'}{n} \frac{1}{1-\ol\alpha_k}\biggl\{ 1 + \sigma^d \bigg[\tau^{-d/2} \wedge \biggl(\frac{\ol\alpha_k}{1-\ol\alpha_k}\biggr)^{d/2} \bigg] \biggr\}(\log n)^{d/2+1}.
\end{align*}
where we $C_d' = 2^{d/2}C_d = (4\sqrt{2}/\sqrt{\pi})^d$.
Here, (i) holds due to the construction in \eqref{eq:score-estimator-X} that $\wh{s}_{X_k}(x)\defn\wh{s}_{t_k}(x/\sqrt{\ol\alpha_k})/\sqrt{\ol\alpha_k}$ and $s_{X_k}(x)=s_{{t_k}}(x/\sqrt{\ol\alpha_k})/\sqrt{\ol\alpha_k}$; (ii) arises from $X_k \disteq \sqrt{\ol\alpha_k}Z_{t_k}$; (iii) invokes \eqref{eq:score-error} in Theorem~\ref{thm:score-error}; (iv) arises from $t_k \geq (1-\ol\alpha_k)/\ol\alpha_k$ and $2t_k \geq (1-\ol\alpha_k)/\ol\alpha_k \vee \tau$.
Consequently, we arrive at
\begin{align}
	\veps_{\score} = \frac1K\sum_{k=1}^K(1-\ol\alpha_k)\veps^2_{\score,k}\lesssim  \frac{C_d'}{n}\bigg\{1 + \sigma^d \bigg[\tau^{-d/2} \wedge \frac{1}K\sum_{k=1}^K \biggl(\frac{\ol\alpha_k}{1-\ol\alpha_k}\biggr)^{d/2}\bigg]\bigg\}(\log n)^{d/2+1}. \label{eq:score-l2-error-temp1}
\end{align}
This suggests we need to bound the sum $\sum_{k=1}^K \bigl({\ol\alpha_k}/({1-\ol\alpha_k})\bigr)^{d/2}$. 

To this end, let $k_0\defn \max\{1\leq k \leq K\colon \ol\alpha_k \geq 1/2\}$, which is well-defined since $\alpha_1 = 1-K^{-c_0}\geq1/2$ for $K$ large enough. Notice that ${\ol\alpha_k}/({1-\ol\alpha_k})$ is decreasing in $k$, since $x\mapsto x/(1-x)$ in increasing in $x$ for $x\in(0,1)$ and $\ol\alpha_k$ is decreasing in $k$. This implies that $\bigl({\ol\alpha_k}/({1-\ol\alpha_k})\bigr)^{d/2} < 1$ for all $k>k_0$.
We can then derive
\begin{align}
	\frac1K \sum_{k=1}^K \biggl(\frac{\ol\alpha_k}{1-\ol\alpha_k}\biggr)^{d/2} 
	& \leq\frac1K  \sum_{k=1}^{k_0} \biggl(\frac{\ol\alpha_k}{1-\ol\alpha_k}\biggr)^{d/2} + \frac1K \sum_{k=k_0+1}^{K} 1 \notag\\
	&\numpf{i}{=} \frac1K\biggl(\frac{\ol\alpha_1}{1-\ol\alpha_1}\biggr)^{d/2} + \frac1{c_1\log K} \sum_{k=2}^{k_0} \frac{\ol\alpha_{k-1}-\ol\alpha_{k}}{\ol\alpha_{k-1}(1-\ol\alpha_{k-1})} \biggl(\frac{\ol\alpha_k}{1-\ol\alpha_k}\biggr)^{d/2}+1 \notag\\
	&\numpf{ii}{\leq} \frac1{K(1-\alpha_1)^{d/2}} + \frac1{c_1\log K} \sum_{k=2}^{k_0} \frac{2}{\ol\alpha_{k}(1-\ol\alpha_{k})} \biggl(\frac{\ol\alpha_k}{1-\ol\alpha_k}\biggr)^{d/2}\big(\ol\alpha_{k-1}-\ol\alpha_{k}\big)+1 \notag\\
	& = \frac1{K(1-\alpha_1)^{d/2}} + \frac2{c_1\log K} \sum_{k=2}^{k_0} \frac{\ol\alpha_{k}^{d/2-1}}{(1-\ol\alpha_{k})^{d/2+1}} \big(\ol\alpha_{k-1}-\ol\alpha_{k}\big)+1
	\label{eq:score-l2-error-temp2}
\end{align}
where (i) arises from our choice of the learning rates in \eqref{eq:learning-rate} that $\ol\alpha_{k-1}-\ol\alpha_{k}=\ol\alpha_{k-1}(1-\ol\alpha_{k-1})c_1\log K/K$; (ii) arises from $\alpha_1=\ol\alpha_1 < 1$ and \eqref{eq:step-size} in Lemma~\ref{lemma:step-size} that $\ol\alpha_{k-1}(1-\ol\alpha_{k-1}) \geq  \ol\alpha_{k}(1-\ol\alpha_{k})/2$.

Notice that for all $d\geq 1$, $x\mapsto x^{d/2-1}/(1-x)^{d/2+1}$ is increasing in $x$ for $x\in[1/2,1)$. Combining this with the fact that $\ol\alpha_k$ is decreasing in $k$ and $\ol\alpha_{k}\geq 1/2$ for all $1\leq k \leq k_0$, we can use the integral approximation to bound
\begin{align}
	\sum_{k=2}^{k_0} \frac{\ol\alpha_{k}^{d/2-1}}{(1-\ol\alpha_{k})^{d/2+1}} \big(\ol\alpha_{k-1}-\ol\alpha_{k}\big) 
	& \leq  \int_{\ol\alpha_{k_0}}^{\ol\alpha_1}\frac{x^{d/2-1}}{(1-x)^{d/2+1}}\diff x  =  \frac2d \bigg(\frac{x}{1-x}\bigg)^{d/2}\Bigg|_{\ol\alpha_{k_0}}^{\ol\alpha_{1}} \leq \frac2d\frac{1}{(1-\alpha_1)^{d/2}}  \notag.
\end{align}
Substituting this into \eqref{eq:score-l2-error-temp2} gives
\begin{align}
	\frac1K \sum_{k=1}^K \biggl(\frac{\ol\alpha_k}{1-\alpha_k}\biggr)^{d/2} 
	&\leq \frac1{K(1-\alpha_1)^{d/2}} + \frac4{c_1 d\log K}\frac{1}{(1-\alpha_1)^{d/2}}  
	\lesssim \frac{K^{c_0d/2}}{d\log K} 
	,\label{eq:score-l2-error-temp3}
\end{align}
where the last step holds since $1-\alpha_1=K^{-c_0}$ and
$K\gtrsim d\log K$.
Putting \eqref{eq:score-l2-error-temp1} and \eqref{eq:score-l2-error-temp3} together yields
\begin{align*}
	\veps^2_{\score} = \frac1K\sum_{k=1}^K(1-\ol\alpha_k)\veps^2_{\score,k} \lesssim \frac{C_d'}{n}\bigg\{1 + \sigma^d \bigg(\tau^{-d/2} \wedge \frac{K^{c_0d/2}}{d\log K}+1 \bigg)\bigg\}(\log n)^{d/2+1} .
\end{align*}
This establishes Claim \eqref{eq:score-error-X}.

Let us proceed to consider Claim \eqref{eq:Jacobian-error-X}. Similar to the above analysis, we first invoke Theorem~\ref{thm:Jacobian-error} to obtain
\begin{align*}
	\veps_{\jcb,k} &\defn \bE\Big[ \big\| J_{\wh{s}_{X_k}}(X_k)-J_{s_{X_k}}(X_k) \big\| \Big] \\
	&\numpf{i}{=} \bE\Big[ \big\| J_{\wh{s}_{t_k}}(X_k/\sqrt{\ol\alpha_k})/{\ol\alpha_k}-J_{t_k}(X_k/\sqrt{\ol\alpha_k})/{\ol\alpha_k} \big\| \Big] \nonumber \\ 
	&\numpf{ii}{=} \frac1{\ol\alpha_k} \bE\Big[ \big\| J_{\wh{s}_{t_k}}(Z_{t_k})-J_{t_k}(Z_{t_k}) \big\| \Big] \nonumber \\
	& \numpf{iii}{\leq} \sqrt{\frac{C_d}{n}} \frac{1}{\ol\alpha_k t_k}\Bigl( 1 + {\sigma^{d/4}}{t_k^{-d/4}} \Bigr)(\log n)^{d/4+1} + \frac{C_d}{n}\frac1{\ol\alpha_k t_k} {\sigma^d}{ t_k^{-d/2}}(\log n)^{d/2+2} \\
	& \numpf{iv}{\leq} \sqrt{\frac{C_d'}{n}} \frac{1}{1-\ol\alpha_k} \biggl\{ 1 + \sigma^{d/2}\bigg[\tau^{-d/4}\wedge \biggl(\frac{\ol\alpha_k}{1-\ol\alpha_k}\biggr)^{d/4}\bigg] \biggr\} (\log n)^{d/4+1} \\ 
	& \quad 
	+ \frac{C_d'}{n} \frac{\sigma^d}{1-\ol\alpha_k}\bigg\{\tau^{-d/2}\wedge\biggl(\frac{\ol\alpha_k}{1-\ol\alpha_k}\biggr)^{d/2}\bigg\} (\log n)^{d/2+2}.
\end{align*}
Here, (i) is true since $J_{s_{X_k}}(x)=J_{t_k}(x/\sqrt{\ol\alpha_k})/\ol\alpha_k$; (ii) arises from $X_k \disteq \sqrt{\ol\alpha_k}Z_{t_k}$; (iii) applies \eqref{eq:Jacobian-error} in Theorem~\ref{thm:Jacobian-error}; (iv) follows from $t_k \geq (1-\ol\alpha_k)/\ol\alpha_k$ and $2t_k \geq (1-\ol\alpha_k)/\ol\alpha_k \vee \tau$.
We can then bound
\begin{align}
	\veps_\jcb=\frac1K\sum_{k=1}^K(1-\ol\alpha_k)\veps_{\jcb,k} &\leq \sqrt{\frac{C_d'}{n}} \biggl\{ 1 + \sigma^{d/2}\bigg[\tau^{-d/4}\wedge \frac{1}{K}\sum_{k=1}^K\biggl(\frac{\ol\alpha_k}{1-\ol\alpha_k}\biggr)^{d/4}\bigg] \biggr\} (\log n)^{d/4+1} \notag\\ 
	& \quad+ \frac{C_d'\sigma^{d}}{n} \bigg\{\tau^{-d/2}\wedge \frac{1}{K}\sum_{k=1}^K\biggl(\frac{\ol\alpha_k}{1-\ol\alpha_k}\biggr)^{d/2}\bigg\} (\log n)^{d/2+2}\label{eq:jcb-error-temp1}
\end{align}

Applying the same argument for \eqref{eq:score-l2-error-temp3}, we can bound
\begin{align}
	\frac{1}{K}\sum_{k=1}^K\biggl(\frac{\ol\alpha_k}{1-\ol\alpha_k}\biggr)^{d/4}
	& \numpf{i}{\leq} \sum_{k=1}^{k_0} \frac1K \biggl(\frac{\ol\alpha_k}{1-\ol\alpha_k}\biggr)^{d/4} +1 \notag\\
	&\numpf{ii}{\leq} \frac1K\frac{1}{(1-\ol\alpha_1)^{d/4}} + \frac2{c_1\log K} \sum_{k=2}^{k_0} \frac{\ol\alpha_{k-1}-\ol\alpha_{k}}{\ol\alpha_{k}(1-\ol\alpha_{k})} \biggl(\frac{\ol\alpha_k}{1-\ol\alpha_k}\biggr)^{d/4} +1 \notag\\
	& \numpf{iii}{\leq} \frac1{K(1-\alpha_1)^{d/4}} +\frac2{c_1\log K}\int_{\ol\alpha_{k_0}}^{\ol\alpha_1}\frac{x^{d/4-1}}{(1-x)^{d/4+1}}\diff x +1 \notag\\
	&\leq \frac1{K(1-\alpha_1)^{d/4}} + \frac8{c_1 d\log K}\frac{1}{(1-\alpha_1)^{d/4}}+1 \notag\\ 
	& \numpf{iv}{\lesssim} \frac{K^{c_0d/4}}{d\log K}
	, \label{eq:jcb-error-temp2}
\end{align}
where (i) holds as ${\ol\alpha_k}/({1-\ol\alpha_k})$ is decreasing in $k$; (ii) holds since our learning rate schedule ensures that $\alpha_1<1$, $\ol\alpha_{k-1}-\ol\alpha_{k}=\ol\alpha_{k-1}(1-\ol\alpha_{k-1})c_1\log K/K$, and $\ol\alpha_{k-1}(1-\ol\alpha_{k-1})\geq \ol\alpha_{k}(1-\ol\alpha_{k})/2$; (iii) is true since $x\mapsto x^{d/4-1}/(1-x)^{d/4+1}$ is increasing in $x$ for $x\in[1/2,1)$ and $\ol\alpha_{k}\geq 1/2$ for all $1\leq k \leq k_0$; (iv) holds since $1-\alpha_1=K^{-c_0}$ and  
$K\gtrsim d\log K$.
Putting\eqref{eq:jcb-error-temp2} and \eqref{eq:score-l2-error-temp3} collectively into \eqref{eq:jcb-error-temp1} yields
\begin{align*}
	\veps_\jcb 
	&\defn \frac1K\sum_{k=1}^K(1-\ol\alpha_k)\veps_{\jcb,k} \\
	& \lesssim \sqrt{\frac{C_d'}n} \biggl\{ 1 + \sigma^{d/2} \bigg(\tau^{-d/4} \wedge \frac{K^{c_0d/4}}{d\log K}\bigg)  \biggr\} (\log n)^{d/4+1} + \frac{C_d'\sigma^d}n\bigg(\tau^{-d/2}\wedge \frac{K^{c_0d/2}}{d\log K}\bigg)(\log n)^{d/2+1},
\end{align*}
as claimed in \eqref{eq:Jacobian-error-X}.

Finally, let us prove Claim \eqref{eq:PSD-cond}. Recalling $\wh{s}_{X_k}(x)\defn\wh{s}_{t_k}(x/\sqrt{\ol\alpha_k})/\sqrt{\ol\alpha_k}$ and $t_k = (1-\ol\alpha_k)/\ol\alpha_k+\tau$, we have
\begin{align*}
	J_{\wh{s}_{X_k}}(x) + \frac{1}{1-\ol\alpha_k}I_d &= \frac{1}{\ol\alpha_k} J_{\wh{s}_{t_k}}(x/\sqrt{\ol\alpha_k}) + \frac{1}{1-\ol\alpha_k}I_d = \frac{1}{\ol\alpha_k}\bigg( J_{\wh{s}_{t_k}}(x/\sqrt{\ol\alpha_k}) + \frac1{t_k-\tau} I_d \bigg) \\ 
	&\geq \frac{1}{\ol\alpha_k}\bigg( J_{\wh{s}_{t_k}}(x/\sqrt{\ol\alpha_k}) + \frac1{t_k} I_d \bigg).
\end{align*}
Thus, it suffices to show that $J_{\wh{s}_t}(x)+\frac1t I_d \psd 0$ for all $x\in\bR^d$ and $t>0$. 

Examining the expression of $J_{\wh{s}_t}(x)$ in \eqref{eq:Jacobian-score-est-expression} and recognizing $\wh p_t(x) =n^{-1} \sum_{i=1}^n \vphi_{t}(X^{(i)}-x)$, we know that
\begin{align*}
\frac{\wh{H}_t(x)}{\wh{p}_t(x)} - \frac{\gd\wh{p}_t(x)\gd\wh{p}_t(x)^\top}{\wh{p}_t^2(x)} & =  \frac{1}{n} \sum_{i=1}^n \frac1{t^2} (X^{(i)}-x)(X^{(i)}-x)^\top \frac{\vphi_{t}(X^{(i)}-x)}{\wh p_t(x)} \notag\\ 
& \quad - \bigg(\frac1{n}\sum_{i=1}^n \frac1t(X^{(i)}-x)\frac{\vphi_{t}(X^{(i)}-x)}{\wh p_t(x)}\bigg)\bigg(\frac1{n}\sum_{i=1}^n \frac1t(X^{(i)}-x)\frac{\vphi_{t}(X^{(i)}-x)}{\wh p_t(x)}\bigg)^\top,
\end{align*}
which is a covariance matrix. 
Therefore, we obtain
\begin{itemize}
\item When $\wh{p}_t(x) \geq \eta_t$, \eqref{eq:J-1} confirms that $J_{\wh{s}_t}(x)+\frac1t I_d \psd 0$.

\item When $\wh{p}_t(x) \leq \eta_t/2$, the claim follows directly from \eqref{eq:J-2}.

\item When $\eta_t/2 < \wh{p}_t(x) < \eta_t$, since $0<\psi(x;\eta)<1$ and $\psi'(x;\eta)>0$ for all $x$, \eqref{eq:J-3} ensures $J_{\wh{s}_t}(x)+\frac1t I_d \psd 0$.
\end{itemize}


%% file: proof-score.tex

\subsection{Analysis for $L^2$ score estimation (proof of Proposition~\ref{thm:score-error})}\label{sub:proof_of_lemma_ref_lemma_score-error}


Recall the definitions of $\cE_t(x)$ and $\cF_t$ in \eqref{eq:event-E} and \eqref{eq:rho-t}, respectively. We first show that for any $x\in\cF_t$, on the event $\cE_t(x)$, the score estimator $\wh{s}_{t}(x)$ defined in \eqref{eq:score-estimate} satisfies
\begin{align}
	\wh{s}_{t}(x)=\frac{\wh{g}_t(x)}{\wh{p}_{t}(x)}. \label{eq:score-Et}
\end{align}
Indeed, for any $x\in\cF_t$, since we set $c_\eta = 4C_\cE \vee 2 \geq 2C_\cE(1+1/\sqrt{c_\eta})$, it is straightforward to verify that
\begin{align*}
	p_{t}(x) \geq \frac{c_\eta\log n}{n(2\pi t)^{d/2}} \geq 2C_\cE \bigg( \frac{\log n}{n(2\pi t)^{d/2}}+ \sqrt{\frac{p_{t}(x)\log n}{n(2\pi t)^{d/2}}} \bigg).
\end{align*}
Thus, one has
\begin{align}
	\wh{p}_{t}(x) & \geq p_{t}(x) - C_\cE \bigg( \frac{\log n}{n(2\pi t)^{d/2}}+ \sqrt{\frac{p_{t}(x)\log n}{n(2\pi t)^{d/2}}} \bigg)
	\geq \frac12 p_{t}(x) 
	\geq \frac{\log n}{n(2\pi t)^{d/2}} = \eta_t \label{eq:p-hat-LB},
\end{align}
where the last step holds as $c_\eta \geq 2$. This proves the claim in \eqref{eq:score-Et}.

Now, given the expression of $\wh{s}_t(x)$, we can then decompose
\begin{align}
& \bE\Big[\big\| \wh{s}_{t}(Z_t)-s_{t}(Z_t) \big\|^2_2\Big] = \int_{\bR^d} \bE\Big[ \big\| \wh{s}_{t}(x)-s_{t}(x) \big\|^2_2 \Big] p_{t}(x) \diff x \notag\\
& \qquad = \underbrace{\int_{\cF_t} \bE\Big[ \big\| \wh{s}_{t}(x)-s_{t}(x) \big\|^2_2 \ind\big\{\cE_t(x)\big\} \Big] p_{t}(x) \diff x}_{\defnrev \chi_1}
+ \underbrace{\int_{\cF_t^\setc} \bE\Big[ \big\| \wh{s}_{t}(x)-s_{t}(x) \big\|^2_2 \ind\big\{\cE_t(x)\big\} \Big] p_{t}(x) \diff x}_{\defnrev \chi_2} \notag\\
& \qquad \quad + \underbrace{\int_{\bR^d} \bE\Big[ \big\| \wh{s}_{t}(x)-s_{t}(x) \big\|^2_2 \ind\big\{\cE_t^\setc(x)\big\} \Big] p_{t}(x) \diff x}_{\defnrev \chi_3}. \label{eq:score-error-temp}
\end{align}
In what follows, we shall bound $\chi_1$, $\chi_2$, and $\chi_3$ individually.

\paragraph*{Step 1: bounding $\chi_1$.}
To control the term $\chi_1$, we first present the following lemma that characterizes the mean $L^2$ error of $\wh{s}_t(x)$ on the event $\cE_t(x)$. 
\begin{lemma}\label{lemma:MSE-score-high-density}
For any $x\in\cF_t$,
\begin{align}
	\bE\Big[ \big\| \wh{s}_t(x)-s_t(x)\big\|_2^2 \ind\big\{\cE_t(x)\big\} \Big] 
	& \lesssim \frac{1}{n(2\pi t)^{d/2}}\bigg(\frac1t+\|s_t(x)\|_2^2\bigg) \frac{1}{p_t(x)}
	\label{eq:MSE-score}.
\end{align}
\end{lemma}
\begin{proof}
	See Appendix~\ref{sub:proof_of_lemma_ref_lemma_mse-score}.
\end{proof}

Taken together with the bound on $\|s_t(x)\|_2$ in \eqref{eq:score-UB-prob-high} from Lemma~\ref{lemma:score-Jacobian-UB-prob-high}, this leads to
\begin{align*}
\chi_1
&\lesssim \frac{1}{n(2\pi t)^{d/2}}\int_{\cF_t}\bigg(\frac1{t} +  \|s_t(x)\|_2^2\bigg) \diff x \lesssim \frac{1}{n(2\pi t)^{d/2}} \frac{\log n}{t}|\cF_t|.
\end{align*}

Next, we present the following lemma, demonstrating that the volume of the set $\cF_t$ is small when $Z_t$ is subgaussian. 
\begin{lemma}\label{lemma:high-density-set-bound}
The set $\cF_t$ defined in \eqref{eq:rho-t} satisfies
	\begin{align}
		|\cF_t| & 
		\leq (32)^{d/2} (t^{d/2} + \sigma^d)(\log n)^{d/2} \label{eq:high-density-set-UB}.
	\end{align}
\end{lemma}
\begin{proof}
	See Appendix~\ref{sub:proof_of_lemma_ref_lemma_high_density_set_bound}.
\end{proof}
As a consequence, we conclude
\begin{align}
	\chi_1 & \lesssim \bigg(\frac{16}{{\pi}}\bigg)^{d/2}\frac{1}{n} \bigg(1+ \frac{\sigma^d}{t^{d/2+1}}\bigg) (\log n)^{d/2+1} \label{eq:score-K1}.
\end{align}

\paragraph*{Step 2: Bounding $\chi_2$.}

We first apply the Cauchy-Schwartz inequality to obtain
\begin{align*}
	\chi_2 & \lesssim \int_{\cF_t^\setc}  \Big( \bE\Big[ \big\| \wh{s}_{t}(x)\big\|_2^2 \Big]+\|s_{t}(x) \|^2_2\Big)\,p_{t}(x) \diff x \notag\\
	& = \underbrace{\int_{\cF_t^\setc} \bE\Big[ \big\| \wh{s}_t(x)  \big\|^2_2 \ind\big\{\wh{p}_{t}(x) \geq \eta_t/2 \big\} \Big] p_{t}(x) \diff x}_{\defnrev (\mathrm{I})}+\underbrace{\int_{\cF_t^\setc} \|s_{t}(x)\|^2_2\,p_{t}(x) \diff x}_{\defnrev (\mathrm{II})}.
\end{align*}
where the last step holds due to our construction of $\wh{s}_{t}(x)$ in \eqref{eq:score-estimate} that $\wh{s}_{t}(x) = 0$ when $\wh{p}_{t}(x) \leq \eta_t / 2$.

Let us begin with the term (I). 
In view of our choice of $\wh{s}_{t}(x)$ in \eqref{eq:score-estimate} and $0<\psi< 1$, we know that when $\wh{p}_t(x) > \eta_t/2$, 
\begin{align}
	\big\| \wh{s}_{t}(x)\big\|_2 &\leq \frac{1}{\wh{p}_t(x)}\|  \wh{g}_t(x)\|_2 \leq \frac2{\eta_t} \big\| \wh{g}_t(x)\big\|_2. \label{eq:score-estimate-UB-as}
\end{align}
It follows that for any $x\in\cF_t^\setc$, one has
\begin{align}
	\bE\Big[ \big\| \wh{s}_t(x) \big\|^2_2 \ind\big\{\wh{p}_{t}(x) \geq \eta_t/2 \big\} \Big]
	&\lesssim \frac1{\eta_t^2}\bE\big[\| \wh{g}_t(x)\|^2_2 \big] \notag\\
	& \numpf{i}{\lesssim} \frac1{\eta_t^2} \Big(\bE\big[\| \wh{g}_t(x)-g_t(x)\|^2_2 \big] + \|g_t(x)\|_2^2 \Big)\notag\\
	& \numpf{ii}{\leq} \frac1{\eta_t^2}\bigg(\frac{p_t(x)}{n(2\pi t)^{d/2}t}+p^2_t(x)\|s_t(x)\|_2^2 \bigg) \notag \\
	& \numpf{iii}{\leq}\frac1{\eta_t^2}\bigg(\frac{c_\eta\eta_t}{n(2\pi t)^{d/2}t}+(c_\eta\eta_t)^2\|s_t(x)\|_2^2 \bigg) \notag \\
	& \numpf{iv}{=} \frac{c_\eta n(2\pi t)^{d/2}}{\log n}\frac{1}{n(2\pi t)^{d/2}t} + c_\eta^2\,\|s_t(x)\|_2^2  \notag\\
	& \numpf{v}{\asymp} \frac{1}{t\log n} + \|s_t(x)\|_2^2, \label{eq:g-hat-low-prob-UB}
\end{align}
where (i) uses the triangle inequality; (ii) arises from \eqref{eq:MSE-g} in Lemma~\ref{lemma:MSE-density} and $g_t(x)=p_t(x)s_t(x)$; (iii) is true since $p_{t}(x) < c_\eta\eta_t$ on $\cF_t^\setc$; (iv) follows from our choice of $\eta_t$ in \eqref{eq:score-estimate}; (v) holds as $c_\eta>0$ is some absolute constant.

Combined with the term (II), this gives
\begin{align*}
	\chi_2 &\lesssim (\mathrm{I}) + (\mathrm{II}) \lesssim \frac{1}{t\log n}\int_{\cF_t^\setc} p_t(x) \diff x+ \int_{\cF_t^\setc}\|s_{t}(x)\|_2^2\,p_{t}(x) \diff x.
\end{align*}

This suggests we need to control the expectations over the set $\cF_t^\setc$ where the density of $p_t$ is vanishingly small, which is accomplished by the following lemma.
\begin{lemma}\label{lemma:prob-small-density}
Recall $$
\cF_t^\setc \defn \bigg\{x\in\bR^d\colon p_t(x) <  \frac{c_\eta\log n}{n(2\pi t)^{d/2}}\bigg\}.$$ One has
\begin{subequations}\label{eq:prob-small}
\begin{align}
	 \int_{\cF_t^\setc} p_t(x) \diff x
	 &\lesssim\bigg(\frac{16}{{\pi}}\bigg)^{d/2} \frac1n\bigg(1+\frac{\sigma^d}{t^{d/2}}\bigg) (\log n)^{d/2+1}
	\label{eq:prob-small-density}, \\
	 \int_{\cF_t^\setc}\|s_{t}(x)\|_2^2\,p_{t}(x) \diff x 
	 & \lesssim \bigg(\frac{16}{{\pi}}\bigg)^{d/2} \frac1n\bigg(\frac1t+\frac{\sigma^d}{t^{d/2+1}}\bigg) (\log n)^{d/2+2} \label{eq:prob-small-score}, \\
	 \int_{\cF_t^\setc}\big\|J_{t}(x)\big\|\,p_{t}(x) \diff x 
	 & \lesssim \bigg(\frac{16}{{\pi}}\bigg)^{d/2} \frac1n\bigg(\frac1t+\frac{\sigma^d}{t^{d/2+1}}\bigg) (\log n)^{d/2+2} \label{eq:prob-small-jacobian},
\end{align}
\end{subequations}
\end{lemma}
\begin{proof}
See Appendix \ref{sub:proof_of_lemma_ref_lemma_prob_small_density}.
\end{proof}
Therefore, we can use \eqref{eq:prob-small-density}--\eqref{eq:prob-small-score} in Lemma~\ref{lemma:prob-small-density} to bound
\begin{align}
	\chi_2 &\lesssim \bigg(\frac{16}{{\pi}}\bigg)^{d/2} \frac1{t\log n} \frac1n\bigg(1+\frac{\sigma^d}{t^{d/2}}\bigg) (\log n)^{d/2+1} 
	+ \bigg(\frac{16}{{\pi}}\bigg)^{d/2} \frac1n\bigg(\frac1t+\frac{\sigma^d}{t^{d/2+1}}\bigg) (\log n)^{d/2+2}
	\notag\\ 
	& \asymp \bigg(\frac{16}{{\pi}}\bigg)^{d/2} \frac1n\bigg(1+\frac{\sigma^d}{t^{d/2+1}}\bigg) (\log n)^{d/2+2} \label{eq:score-K2}.
\end{align}


\paragraph*{Step 3: Bounding $\chi_3$.}
Applying the Cauchy-Schwartz inequality yields
\begin{align}
\int_{\bR^d} \bE\Big[ \big\| \wh{s}_{t}(x)-s_{t}(x) \big\|^2_2 \ind\big\{\cE^\setc(x)\big\} \Big] p_{t}(x) \diff x  
& \leq \int_{\bR^d} \sqrt{\bE\big[\| \wh{s}_{t}(x)-s_{t}(x)\|^4_2 \big]} \sqrt{\bP\big(\cE^\setc(x)\big)}\, p_{t}(x) \diff x \notag\\
& \lesssim \frac1{n^5} \int_{\bR^d} \Big(\sqrt{\bE\big[\| \wh{s}_{t}(x)\|_2^4 \big]} + \|s_{t}(x) \|^2_2\Big)\, p_{t}(x) \diff x
\end{align}
where the last step uses \eqref{eq:density-concen-ineq} in Lemma~\ref{lemma:MSE-density} and $\sqrt{a+b}\leq \sqrt{a}+\sqrt{b}$ for any $a,b\geq 0$. Hence, it remains to control the two integrals.

Regarding the first term involving $\wh{s}_{t}(x)$, in light of \eqref{eq:score-estimate-UB-as} that $\|\wh{s}_{t}(x)\|_2 \leq 2\,\|\wh{g}_t(x)\|_2/\eta_t$, it suffices to control $\|\wh{g}_t(x)\|_2$.
Recalling the expression of $\wh{g}_t(x)$ in \eqref{eq:g-hat}, 
one can bound
\begin{align}
	\big\| \wh{g}_t(x)\big\|_2
	& = \bigg\| \frac1{nt}\sum_{i=1}^n (X_i-x)\vphi_t(X_i-x)\bigg\|_2 \notag\\
	& \leq \max_{i\in[n]}\frac1{t} \big\| (X_i-x)\vphi_t(X_i-x)\big\|_2 \notag\\
	& \leq \frac1t \sup_{x\in\bR^d}\bigg\{\|x\|_2 \frac1{(2\pi t)^{d/2}}\exp\left(-\|x\|_2^2/\left(2t\right)\right)\bigg\} \notag\\
	& = \frac1{\sqrt{e}\,(2\pi t)^{d/2}\sqrt{t}}. \label{eq:g-hat-UB}
\end{align}
Plugging the value of $\eta_t$ into \eqref{eq:g-hat-UB} gives
\begin{align}
	\big\| \wh{s}_{t}(x)\big\|_2 \leq \frac2{\eta_t} \big\| \wh{g}_t(x)\big\|_2 \lesssim \frac{n(2\pi t)^{d/2}}{\log n} \frac1{(2\pi t)^{d/2}\sqrt{t}} \leq \frac{n}{\sqrt t\,\log n},
\end{align}
thereby leading to
\begin{align}
	\int_{\bR^d} \sqrt{\bE\big[ \| \wh{s}_{t}(x)\|_2^4\big]}\,p_{t}(x) \diff x 
	& \lesssim  \frac{n^2}{t \log^2 n}. \label{eq:score-K3-I}
\end{align}

Turning to the second term, applying \eqref{eq:score-lp-norm} from Lemma~\ref{lemma:score-lp-norm} allows us to bound
\begin{align}
	\int_{\bR^d} \|s_{t}(x)\|^2_2\,p_{t}(x) \diff x
	\lesssim \frac{d}{t}, \label{eq:score-K3-II}
\end{align}

Therefore, combining \eqref{eq:score-K3-I} and \eqref{eq:score-K3-II} reveals that
\begin{align}
\chi_3 \lesssim \frac1{n^5}\bigg(\frac{n^2}{t \log^2 n} +\frac{d}{t} \bigg) \lesssim \frac{d}{n^3 t}
	\label{eq:score-K3}.
\end{align}

\paragraph*{Step 4: Combining bounds for $\chi_1,\chi_2,\chi_3$.}

To finish up, substituting the bounds \eqref{eq:score-K1}, \eqref{eq:score-K2}, and \eqref{eq:score-K3} into \eqref{eq:score-error-temp} yields
\begin{align*}
	&\int_{\bR^d} \bE\Big[ \big\| \wh{s}_{t}(x)-s_{t}(x) \big\|^2_2 \Big] p_{t}(x) \diff x \\ 
	&\qquad \lesssim \bigg(\frac{16}{{\pi}}\bigg)^{d/2}\frac1n \bigg( \frac1{t} + \frac{\sigma^d}{t^{d/2+1}}\bigg)(\log n)^{d/2+1}
	 + \bigg(\frac{16}{{\pi}}\bigg)^{d/2} \frac1n \bigg( \frac1{t} + \frac{\sigma^d}{t^{d/2+1}}\bigg) (\log n)^{d/2+2} 
	+ \frac{d}{n^3t}  \\
	&\qquad \asymp \bigg(\frac{16}{{\pi}}\bigg)^{d/2} \frac1n\bigg( \frac1{t} + \frac{\sigma^d}{t^{d/2+1}} \bigg)(\log n)^{d/2+2}.
\end{align*}
This concludes the proof of Proposition~\ref{thm:score-error}.

%% file: proof-jacobian.tex

\subsection{Analysis for Jacobian estimation (proof of Proposition~\ref{thm:Jacobian-error})}\label{sub:proof_of_lemma_ref_lemma_Jacobian-error}

Similar to the proof of Proposition~\ref{thm:score-error} in Appendix~\ref{sub:proof_of_lemma_ref_lemma_score-error}, we first decompose
\begin{align}
& \bE\Big[\big\| J_{\wh{s}_{t}}(Z_t)-J_{s_t}(Z_t) \big\|\Big] = \int_{\bR^d} \bE\Big[ \big\| J_{\wh{s}_t}(x)-J_{s_t}(x) \big\| \Big] p_{t}(x) \diff x \notag\\
& \qquad = \underbrace{\int_{\cF_t} \bE\Big[ \big\| J_{\wh{s}_t}(x)-J_{s_t}(x) \big\| \ind\big\{\cE_t(x)\big\} \Big] p_{t}(x) \diff x}_{\defn \xi_1}
+ \underbrace{\int_{\cF_t^\setc} \bE\Big[ \big\| J_{\wh{s}_t}(x)-J_{s_t}(x) \big\| \ind\big\{\cE_t(x)\big\} \Big] p_{t}(x) \diff x}_{\defn \xi_2} \notag\\
& \qquad \quad + \underbrace{\int_{\bR^d} \bE\Big[ \big\| J_{\wh{s}_t}(x)-J_{s_t}(x) \big\| \ind\big\{\cE_t^\setc(x)\big\} \Big] p_{t}(x) \diff x}_{\defn \xi_3}. \label{eq:Jacobian-error-temp}
\end{align}
As a result, it suffices to control these three quantities individually.

\paragraph*{Step 1: Bounding $\xi_1$.}

To bound $\xi_1$, we begin by presenting Lemma~\ref{lemma:MSE-Jacobian-high-density}, which characterizes the spectral norm of the Jacobian estimation error $J_{\wh{s}_t}(x)$ conditional on $\cE_t(x)$.
\begin{lemma}\label{lemma:MSE-Jacobian-high-density}
For any $x\in\cF_t$, one has
	\begin{align}
		\bE\Big[ \big\|J_{\wh{s}_t}(x) -J_{s_t}(x) \big\| \ind\big\{\cE_t(x)\big\} \Big]
		& \lesssim \frac1{\sqrt{n(2\pi t)^{d/2}}} \bigg(\frac1{t} +  \frac{\|H_t(x)\|}{p_t(x)} +  \|s_t(x)\|_2^2\bigg)\frac1{\sqrt{p_{t}(x)}}
		\label{eq:MSE-Jacobian}.
	\end{align}
\end{lemma}
\begin{proof}
	See Appendix~\ref{sub:proof_of_lemma_ref_lemma_mse-Jacobian}.
\end{proof}

It follows that
\begin{align*}
	\xi_1& \lesssim \frac1{\sqrt{n(2\pi t)^{d/2}}}  \int_{\cF_t}\bigg(\frac1{t} +  \frac{\|H_t(x)\|}{p_t(x)} +  \|s_t(x)\|_2^2\bigg)\sqrt{p_t(x)} \diff x \lesssim \frac1{\sqrt{n}}\frac1{(2\pi t)^{d/4}}  \frac{\log n}{t}\int_{\cF_t}\sqrt{p_t(x)} \diff x. 
\end{align*}
where we use \eqref{eq:score-Jacobian-UB-prob-high} in Lemma~\ref{lemma:score-Jacobian-UB-prob-high} in the last line.

Further, applying the Cauchy-Schwartz inequality yields
\begin{align*}
	\int_{\cF_t}  \sqrt{p_{t}(x)} \diff x 
	& \leq \bigg(\int_{\cF_t} p_{t}(x) \diff x\bigg)^{1/2} \sqrt{|\cF_t|} \leq \sqrt{|\cF_t|} \leq (32)^{d/4}\big( t^{d/4}+\sigma^{d/2})(\log n)^{d/4} 
\end{align*}
where the last step uses \eqref{eq:high-density-set-UB} in Lemma~\ref{lemma:high-density-set-bound} and $\sqrt{a+b}\leq \sqrt{a}+\sqrt{b}$ for any $a,b>0$.

Taking the two bounds collectively results in
\begin{align}
	\xi_1 & 
\lesssim \frac1{\sqrt{n}}\frac1{(2\pi t)^{d/4}}  \frac{\log n}{t} \cdot (32)^{d/4}\big( t^{d/4}+\sigma^{d/2})(\log n)^{d/4} = \bigg(\frac{16}{\pi}\bigg)^{d/4}
	\frac1{\sqrt{n}} \bigg(\frac1t + \frac{\sigma^{d/2}}{t^{d/4+1}}\bigg) (\log n )^{d/4+1}  \label{eq:Jacobian-K1}.
\end{align}

\paragraph*{Step 2: Bounding $\xi_2$.}

As shown in \eqref{eq:Jacobian-score-est-expression}, when $\wh p_{t}(x) \leq \eta_t/2$, one has $\wh{s}_{t}(x)=0$ and $J_{\wh{s}_{t}}(x) = 0$. Hence, we can use the triangle inequality to decompose
\begin{align*}
	\xi_2 & \leq \int_{\cF_t^\setc} \Big(\bE\big[ \| J_{\wh{s}_t}(x)\big]+\|J_{s_t}(x) \| \Big) p_{t}(x) \diff x \notag\\
	&  
	= \underbrace{\int_{\cF_t^\setc} \bE\Big[ \| J_{\wh{s}_t}(x) \| \ind\big\{\wh p_{t}(x) \geq \eta_t/2 \big\} \Big] p_{t}(x) \diff x}_{\defnrev (\mathrm{I})}+\underbrace{\int_{\cF_t^\setc} \|J_{s_t}(x) \| \,p_{t}(x) \diff x}_{\defnrev (\mathrm{II})}.
\end{align*}

We start with the term (I) and focus on the event $\{\wh{p}_t(x)\geq \eta_t/2\}$. Recall the expression of the Jacobian $J_{\wh{s}_t}(x) $ of $\wh{s}_t(x)$ in \eqref{eq:Jacobian-score-est-expression}. 
When $\eta_t/2 \leq \wh{p}_t(x) \leq \eta_t$, one has
\begin{align*}
	J_{\wh{s}_t}(x) 
	& = \biggl(-\frac1t I_d + \frac{\wh{H}_t(x)}{\wh{p}_t(x)} - \frac{\wh{g}_t(x)\wh{g}_t(x)^\top}{\wh{p}_t^2(x)}\biggr) \psi\bigl(\wh{p}_{t}(x);\eta_t\bigr)  + \psi'\bigl(\wh{p}_{t}(x);\eta_t\bigr)\frac{\wh g_t(x)\wh g_t(x)^\top}{\wh p_t(x)}.
\end{align*}
It is straightforward to compute
\begin{align*}
	\psi'(x;\eta) = \exp\biggl(\frac{1-2(2x/\eta-1)}{(2x/\eta-1)(2-2x/\eta)}\biggr)\biggl\{1+\exp\biggl(\frac{1-2(2x/\eta-1)}{(2x/\eta-1)(2-2x/\eta)}\biggr)\biggr\}^{-2}\frac{2(2x/\eta)^2-6(2x/\eta)+5}{(2x/\eta-1)^2(2-2x/\eta)^2}\cdot\frac2\eta,
\end{align*}
and $$\max_{\eta/2 < x < \eta} |\psi'(x;\eta)| = \psi'(3\eta/4;\eta) = 4/\eta.$$
Combined with $\max_{\eta/2 < x < \eta}|\psi(x;\eta)|\leq 1$, this gives
\begin{align}
	\big\|J_{\wh{s}_t}(x)\big\| \leq \biggl\|-\frac1t I_d + \frac{\wh{H}_t(x)}{\wh{p}_t(x)} - \frac{\wh{g}_t(x)\wh{g}_t(x)^\top}{\wh{p}_t^2(x)}\biggr\| + \frac{4}{\eta_t}\biggl\|\frac{\wh g_t(x)\wh g_t(x)^\top}{\wh p_t(x)}\biggr\| \leq \frac1t + \frac{2}{\eta_t}\bigl\|\wh{H}_t(x)\bigr\| + \frac{12}{\eta_t^2}\bigl\|\wh g_t(x)\bigr\|_2^2, \label{eq:J-s-hat-UB}
\end{align}
where the last step holds as $\wh p_t(x)\geq \eta_t/2$.
Clearly, the above bound also holds when $\wh p_t(x) \geq \eta_t$.

As a result, we can use the triangle inequality to obtain
\begin{align*}
(\mathrm{I})\lesssim \int_{\cF_t^\setc} \bigg(\frac1t + \frac{1}{\eta_t}\bigl\|\wh{H}_t(x)\bigr\| + \frac{1}{\eta_t^2}\bigl\|\wh g_t(x)\bigr\|_2^2\bigg) p_{t}(x) \diff x \lesssim \int_{\cF_t^\setc} \bigg(\frac1t + \frac{1}{\eta_t}\bigl\|\wh{H}_t(x)\bigr\| + \|s_t(x)\|_2^2\bigg) p_{t}(x) \diff x
\end{align*}
where the last step arises from  \eqref{eq:g-hat-low-prob-UB}.
Regarding the term involving $\wh H_t$, we can apply the triangle to bound
	\begin{align*}
		\frac{1}{\eta_t}\bE\Big[\big\|\wh{H}_t(x)\big\|\Big]  &\leq \frac{1}{\eta_t}\bE\Big[\big\|\wh{H}_t(x)-H_t(x)\big\|\Big] + \frac{1}{\eta_t} \|H_t(x)\| \notag\\
		&\numpf{i}{\lesssim} \frac1{\eta_t} \bigg(\frac1{n(2\pi t)^{d/2}t} + \frac1t\sqrt{\frac{p_{t}(x)}{n(2\pi t)^{d/2}}}\bigg) + \frac{p_t(x)}{\eta_t} \frac{\|H_t(x)\|}{p_t(x)} \notag\\
		&\numpf{ii}{\leq} \frac1{\eta_t} \bigg(\frac1{n(2\pi t)^{d/2}t} + \frac1t\sqrt{\frac{c_\eta \eta_t}{n(2\pi t)^{d/2}}}\bigg)+ \frac{c_\eta \eta_t}{\eta_t} \frac{\|H_t(x)\|}{p_t(x)} \notag\\
		&\numpf{iii}{\lesssim}\frac{n(2\pi t)^{d/2}}{\log n} \bigg(\frac1{n(2\pi t)^{d/2}t} + \frac1t\sqrt{\frac{\log n}{n(2\pi t)^{d/2}}\frac{1}{n(2\pi t)^{d/2}}}\bigg) + \frac{\|H_t(x)\|}{p_t(x)} \notag\\
		& \asymp \frac1{t\sqrt{\log n}} + \frac{\|H_t(x)\|}{p_t(x)} 
	\end{align*}
	where (i) uses \eqref{eq:Phi-error-exp-UB} in Lemma~\ref{lemma:MSE-density}; (ii) $p_{t}(x)\leq c_\eta \eta_t$ on the set $\cF_t^\setc$; (iii) plugs in the values of $\eta_t$ in \eqref{eq:score-estimate}.
	Thus, we find that
\begin{align}
	(\mathrm{I}) & \lesssim \int_{\cF_t^\setc} \bigg(\frac1t + \frac{\|H_t(x)\|}{p_{t}(x)} + \|s_{t}(x)\|_2^2 \bigg) p_{t}(x)\diff x  \label{eq:Jacobian-K2-II}
\end{align}

Combined with the term (II), we arrive at
\begin{align*}
	\xi_2 \lesssim \int_{\cF_t^\setc} \bigg(\frac1t + \frac{\|H_t(x)\|}{p_{t}(x)} + \|s_{t}(x)\|_2^2 + \|J_t(x)\|  \bigg) p_{t}(x)\diff x \asymp \int_{\cF_t^\setc} \bigg(\frac1t + \|s_{t}(x)\|_2^2 + \|J_t(x)\|  \bigg) p_{t}(x)\diff x
\end{align*}
where the last step holds due to \eqref{eq:Jacobian-expression} and the triangle inequality that
\begin{align*}
	\frac{\|H_t(x)\|}{p_{t}(x)} \leq \big\|J_{s_t}(x) \big\| + \frac1t + \| s_{t}(x)\|_2^2.
\end{align*}

As a final step, invoking \eqref{eq:prob-small} in Lemma~\ref{lemma:prob-small-density} leads to
\begin{align}
	\xi_2 &\lesssim \bigg(\frac{16}{{\pi}}\bigg)^{d/2} \frac1n\bigg(\frac1t+\frac{\sigma^d}{t^{d/2+1}}\bigg) (\log n)^{d/2} +\bigg(\frac{16}{{\pi}}\bigg)^{d/2} \frac1n\bigg(\frac1t+\frac{\sigma^d}{t^{d/2+1}}\bigg) (\log n)^{d/2+2} \notag\\ 
	&\asymp \bigg(\frac{16}{{\pi}}\bigg)^{d/2} \frac1n\bigg(\frac1t+\frac{\sigma^d}{t^{d/2+1}}\bigg) (\log n)^{d/2+2} \label{eq:Jacobian-K2}.
\end{align}

\paragraph*{Step 3: Bounding $\xi_3$.}
It remains to control the quantity $\xi_3$. We can first use the Cauchy-Schwartz inequality to bound
\begin{align*}
\xi_3 & \leq \int_{\bR^d} \sqrt{\bE\big[ \| J_{\wh{s}_t}(x)-J_{s_t}(x) \|^2 \big]} \sqrt{\bP\big(\cE^\setc(x)\big)}\, p_{t}(x) \diff x \lesssim n^{-5} \int_{\bR^d} \Big(\sqrt{\bE\big[ \| J_{\wh{s}_t}(x)\|^2}\big] + \|J_{s_t}(x) \|\Big)\, p_{t}(x) \diff x,
\end{align*}
where the last step uses \eqref{eq:density-concen-ineq}, the AM-GM inequality, and $\sqrt{a+b}\leq \sqrt{a}+\sqrt{b}$ for any $a,b\geq 0$.

As shown in \eqref{eq:J-s-hat-UB}, we can bound
\begin{align*}
	\big\| J_{\wh{s}_t}(x)\big\|
	& \lesssim \frac1{t} + \frac{1}{\eta_t}\big\|\wh{H}_t(x)\big\| + \frac1{\eta_t^2} \big\| \wh{g}_t(x)\big\|_2^2 \\
	& \numpf{i}{\lesssim} \frac1t + \frac{1}{\eta_t} \frac{1}{(2\pi t)^{d/2}t} + \frac{1}{\eta_t^2} \frac1{(2\pi t)^{d}t} \\
	& \numpf{ii}{\asymp} \frac1t + \frac{n(2\pi t)^{d/2}}{\log n} \frac{1}{(2\pi t)^{d/2}t} + \frac{n^2(2\pi t)^{d}}{\log^2 n} \frac1{(2\pi t)^{d}t} \\
	& \asymp \frac{n^2}{t\log^2n},
\end{align*}
where (i) uses \eqref{eq:Sigma-hat-UB} and \eqref{eq:g-hat-UB}; (ii) arises from the choice of $\eta_t$ in \eqref{eq:score-estimate}. Hence, we find that
\begin{align}
	\int_{\bR^d} \sqrt{\bE\big[\| J_{\wh{s}_t}(x)\|^2\big]}\,p_{t}(x) \diff x 
	& \lesssim \frac{n^2}{t}. 
	\notag
\end{align}

Additionally, we know from Lemma~\ref{lemma:score-lp-norm} that
\begin{align}
	\int_{\bR^d} \big\|J_{s_t}(x)\big\|\,p_{t}(x) \diff x
	\lesssim \frac{d}{t}, 
	\notag
\end{align}

Consequently, 
putting the above two bounds collectively gives
\begin{align}
\xi_3 \lesssim \frac{1}{n^5}\bigg(\frac{n^2}{t}+\frac{d}{t}\bigg) \lesssim \frac{d}{n^3 t}.
	\label{eq:Jacobian-K3}
\end{align}

\paragraph*{Step 4: Combining bounds for $\xi_1,\xi_2,\xi_3$.}

To sum up, plugging \eqref{eq:Jacobian-K1}, \eqref{eq:Jacobian-K2}, and \eqref{eq:Jacobian-K3} into \eqref{eq:Jacobian-error-temp} leads to
\begin{align*}
	& \int_{\bR^d} \bE\Big[ \big\| J_{\wh{s}_t}(x)-J_{s_t}(x) \big\| \Big] p_{t}(x) \diff x \leq \xi_1+\xi_2+\xi_3 \\ 
	& \qquad \lesssim  \bigg(\frac{16}{{\pi}}\bigg)^{d/4} \frac1{\sqrt{n}} \bigg(\frac1t + \frac{\sigma^{d/2}}{t^{d/4+1}}\bigg) (\log n )^{d/4+1} + \bigg(\frac{16}{{\pi}}\bigg)^{d/2} \frac1n\bigg(\frac1t+\frac{\sigma^d}{t^{d/2+1}}\bigg) (\log n)^{d/2+2}
	+ \frac{d}{n^3 t} \\
	& \qquad \asymp \bigg(\frac{16}{{\pi}}\bigg)^{d/4} \frac1{\sqrt{n}} \bigg(\frac1t + \frac{\sigma^{d/2}}{t^{d/4+1}}\bigg) (\log n )^{d/4+1} + \bigg(\frac{16}{{\pi}}\bigg)^{d/2} \frac1n\bigg(\frac1t+\frac{\sigma^d}{t^{d/2+1}}\bigg) (\log n)^{d/2+2}
\end{align*}
This finishes the proof of Proposition~\ref{thm:Jacobian-error}.

%% file: proof-early-stopping.tex
\subsection{Analysis for bias in the initial distribution $p_0^\star$  (proof of Proposition \ref{thm:early-stopping})}\label{sub:proof_of_theorem_ref_thm_early_stopping}

Recall that $Z_0\sim p_0$, $Z_t\sim p_t$, and $X_0\sim p^\star_0=p^\star \ast \cN(0,\tau I_d)$.
We will establish a stronger result than Claim \eqref{eq:early-claim}: for any $0<t<1$,
\begin{align}
	\TV(p^\star,{p_t}) \leq \wt C t^{\frac \beta2 \wedge 1} \big(\log(1/t)\big)^{d/2} \label{eq:TV-strong-claim}
\end{align}
for some constant $\wt C>0$ that depends only on $d,\beta,L$, and $\sigma$.

Suppose \eqref{eq:TV-strong-claim} holds. As $X_1 \disteq \sqrt{\alpha_1}Z_{t_1}$ with $t_1 = {(1-\alpha_1)}/{\alpha_1}+\tau\asymp 1-\alpha_1+\tau < 1$, \eqref{eq:TV-strong-claim} immediately gives the advertised result in \eqref{eq:early-claim}.

Hence, the remainder of this section focuses on proving \eqref{eq:TV-strong-claim}.
To this end, we define
\begin{align}
	R_y \defn \sqrt{\beta t\log(1/t)} \quad \text{ and }\quad R_x \defn 2(\sigma\vee 1)\sqrt{\beta\log(1/t)}.
\end{align}
Notice that $R_x\geq 2R_y$ since $t\leq 1$.
Recognizing $p_t$ is the convolution of the target distribution $p^\star$ with $\cN(0,tI_d)$, we can decompose the TV distance:
\begin{align*}
\TV(p^\star,\,p_t)&=\frac12\int_{x\in\mathbb{R}^d} \bigl| p^\star(x) - p_t(x) \bigr|\diff x =\int_{x\in\mathbb{R}^d} \bigl| p^\star(x) - (p^\star \ast \varphi_t)(x) \bigr|\diff x \\
&=\frac12\int_{x\in\mathbb{R}^d} \biggl|\int_{y\in\mathbb{R}^d} \big(p^\star(x) -  p^\star(x-y)\big)\varphi_t(y)\diff y \biggr|\diff x \\
& \leq \frac12\underbrace{\int_{\|x\|_\infty\leq R_x} \biggl|\int_{y\in\mathbb{R}^d} \big(p^\star(x) -  p^\star(x-y)\big)\varphi_t(y)\diff y \biggr|\diff x}_{\defnrev(\mathrm{I})} \\ 
& \quad + \frac12\underbrace{\int_{\|y\|_\infty> R_y}\varphi_t(y)\int_{\|x\|_\infty > R_x}   \bigl|p^\star(x) - p^\star(x-y)\bigr|\diff x\diff y}_{\defnrev(\mathrm{II})} \\
&\quad +\frac12\underbrace{\int_{\|y\|_\infty\leq R_y} \varphi_t(y)\int_{\|x\|_\infty > R_x} \bigl|p^\star(x) - p^\star(x-y)\bigr|\diff x\diff y}_{\defnrev(\mathrm{III})}
\end{align*}
In what follows, we will control (I), (II), and (III) separately.

\begin{itemize}

\item We begin with term (I). Recall the multi-index notation in Assumption~\ref{assume:smooth} and $\vphi_t$ is the density of $\cN(0,tI_d)$.
Since $p^\star$ is $\beta$-H\"older-smooth, we can use Taylor's theorem to expand $p^\star(x+y)$ around $x$ up to order $\lfloor\beta\rfloor$:
\begin{align*}
p^\star(x+y) &= p^\star(x)+\sum_{1\leq |s| \leq \lfloor\beta\rfloor} \frac{1}{s!} \partial^s p^\star(x) y^s + R(x,y)
\end{align*}
where the remainder term satisfies $|R(x,y)| \leq L C_{d,\beta} \|y\|_2^{\beta}$ for some constant $C_{d,\beta}>0$ depending on $d$ and $\beta$.
Hence, we can bound
\begin{align*}
	&\int_{y\in\bR^d} \big(p^\star(x) - p^\star(x-y)\big)\vphi_t(y)\diff y 
	= \int_{y\in\bR^d} \big(p^\star(x+y) - p^\star(x)\big)\vphi_t(y)\diff y  \\
	&\qquad = \sum_{1\leq|s| \leq \lfloor\beta\rfloor} \frac{1}{s!}\partial^s p^\star(x) \int_{\bR^d}  y^s \vphi_t(y)\diff y  + LC_{d,\beta}O\bigg(\int_{\bR^d}  \|y\|_2^{\beta} \vphi_t(y)\diff y\bigg) 
\end{align*}
where the first step uses the symmetry of $\vphi_t$. Using the standard Gaussian property and the change of variable $z=y/\sqrt{t}$, one has:
\begin{align*}
	\int_{\bR^d}y^s \vphi_t(y)\diff y =t^{|s|/2} \int_{\bR^d} \prod_{i=1}^d z_i^{s_i} \vphi_1(z)\diff z \leq C_s t^{|s|/2}
\end{align*}
for some constant $C_s$ depending on $s$ if $|s|$ is even, and the integral is zero for odd $|s|$. 
Similarly, the standard Gaussian property tells us
\begin{align*}
	\int_{\bR^d}  \|y\|_2^{\beta} \vphi_t(y)\diff y & = t^{\beta/2} \int_{\bR^d}  \|z\|_2^{\beta} \vphi_1(z)\diff z = 2^{\beta/2} \frac{\Gamma(\frac{d+\beta}{2})}{\Gamma(\frac d2)} t^{\beta/2}.
\end{align*}
Combining the above observations with $t<1$ gives us
\begin{align*}
\bigg|\int_{y\in\bR^d} \big(p^\star(x) - p^\star(x-y)\big)\vphi_t(y)\diff y\bigg| \leq C_{d,\beta,L}t^{\beta/2 \wedge 1},
\end{align*}
for some constant $C_{d,\beta,L}>0$ depending on $d,\beta$, and $L$.
It follows that
\begin{align}
	(\mathrm{II}) &\leq  (2R_x)^d C_{d,\beta,L}t^{\beta/2 \wedge 1} \leq C_{d,\beta,L,\sigma}t^{\beta/2 \wedge 1} \log^{d/2}(1/t), \label{eq:early-2}
\end{align}
for some constant $C_{\sigma,d,\beta,L}>0$ depending on $d,\beta,L$, and $\sigma$.

	\item We proceed to control term (II). Since $p^\star$ is a density, we have
\begin{align*}
	(\mathrm{II}) &\leq \int_{\|y\|_\infty> R_y} \varphi_t(y)\int_{x\in\mathbb{R}^d} \bigl(p^\star(x) + p^\star(x-y)\bigr)\diff x \diff y \leq 2\int_{\|y\|_\infty> R_y} \varphi_t(y)\diff y = 2\bP\Big\{\sqrt{t}\left\|W\right\|_\infty > R_y\Big\},
\end{align*}
where $W\sim\cN(0,I_d)$ is a standard Gaussian random vector in $\bR^d$. Given $R_y/\sqrt{t} = \sqrt{\beta\log(1/t)}$, we know from the union bound and standard Gaussian tail bound that
\begin{align}
	(\mathrm{II})\leq \bP\Big\{\|W\|_\infty > \sqrt{\beta\log(1/t)} \Big\} \leq \sum_{i=1}^d \bP\big\{|W_i| >  \sqrt{\beta\log(1/t)}\big\} \leq 2d t^{\beta/2}. \label{eq:early-1}
\end{align}

\item 
Turning to (III), for any $x,y$ satisfying $\|x\|_\infty>R_x$ and $\|y\|_\infty \leq R_y$, we know from $R_x \geq 2R_y$ that $\|x\|_\infty> 2\|y\|_\infty$. Applying the triangle inequality yields
\begin{align*}
	\|x-y\|_\infty\geq \|x\|_\infty-\|y\|_\infty \geq \|x\|_\infty/2>R_x/2.
\end{align*} 
This allows us to bound
\begin{align*}
	\int_{\|x\|_\infty > R_x} \bigl|p^\star(x) - p^\star(x-y)\bigr|\diff x &\leq \int_{\|x\|_\infty > R_x} p^\star(x) \diff x + \int_{\|x\|_\infty > R_x} p^\star(x-y)\diff x \\
	& \leq \int_{\|x\|_\infty > R_x} p^\star(x) \diff x + \int_{\|x-y\|_\infty >R_x/2} p^\star(x-y)\diff x\\
	& \leq 2\,\bP\bigl\{\|Z_0\|_\infty>R_x/2\bigr\} \\
	& \numpf{i}{\leq} 2\sum_{i=1}^d\,\bP\Bigl\{|Z_0^\top e_i| >\sigma\sqrt{\beta\log(1/t)}\Bigr\} \\
	& \numpf{ii}{\leq} 2dt^{\beta/2}
\end{align*}
where (i) holds due to the choice of $R_x$; (ii) holds as $Z_0$ and $\{Z_0^\top e_i\}_{i\in[d]}$ are $\sigma$-subguassian. It follows that
\begin{align}
	(\mathrm{III}) \leq 2dt^{\beta/2}\int_{\|y\|_\infty\leq R_y} \varphi_t(y)  \diff y \leq 2dt^{\beta/2}. \label{eq:early-3}
\end{align}
\item
To finish up, combining the bounds \eqref{eq:early-2}--\eqref{eq:early-3} yields the claim in \eqref{eq:TV-strong-claim}:
\begin{align*}
\TV(p^\star,\,p_t) \lesssim C_{d,\beta,L,\sigma}t^{\beta/2 \wedge 1} \big(\log(1/t)\big)^{d/2} + dt^{\beta/2} \leq \wt C t^{\beta/2\wedge 1}  \big(\log(1/t)\big)^{d/2}.
\end{align*}
for some constant $\wt C>0$ depending on $d,\beta,L$, and $\sigma$.

\end{itemize}

%% file: proof-learning-rate.tex

\subsection{Proof of Lemma~\ref{lemma:step-size}}
\label{sec:pf-lemma:step-size}
Before beginning the proof, we note that the update rule \eqref{eq:learning-rate} implies that
\begin{align}
	\frac{1-\alpha_k}{1-\ol\alpha_{k-1}} = \frac{\ol\alpha_{k-1}-\ol\alpha_k}{\ol\alpha_{k-1}-\ol\alpha_{k}} = \frac{c_1\log K}{K}.
\end{align}
\begin{itemize}
	\item Let us begin with Claim \eqref{eq:learning-rate-3}. Recall the definition $\ol\alpha_k= \prod_{i=1}^k \alpha_i \in (0,1)$. By the update rule \eqref{eq:learning-rate}, we have
	\begin{align*}
		1<\frac{1-\ol\alpha_k}{1-\ol\alpha_{k-1}} = 1 + \frac{\ol\alpha_{k-1}-\ol\alpha_k}{1-\ol\alpha_{k-1}} = 1+ \frac{c_1 \log K}{K}\ol\alpha_{k-1} <1+ \frac{c_1 \log K}{K}, 
	\end{align*}
	where the first inequality and the last step holds due to $\alpha_k \in (0,1)$ and thus $\ol\alpha_k= \prod_{i=1}^k \alpha_i \in (0,1)$ is decreasing in $k$.
	\item Given the above observation, we further know that for all $k\geq 2$:
	\begin{align*}
		1-\alpha_k \leq \frac{1-\alpha_k}{1-\ol\alpha_k} = \frac{1-\alpha_k}{1-\ol\alpha_{k-1}} \frac{1-\ol\alpha_{k-1}}{1-\ol\alpha_k}  < \frac{1-\alpha_k}{1-\ol\alpha_{k-1}}= \frac{ c_1\log K}{K} \leq \frac12,
	\end{align*}
	where the last step holds as long as $K$ is large enough. This proves Claim \eqref{eq:alpha-t-lb}.
	
	\item In addition, we have
	\begin{align*}
		\frac{1-\alpha_k}{\alpha_k-\ol\alpha_k}=\frac1{\alpha_k}\frac{1-\alpha_k}{1-\ol\alpha_{k-1}} = \frac1{\alpha_k} \frac{c_1\log K}{ K}.
	\end{align*}
	As $1/2 < \alpha_k \leq 1$, we obtain
	\begin{align*}
		\frac{c_1\log K}{K} \leq \frac{1-\alpha_k}{\alpha_k-\ol\alpha_k} \leq \frac{2c_1\log K}{K}
	\end{align*}
	as claimed in \eqref{eq:learning-rate-4}.
	\item Finally, let us consider Claim \eqref{eq:learning-rate-6}. We first claim that $\ol\alpha_k \leq 1/2$ for all $k\geq K/2+1$. Let us prove it by contradiction. Assume the statement is false, that is, there exists some $k\geq K/2+1$ such that $\ol\alpha_k > 1/2$. Given $\ol\alpha_{k}$ is decreasing in $k$, this implies that $\ol\alpha_{k} > 1/2$ for all $1\leq k \leq K/2+1$.  In view of the update rule \eqref{eq:learning-rate}, we can deduce that for all $1\leq k \leq K/2$,
	\begin{align*}
		1-\ol\alpha_{k+1}  = (1-\ol\alpha_{k})\bigg(1+\ol\alpha_{k}\frac{c_1\log K}{K}\bigg) > (1-\ol\alpha_{k})\bigg(1+\frac{c_1\log K}{2K}\bigg).
	\end{align*}
	As $1-\ol\alpha_1 = 1-\alpha_1=K^{-c_0}$, this implies
	\begin{align*}
	1-\ol\alpha_{K/2+1} &> (1-\ol\alpha_1)\bigg(1+\frac{c_1\log K}{2K}\bigg)^{K/2} = K^{-c_0}\bigg(1+\frac{c_1\log K}{2K}\bigg)^{\frac{2K}{c_1\log K}\frac{c_1 \log K}{4}} \\ 
	& \numpf{i}{\geq} K^{-c_0} \bigg(1+\frac{1}{4}\bigg)^{c_1 \log K}=K^{-c_0+c_1\log5/4}  \numpf{ii}{>} \frac{1}{2},
	\end{align*}
	where (i) holds as $(1+1/x)^x$ is increasing in $x$ for $x>0$ and $2K/(c_1\log K)>4$ for $K$ large enough; (ii) holds as long as $c_1\geq 5c_0$. This leads to a contradiction. Thus, we know that $\ol\alpha_k \leq 1/2$ for all $k\geq K/2+1$.
	Combined with the update rule \eqref{eq:learning-rate}, this means that for all $k\geq K/2$:
	\begin{align*}
		\ol\alpha_{k+1}= \ol\alpha_{k}\bigg(1-(1-\ol\alpha_k)\frac{c_1\log K}{K}\bigg) \le \ol\alpha_{k}\bigg(1-\frac{c_1\log K}{2K}\bigg).
	\end{align*}
	Therefore, we arrive at the claim that
	\begin{align*}
	\ol\alpha_{K} \le \ol\alpha_{K/2+1}\bigg(1-\frac{c_1\log K}{2K}\bigg)^{K/2-1} \leq \frac{\frac12}{1-\frac{c_1\log K}{2K}} \bigg(1-\frac{c_1\log K}{2K}\bigg)^{\frac{2K}{c_1\log K}\frac{c_1 \log K}{4}} \leq \frac{1}{K^{c_1/4}},
	\end{align*}
    where the last step holds due to $\ol\alpha_{K/2+1} \leq 1/2$, $(c_1\log K)/K\leq 1/2$, and $(1-1/x)^{x} \leq \exp(-1)$ for $x>1$.
\end{itemize}

%% file: proof-lemma.tex

\subsection{Proof of Lemma \ref{lemma:mixing}}\label{sub:proof_of_lemma_ref_lemma_mixing}
Let us first control the KL divergence between $p_{X_{K}}$ and $p_{Y_{K}}$.
Recall that $Y_K\sim\cN(0,I_d)$ and $X_K\sim\cN(\sqrt{\ol\alpha_K}\,x_0,(1-\ol\alpha_K)I_d)$ given $X_0=x_0$. We can derive
\begin{align*}
	\KL\big(p_{X_{K}}\,\|\,p_{Y_{K}}\big) &\numpf{i}{\leq} \bE_{X_0\sim p^\star_0} \Big[ \KL\big(p_{X_{K}}(\cdot \mid X_0)\,\|\,p_{Y_{K}}(\cdot )\big)\Big] \nonumber \\
	& \numpf{ii}{=} \frac12 \bE \Big[d(1-\ol\alpha_K)-d+\big\|\sqrt{\ol\alpha_K}X_0\big\|_2^2 - d\log(1-\ol\alpha_K) \Big] \nonumber \\
	& \numpf{iii}{\leq} \frac12 \ol\alpha_K \bE\big[\| X_0\|_2^2\big] \numpf{iv}{\leq} \frac12\bE\big[\| X_0\|_2^2\big] K^{-c_1/4},
\end{align*}
where (i) uses the convexity of the KL divergence; (ii) applies the KL divergence formula for Gaussian distributions; (iii) is true as the chosen learning rate ensures that $\ol\alpha_K \leq K^{-c_1/4}\leq 1/2$ for $K$ large enough and $\log(1-x)\geq-x$ for all $x\in[0,1/2]$; (iv) uses $\ol\alpha_K \leq K^{-c_1/4}$ again.
Hence, it follows from Pinsker's inequality that
\begin{align*}
	\TV\big(p_{X_{K}},p_{Y_{K}}\big) \leq \sqrt{\frac12\KL\big(p_{X_{K}}\,\|\,p_{Y_{K}}\big)} \leq\frac12 \sqrt{\bE\big[\| X_0\|_2^2\big]}\, K^{-c_1/8}.
\end{align*}

\subsection{Proof of Lemma~\ref{lemma:MSE-density}}\label{sub:proof_of_lemma_ref_lemma_mse-density}

Recall that the training data satisfy $X^{(i)}\disteq Z_0\sim p_0= p^\star$ and the density estimator $\wh{p}_{t}(x)$ is defined in \eqref{eq:p_hat}:
\begin{align*}
 	\wh{p}_{t}(x) &\defn \frac1{n} \sum_{i=1}^n \vphi_{t}(X^{(i)}-x) \qquad \text{with}\qquad \vphi_t(x)\defn\frac1{(2\pi t)^{d/2}}\exp\bigl(-x^2/(2t)\bigr).
\end{align*} 
\paragraph*{Estimation error of $\wh{p}_t$.}
Notice that $$\bE\big[\vphi_{t}(Z_0-x)\big]\big)=\int_{y\in\bR^d} \vphi_{t}(y-x)p_0(y)\diff y=\int_{y\in\bR^d} \vphi_{t}(x-y)p_0(y)\diff y = (\vphi_{t}\ast p_0)(x) = p_{t}(x).$$ This means that
\begin{align*}
	\wh{p}_t(x) - p_{t}(x) = \frac1n \sum_{i=1}^n \big(\vphi_{t}(X^{(i)}-x) - \bE\big[\vphi_{t}(X^{(i)}-x)\big]\big)
\end{align*}
is a sum of i.i.d.~zero-mean random variables. The variance of $\wh{p}_t(x)$ can be bounded by
\begin{align}
	\var\big(\wh{p}_t(x) \big) &\leq \frac1n \bE\big[ \vphi_{t}^2(Z_0-x) \big] \notag \\
	& = \frac1n \int_{y\in\bR^d} \frac1{(2\pi t)^{d}}\exp\bigl(-{\|x-y\|_2^2}/t \bigr)p_0(y)\diff y \notag\\
	& \leq \frac1{n(2\pi t)^{d/2}} \int_{y\in\bR^d} \frac1{(2\pi t)^{d/2}}\exp\bigl(-\|x-y\|_2^2/(2t)\bigr)p_0(y)\diff y \notag\\
	& = \frac1{n(2\pi t)^{d/2}}(\vphi_{t} \ast p_0)(x) \notag\\
	& = \frac{p_{t}(x)}{n(2\pi t)^{d/2}}
	\label{eq:var-p-hat}.
\end{align}
This proves Claim \eqref{eq:MSE-pdf}.

Furthermore, we can apply the triangle inequality to get
\begin{align*}
	\max_{i\in[n]} \frac1n \big|\vphi_{t}(X^{(i)}-x) - \bE[\vphi_{t}(X^{(i)}-x)]\big|
	& \leq \frac2n  \sup_{x\in\bR^d} \vphi_{t}(x)
	 = \frac2n  \sup_{x\in\bR^d} \bigg\{\frac1{(2\pi t)^{d/2}}\exp\bigl(-\|x\|_2^2/(2t)\bigr)\bigg\} \\
	 & = \frac2{n(2\pi t)^{d/2}}.
\end{align*}
Combining this with \eqref{eq:var-p-hat}, we can invoke the Bernstein inequality to obtain
\begin{align*}
	\bP\Bigg\{\big|\wh{p}_t(x) - p_{t}(x)\big| > C_\cE \bigg( \frac{\log n}{n(2\pi t)^{d/2}}+ \sqrt{\frac{p_{t}(x)\log n}{n(2\pi t)^{d/2}}} \bigg) \Bigg\} \lesssim n^{-10}
\end{align*}
as long as $C_\cE>0$ is large enough. This proves Claim \eqref{eq:density-concen-ineq}.

\paragraph*{Estimation error of $\wh{g}_t$.}

Notice that
\begin{align*}
	\wh{g}_t(x) - g_{t}(x)&= \frac{1}{nt}\sum_{i=1}^n \Big((X^{(i)}-x)\vphi_{t}(X^{(i)}-x) - \bE \big[(Z_0-x)\vphi_{t}(Z_0-x) \big] \Big)
\end{align*}
is a sum of independent zero-mean random vectors.
This allows us to derive
\begin{align}
	\bE\Big[ \big\| \wh{g}_t(x)-g_t(x)\big\|_2^2 \Big]
	& = \frac{1}{nt^2} \bE \Big[ \big\|(X^{(i)}-x)\vphi_t(X^{(i)}-x) - \bE \big[(Z_0-x)\vphi_t(Z_0-x) \big] \big\|_2^2 \Big] \notag \\
	& = \frac{1}{nt^2} \bigg(\bE \Big[ \big\|(X^{(i)}-x)\vphi_t(X^{(i)}-x)\big\|_2^2 \Big] - \Big\|\bE \big[(Z_0-x)\vphi_t(Z_0-x) \big] \Big\|_2^2 \bigg) \notag \\
	& \leq \frac{1}{nt^2} \bE \Big[ \big\|(Z_0-x)\vphi_t(Z_0-x)\big\|_2^2 \Big] \notag \\
	& = \frac{1}{nt^2} \bE \Big[ \big\|Z_0-x\big\|_2^2\vphi_t^2(Z_0-x) \Big]
	\label{eq:g-hat-error-expression}.
\end{align}
It is straightforward to calculate
\begin{align}
\bE \Big[ \|Z_0-x\|_2^2\vphi_t^2(Z_0-x) \Big]
	& = \frac{t}{(2\pi t)^{d/2}}  \int_{\bR^d} \frac{1}{t}\|y-x\|_2^2 \frac{1}{(2\pi t)^{d/2}}\exp\bigl(-{\|y-x\|_2^2}/{t}\bigr)p_0(y)\diff y \notag \\
	& \numpf{i}{\leq} \frac{t}{(2\pi t)^{d/2}}   \int_{\bR^d} \frac{1}{(2\pi t)^{d/2}}\exp\bigl(-{\|y-x\|_2^2}/({2t})\bigr)p_0(y)\diff y \notag \\
	& = \frac{t}{(2\pi t)^{d/2}} (\vphi_t \ast p_0)(x) \notag \\
	& = \frac{t}{(2\pi t)^{d/2}}p_t(x)\label{eq:X-phi-2nd-UB},
\end{align}
where (i) holds as $x \leq \exp(x/2)$ for any $x\geq 0$.
Combining \eqref{eq:g-hat-error-expression} and \eqref{eq:X-phi-2nd-UB} establishes Claim \eqref{eq:MSE-g}.

\paragraph*{Estimation error of $\wh{H}_t$.}
Note that
\begin{align*}
	\wh{H}_t(x)-H_t(x) = \frac1{n} \sum_{i=1}^n \bigg(\frac1{t^2} (X^{(i)}-x)(X^{(i)}-x)^\top \vphi_t(X^{(i)}-x) -\frac1{t^2} \bE\big[(Z_0-x)(Z_0-x)^\top \vphi_t(Z_0-x)\big] \Big)
\end{align*}
is a sum of i.i.d.~zero-mean random matrices.
Straightforward calculation shows that
\begin{align}
	\max_{i\in[n]} \frac1n \bigg\| \frac1{t^2} (X^{(i)}-x)(X^{(i)}-x)^\top \vphi_t(X^{(i)}-x) - H_t(x) \bigg\|
	& \leq \frac1{nt^2}\max_{i\in[n]}  \|X^{(i)}-x\|_2^2 \vphi_t(X^{(i)}-x) +\frac1n\big\|H_t(x)\big\| \notag\\
	& \numpf{i}{\leq} \frac{2}{nt^2} \frac{1}{(2\pi t)^{d/2}} \sup_{y\in\bR^d} \Big\{\|y\|_2^2 \exp\bigl(-{\|y\|_2^2}/({2t})\bigr)\Big\} \notag\\
	& = \frac{4}{e} \frac{1}{n(2\pi t)^{d/2}t} \defnrev B, \label{eq:Sigma-hat-UB}
\end{align}
where (i) is true since $H_t(x) = t^{-2} \bE[(Z_0-x)(Z_0-x)^\top \vphi_t(Z_0-x)]$.
In addition, we have
\begin{align*}
	& \frac{1}{n^2t^4} \bigg\| \sum_{i=1}^n \bE \Big[ \big( (X^{(i)}-x)(X^{(i)}-x)^\top \vphi_t(X^{(i)}-x) - \bE\big[(X^{(i)}-x)(X^{(i)}-x)^\top \vphi_t(X^{(i)}-x)\big] \big)^2 \Big] \bigg\| \\
	& \qquad \numpf{i}{\leq} \frac{1}{n^2t^4} \bigg\| \sum_{i=1}^n \bE \Big[ \big( (X^{(i)}-x)(X^{(i)}-x)^\top \vphi_t(X^{(i)}-x) \big)^2 \Big] \bigg\| \\
	& \qquad = \frac1{nt^4} \bE\Big[\|Z_0-x\|_2^4 \, \vphi_t^2(Z_0-x) \Big] \\
	& \qquad = \frac1{nt^2} \int_{\bR^d} \frac1{t^2} \|y-x\|_2^4 \frac{1}{(2\pi t)^{d}} \exp\bigl(-{\|y-x\|_2^2}/{t}\bigr)p_0(y)\diff y \notag\\
	& \qquad \numpf{ii}{\leq} \frac3{nt^2} \frac{1}{(2\pi t)^{d/2}} \int_{\bR^d} \frac{1}{(2\pi t)^{d/2}} \exp\bigl(-{\|y-x\|_2^2}/({2t})\bigr)p_0(y)\diff y \\
	& \qquad = \frac{3}{n(2\pi t)^{d/2}t^2} (\vphi_t \ast p_0)(x) \\
	& \qquad = \frac{3p_t(x)}{n(2\pi t)^{d/2}t^2} \defnrev V.
\end{align*}
Here, (i) is true as $\bE\big[(X-\bE[X])^2\big] \preccurlyeq \bE[X^2]$; (ii) uses $x^2 \leq 3\exp(x/2)$ for any $x \geq 0$.
Invoking the matrix Bernstein inequality gives that for any $a>0$:
\begin{align*}
	\bP\Big\{\big\|\wh{H}_t(x)-H_t(x)\big\| > a \Big\} 
	\leq 2\exp\left( -\frac38\min\left\{\frac{a^2}{V},\frac{a}{B}\right\} \right).
\end{align*}
Finally, we can invoke Lemma~\ref{lemma:con-ineq-exp} in Appendix~\ref{sec:auxiliary_lemmas} to get
\begin{align*}
	\bE\Big[\big\|\wh{H}_t(x)-H_t(x)\big\|\Big] \lesssim B + \sqrt{V} \lesssim \frac1{n(2\pi t)^{d/2}t} + \sqrt{\frac{p_t(x)}{n(2\pi t)^{d/2}t^2}},
\end{align*}
as claimed in \eqref{eq:Phi-error-exp-UB}.

\subsection{Proof of Lemma~\ref{lemma:score-Jacobian-UB-prob-high}}\label{sub:proof_of_lemma_ref_lemma_score_jacobian_ub_prob_high}

The upper bound for the $\ell_2$ norm of the score function $s_t(x)$ has been previously established in  \cite{jiang2009general,saha2020nonparametric}. For the sake of completeness and to ensure our analysis is self-contained, we present a derivation of this bound below.

In light of the expression of the score function in \eqref{eq:score-expression}, one can derive
\begin{align}
	\|s_t(x)\|_2^2 &\leq \frac1{t^2} \bE\big[\|Z_0-x\|^2_2 \mid Z_t = x\big] \notag \\
	&\numpf{i}{\leq} \frac{2}{t}\log \bE\Big[\exp\bigl(\|Z_0-x\|_2^2/(2t)\bigr) \mid Z_t = x\Big] \notag\\
	& \numpf{ii}{=} \frac2t \log \frac{1}{(2\pi t)^{d/2}p_t(x)} \label{eq:score-UB-log}.
\end{align}
Here, (i) holds due to the concavity of $x\mapsto \log x$ and Jensen's inequality; (ii) holds as
\begin{align*}
	p_{Z_0\mid Z_t}(y\mid x) = \frac{1}{(2\pi t)^{d/2}}\exp\bigl(-\|y-x\|_2^2/(2t)\bigr) \frac{p_0(y)}{p_t(x)}
\end{align*}
and hence
\begin{align*}
	\bE\big[\exp\bigl(\|Z_0-x\|_2^2/(2t)\bigr) \mid Z_t = x\big] &= \int_{y\in\bR^d} \exp\bigl(\|y-x\|_2^2/(2t)\bigr) p_{Z_0\mid Z_t}(y\mid x) \diff y=  \frac{1}{(2\pi t)^{d/2}p_t(x)}.
\end{align*}

Therefore, for any $x$ such that $p_t(x) \geq c_\eta \eta_t = c_\eta\log n/\big(n(2\pi t)^{d/2}\big)$, one has
\begin{align*}
	\|s_t(x)\|_2^2 \leq \frac2t \log \frac{n}{c_\eta\log n} \leq \frac{2\log n}{t},
\end{align*}
where the last step holds as long as $c_\eta \geq 2$ and $n>1$. This proves Claim \eqref{eq:score-UB-prob-high}.


Next, let us consider the spectral norm of $H_t$. By the expression in \eqref{eq:Phi-p-exp}, we can bound
\begin{align}
	\frac1{p_t(x)}\big\|H_t(x)\big\| \leq \frac1{t^2} \bE\big[\|Z_0-x\|^2_2 \mid Z_t = x\big] \leq \frac2t \log \frac{1}{(2\pi t)^{d/2}p_t(x)} \label{eq:H-log}.
\end{align}
Applying the same argument as above, we conclude that for any $x$ such that $p_t(x) \geq c_\eta \eta_t$,
\begin{align*}
	\big\|H_t(x)\big\|\lesssim \frac{p_t(x)\log n}{t}
\end{align*}
provided $c_\eta \geq 2$. This establishes Claim \eqref{eq:Sigma-UB-prob-high}

Finally, let us consider the spectral norm of the Jacobian $J_t$.
In view of the expression in \eqref{eq:Jacobian-expression}, we know from \eqref{eq:score-UB-log} and \eqref{eq:H-log} that
\begin{align}
	\big\|J_{t}(x)\big\| \leq \frac1t+ \frac1{t^2} \frac{1}{p_t(x)}\big\|H_t(x)\big\| + \|s_t(x)\|_2^2 \leq \frac1t+\frac4t \log \frac{1}{(2\pi t)^{d/2}p_t(x)} \label{eq:jcb-UB-log}.
\end{align}
Therefore, for any $x$ such that $p_t(x) \geq c_\eta \eta_t$, one has
\begin{align*}
	\big\|J_{t}(x)\big\| \leq \frac1t + \frac{\log n}t \asymp \frac{\log n}{t}.
\end{align*}
as long as $c_\eta \geq 2$. Thus, we prove Claim \eqref{eq:Jacobian-UB-prob-high}.

\subsection{Proof of Lemma~\ref{lemma:MSE-score-high-density}}\label{sub:proof_of_lemma_ref_lemma_mse-score}

%

Fix an arbitrary $x\in\cF_t$. 
As shown in \eqref{eq:score-Et}, on the event $\cE_t(x)$, the score estimator becomes $\wh{s}_t(x) = \wh{g}_t(x)/\wh{p}_t(x)$.
We can then use the triangle inequality to bound
\begin{align}
	\big\|\wh{s}_t(x)-s_t(x)\big\|_2 & = \bigg\|\frac{\wh{g}_t(x)}{\wh{p}_t(x)}- \frac{g_t(x)}{p_{t}(x)} \bigg\|_2 = \bigg\| \frac{\wh{g}_t(x)-g_t(x)}{\wh{p}_t(x)} + \frac{p_{t}(x)-\wh{p}_{t}(x)}{\wh{p}_{t}(x)}\frac{g_t(x)}{p_{t}(x)} \bigg\|_2 \notag \\
	& \leq \frac{1}{\wh{p}_t(x)} \big\|\wh{g}_t(x)-g_t(x)\big\|_2
	+ \frac{1}{\wh{p}_t(x)} \big|\wh{p}_t(x)-p_{t}(x)\big| \|s_{t}(x)\|_2 \notag \\ 
	& \lesssim \frac{1}{p_t(x)} \big\|\wh{g}_t(x)-g_t(x)\big\|_2
	+ \frac{1}{p_t(x)} \big|\wh{p}_t(x)-p_{t}(x)\big| \|s_{t}(x)\|_2\label{eq:score-estimate-decomp},
\end{align}
where the last step holds as $\wh{p}_t(x)\geq p_t(x)/2$ on the event $\cE_t(x)$.

Regarding the first term, we can use \eqref{eq:MSE-g} in Lemma~\ref{lemma:MSE-density} to bound 
\begin{align}
	\frac{1}{p_t^2(x)} \bE\Big[ \big\| \wh{g}_t(x)-g_t(x)\big\|_2^2 \Big] \leq   \frac{1}{n(2\pi t)^{d/2}} \frac{1}{tp_t(x)}
	\notag.
\end{align}
As for the second term, applying \eqref{eq:MSE-pdf} in Lemma~\ref{lemma:MSE-density} yields that
\begin{align}
	\frac{1}{p_t^2(x)} \bE\Big[ \big| \wh{p}_t(x)-p_t(x)\big|^2 \Big]  \|s_t(x)\|_2^2
	& \leq  \frac{1}{n(2\pi t)^{d/2}}  \frac{\|s_t(x)\|_2^2}{p_t(x)}  
	\notag.
\end{align}
Consequently, putting the two bounds above into \eqref{eq:score-estimate-decomp} leads to 
\begin{align*}
	\bE\Big[ \big\|\wh{s}_t(x)- s_t(x) \big\|_2^2 \ind\big\{\cE_t(x)\big\} \Big]  
	\lesssim \frac{1}{n(2\pi t)^{d/2}}\bigg(\frac1t+\|s_t(x)\|_2^2\bigg) \frac{1}{p_t(x)},
\end{align*}
as claimed in \eqref{eq:MSE-score}.

\subsection{Proof of Lemma~\ref{lemma:high-density-set-bound}}\label{sub:proof_of_lemma_ref_lemma_high_density_set_bound}

Recall $\eta_t \defn \frac{\log n}{(2\pi t)^{d/2}n}$. Denote $B \defn 2 \sqrt{(\sigma^2 +t)\log n}$.
We claim that
\begin{align}
	\cF_t = \big\{x\in \bR^d \colon p_t(x) \geq c_\eta \eta_t \big\} \subset \big\{x\in \bR^d \colon \|x\|_\infty \leq B \big\}. \label{eq:high-density-set-bound-temp}
\end{align}

Suppose Claim \eqref{eq:high-density-set-bound-temp} holds. Then it follows immediately that
\begin{align*}
	|\cF_t| &\leq (2B)^d \leq 4^d (t+\sigma^2)^{d/2} (\log n)^{d/2} \leq (32)^{d/2} (t^{d/2}+\sigma^d) (\log n)^{d/2}.
\end{align*}
where the last line holds since $(x+y)^{d/2} \leq 2^{d/2}(x^{d/2}+y^{d/2})$ for any $d\geq 1$ and $x,y>0$.

Therefore, the remainder of the proof focuses on establishing Claim \eqref{eq:high-density-set-bound-temp}, which we prove by contradiction. Suppose there exists some $z\in \bR^d$ such that $p_t(z)\geq c_\eta \eta_t$ and $\|z\|_\infty > B$. Without loss of generality, we assume $|z_1|>B$.
For notational convenience, Let $Z_t^{(i)}$ denote the $i$-th coordinate of $Z_t = (Z_t^{(1)}, \dots, Z_t^{(d)})$ and $p_t^{(i)}$ denote the density of the marginal distribution of $Z_t^{(i)}$. 
Note that
\begin{align*}
	p_t(z) &= \int_{y\in \bR^d} p_0(y) \prod_{i=1}^d \bigg(\frac1{\sqrt{2\pi t}}\exp\bigl(-{(z_i-y_i)^2}/({2t})\bigr)\bigg) \diff y \notag\\
	&\leq \frac1{(2\pi t)^{(d-1)/2}}\int_{y_1\in\bR} p_0^{(1)}(y_1) \frac1{\sqrt{2\pi t}}\exp\bigl(-{(z_1-y_1)^2}/({2t})\bigr) \int_{(y_2,\dots,y_d)\in\bR^{d-1}} p_0^{(-1)}(y_2\dots,y_d\mid y_1) \diff y_2\dots\diff y_d\diff y_1 \notag\\
	& = \frac1{(2\pi t)^{(d-1)/2}} \int_{y_1\in \bR} p_0^{(1)}(y_1) \frac1{\sqrt{2\pi t}}\exp\bigl(-{(z_1-y_1)^2}/({2t})\bigr) \diff y_1
	= \frac1{(2\pi t)^{(d-1)/2}} p_t^{(1)}(z_1).
\end{align*}
This implies that
\begin{align}
	p_t^{(1)}(z_1) \geq (2\pi t)^{(d-1)/2} p_t(z) \geq \frac{c_\eta \log n}{\sqrt{2\pi t}\,n}. \label{eq:p-t-marginal-LB}
\end{align}

Now, we choose $R\defn \frac{c_\eta \sqrt{t\log n}}{2n}$. For any $y_1\in\bR$ such that $|y_1-z_1|\leq R$, one has
\begin{align}
	|y_1|\geq |z_1| - |y_1-z_1| \numpf{i}{>} B -R = 2\sqrt{(\sigma^2 +t)\log n} - \frac{c_\eta \sqrt{t\log n}}{2n} \numpf{ii}{\geq} \sqrt{2(\sigma^2 +t)\log n}, \label{eq:y1-lb}
\end{align}
where (i) arises from the assumption; (ii) is true as long as $n\geq c_\eta$.
Define the function $\vphi_t^{(1)}(x):\bR\to\bR$ by $\vphi_t^{(1)}(x)\defn(2\pi t)^{-1/2}\exp\bigl(-x^2/(2t)\bigr)$. We can derive
\begin{align}
	\Big|p_t^{(1)}(y_1)-p_t^{(1)}(z_1)\Big|
	&=\bigg|\int_{x_1\in\bR}\Big(\vphi_t^{(1)}(y_1-x_1)-\vphi_t^{(1)}(z_1-x_1)\Big)p_0^{(1)}(x_1) \diff x_1\bigg| \notag\\
	& \leq \sup_{x_1\in\bR}\Big|\vphi_t^{(1)}(y_1-x_1)-\vphi_t^{(1)}(z_1-x_1)\Big| \int_{\bR} p_0^{(1)}(x_1)\diff x_1 \notag\\
	&\leq  |y_1-z_1| \sup_{x_1\in\bR} \bigg|\frac{\mathrm d}{\mathrm d x_{1}}\vphi_t^{(1)}(x_1)\bigg|  \notag\\
	& =  \frac{R}{\sqrt{2\pi t}} \frac1t \sup_{x_1\in\bR} \Big\{|x_1|\exp\bigl(-{x_1^2}/({2t})\bigr)\Big\} \notag\\
	&\leq \frac{R}{\sqrt{2\pi}\, t} = \frac{c_\eta \sqrt{\log n}}{2\sqrt{2\pi t}\,n} \leq p_t^{(1)}(z_1),
\end{align}
where the last line holds due to the choice of $R$ and \eqref{eq:p-t-marginal-LB}. This implies that for any $y_1\in\bR$ such that $|y_1-z_1|\leq R$:
\begin{align}
	p_t^{(1)}(z_1)\geq \frac12 p_t^{(1)}(y_1). \label{eq:p-z1-y1-lb}
\end{align}

Since $Z_0$ is $\sigma$-subgaussian, we know that $Z_t$ is $\sqrt{\sigma^2+t}$-subgaussian and $Z_t^{(1)}=Z_t^\top e_1$ is $\sqrt{\sigma^2+t}$-subgaussian. Hence, we can use the definition in Assumption~\ref{assume:subgaussian} to deduce
\begin{align*}
	2 &\geq \int_{y_1\in\bR} \exp\biggl(\frac{y^2_1}{\sigma^2+t}\biggr)p_t^{(1)}(y_1)\diff y_1 \geq \int_{y_1\colon|y_1-z_1|\leq R} \exp\biggl(\frac{y_1^2}{\sigma^2+t}\biggr)p_t^{(1)}(y_1)\diff y_1 \notag\\
	& \numpf{i}{\geq} 2R\exp\biggl(\frac{(B-R)^2}{\sigma^2+t}\biggr) \frac12 p_t^{(1)}(z_1)  \\
	& \numpf{ii}{\geq} R n^{2} p_t^{(1)}(z_1) \numpf{iii}{\geq} \frac{c_\eta \sqrt{t\log n}}{2n} n^{2}  \frac{c_\eta \log n}{n\sqrt{2\pi t}} = \frac{c_\eta^2}{2\sqrt{2\pi}} \log^{3/2} n,
\end{align*}
where (i) uses \eqref{eq:y1-lb} and \eqref{eq:p-z1-y1-lb}; (ii) arises from \eqref{eq:y1-lb}; (iii) follows from $R= \frac{c_\eta \sqrt{t\log n}}{2n}$ and \eqref{eq:p-t-marginal-LB}. This leads to contradiction for $n$ large enough, and thus proves Claim \eqref{eq:high-density-set-bound-temp}.

\subsection{Proof of Lemma \ref{lemma:prob-small-density}}\label{sub:proof_of_lemma_ref_lemma_prob_small_density}

\paragraph*{Proof of Claim \eqref{eq:prob-small-density}.}


Set $\cB\defn \big\{x\in\bR^d\colon \big\|x-\bE[Z_t]\big\|_\infty \leq B\big\}$ with $B \defn 2\sqrt{(\sigma^2+t)\log n}$. 
\begin{align*}
	\bP\big\{p_t(Z_t) < c_\eta \eta_t \big\} &= \bP\big\{p_t(Z_t) < c_\eta \eta_t,\,Z_t\in \cB\big\} + \bP\big\{p_t(Z_t) < c_\eta \eta_t,\,Z_t\notin \cB\big\} \\
	& \leq \bP\big\{p_t(Z_t) < c_\eta \eta_t,\,Z_t\in \cB\big\} + \bP\big\{Z_t\notin \cB\big\}.
\end{align*}
Notice that
\begin{align}
	\bP\big\{p_t(Z_t) < c_\eta \eta_t,\,Z_t\in \cB\big\} &= \int_{\cB} \ind\big\{p_t(x) < c_\eta \eta_t\big\} p_t(x) \diff x \leq c_\eta \eta_t|\cB|  \notag \\
	&\leq 4^d c_\eta \eta_t (\sigma^2+t)^{d/2} (\log n)^{d/2} = c_\eta \bigg(\frac{8}{\pi}\bigg)^{d/2} \frac1n \frac{(\sigma^2+t)^{d/2}}{t^{d/2}} (\log n)^{d/2+1} \notag \\
	& \lesssim \Big(\frac{16}{\pi}\Big)^{d/2} \frac1n \bigg(1+\frac{\sigma^d}{t^{d/2}}\bigg) (\log n)^{d/2+1},
	\label{eq:prob-small-density-temp}
\end{align}
where we use $c_\eta \asymp 1$,  $\eta_t=\frac{\log n}{n(2\pi t)^{d/2}}$ and $(x+y)^{d/2} \leq 2^{d/2}(x^{d/2}+y^{d/2})$ for any $d\geq 1$ and $x,y>0$.

In addition, since $Z_t^{(i)}\defn Z_t^\top e_i$ is $\sqrt{\sigma^2+t}$-subgaussian for all $i\in[d]$, applying the union bound gives
\begin{align}
	\bP\big\{Z_t\notin \cB\big\} \leq \sum_{i=1}^d\bP\Big\{\big|Z_t^{(i)}-\bE[Z_t^{(i)}]\big| > B\Big\}\leq 2d\exp\biggl(-\frac{B^2}{2(\sigma^2+t)}\biggr) \leq \frac{2d}{n^2}.\label{eq:X-t-notin-B}
\end{align}
Collecting these two bounds together yields
\begin{align*}
	\bP\big\{p_t(Z_t) < c_\eta \eta_t \big\}
	& \lesssim \Big(\frac{16}{\pi}\Big)^{d/2} \frac1n \bigg(1+\frac{\sigma^d}{t^{d/2}}\bigg) (\log n)^{d/2+1} + \frac{d}{n^2} 
	\asymp \Big(\frac{16}{\pi}\Big)^{d/2} \frac1n \bigg(1+\frac{\sigma^d}{t^{d/2}}\bigg) (\log n)^{d/2+1}.
\end{align*}

\paragraph*{Proof of Claim \eqref{eq:prob-small-score}.}

Claim \eqref{eq:prob-small-score} can be established an approach similar to that developed in \citet[Lemma 5]{wibisono2024optimal}.
For the sake of clarity and completeness, we present the full proof below.
Recall $\cB\defn \big\{x\in\bR^d\colon \big\|x-\bE[Z_t]\big\|_\infty \leq B\big\}$ with $B \defn 4\sqrt{(\sigma^2+t)\log n}$. We can derive
\begin{align*}
	&\bE\Big[\|s_t(Z_t)\|_2^2\,\ind\big\{p_t(Z_t) < c_\eta \eta_t \big\} \Big] \notag\\ 
	&\qquad= \bE\Big[\|s_t(Z_t)\|_2^2\,\ind\big\{p_t(Z_t) < c_\eta \eta_t \big\}\ind\big\{Z_t\in\cB\big\} \Big] + \bE\Big[\|s_t(Z_t)\|_2^2\,\ind\big\{p_t(Z_t) < c_\eta \eta_t \big\}\ind\big\{Z_t\notin\cB\big\} \Big] \\
	&\qquad \leq \bE\Big[\|s_t(Z_t)\|_2^2\,\ind\big\{p_t(Z_t) < c_\eta \eta_t \big\}\ind\big\{Z_t\in\cB\big\} \Big] + \bE\Big[\|s_t(Z_t)\|_2^2\,\ind\big\{Z_t\notin\cB\big\} \Big].
\end{align*}


When $p_t(x) < c_\eta \eta_t = \frac{c_\eta\log n}{n(2\pi t)^{d/2}}$, one has $(2\pi t)^{d/2}p_t(x)\leq (c_\eta \log n)/ n \leq 1/e$ for $n$ large enough.
Thus, we can bound
\begin{align}
	\bE\Big[\|s_t(Z_t)\|_2^2\,\ind\big\{p_t(Z_t) < c_\eta \eta_t \big\}\ind\big\{Z_t\in\cB\big\} \Big]
	&=\int_{\cB} \ind\big\{p_t(x) < c_\eta \eta_t\big\}\|s_t(x)\|_2^2\, p_t(x) \diff x \notag\\
	& \numpf{i}{\lesssim} \int_{\cB} \ind\big\{p_t(x) < c_\eta \eta_t\big\} p_t(x) \frac1t \log \frac{1}{(2\pi t)^{d/2}p_t(x)} \diff x \notag\\
	& \numpf{ii}{\leq}  \int_{\cB} \ind\big\{p_t(x) < c_\eta \eta_t\big\} c_\eta \eta_t \frac1t \log \frac{1}{(2\pi t)^{d/2}c_\eta \eta_t} \diff x \notag\\
	& \numpf{iii}{\lesssim}  \frac{c_\eta \eta_t |\cB|}t\log n \notag\\
	& \numpf{iv}{\asymp} \Big(\frac{16}{\pi}\Big)^{d/2} \frac1{nt} \bigg(1+\frac{\sigma^d}{t^{d/2}}\bigg) (\log n)^{d/2+2}
	\label{eq:prob-small-score-term1}
\end{align}
Here, (i) uses \eqref{eq:score-UB-log} from Lemma~\ref{lemma:score-Jacobian-UB-prob-high}; (ii) holds as $x\mapsto x\log(1/x)$ is increasing on $(0,1/e)$; (iii) holds $(2\pi t)^{d/2}c_\eta \eta_t=(c_\eta \log n) /n$; (iv) arises from \eqref{eq:prob-small-density-temp}.

Meanwhile, we can apply the Cauchy-Schwartz inequality to bound the second term as
\begin{align}
	\bE\Big[\|s_t(Z_t)\|_2^2\,\ind\big\{Z_t\notin\cB\big\} \Big] \leq \sqrt{\bE\big[\|s_t(Z_t)\|_2^4\big]}\sqrt{\bP\{Z_t\notin \cB\}} \lesssim \frac{d}{t} \frac{\sqrt{d}}{n} = \frac{d^{3/2}}{nt},
	\label{eq:prob-small-score-term2}
\end{align}
where we use \eqref{eq:score-lp-norm} in Lemma~\ref{lemma:score-lp-norm} and \eqref{eq:X-t-notin-B}.

Combining \eqref{eq:prob-small-score-term1} and \eqref{eq:prob-small-score-term2} gives
\begin{align*}
	\bE\Big[\|s_t(Z_t)\|_2^2\,\ind\big\{p_t(Z_t) < c_\eta \eta_t \big\} \Big]
	& \lesssim \Big(\frac{16}{\pi}\Big)^{d/2} \frac1{nt} \bigg(1+\frac{\sigma^d}{t^{d/2}}\bigg) (\log n)^{d/2+2} + \frac{d^{3/2}}{nt} \\ 
	& \asymp \Big(\frac{16}{\pi}\Big)^{d/2} \frac1{nt} \bigg(1+\frac{\sigma^d}{t^{d/2}}\bigg) (\log n)^{d/2+2}.
\end{align*}
This finishes the proof of \eqref{eq:prob-small-score}.

\paragraph*{Proof of Claim \eqref{eq:prob-small-jacobian}.}

We first decompose
\begin{align*}
	&\bE\Big[\big\|J_t(Z_t)\big\|\,\ind\big\{p_t(Z_t) < c_\eta \eta_t \big\} \Big] \\
	&\qquad = \bE\Big[\big\|J_t(Z_t)\big\|\,\ind\big\{p_t(Z_t) < c_\eta \eta_t \big\}\ind\big\{Z_t\in\cB\big\} \Big] + \bE\Big[\big\|J_t(Z_t)\big\|\,\ind\big\{p_t(Z_t) < c_\eta \eta_t \big\}\ind\big\{Z_t\notin\cB\big\} \Big] \\
	&\qquad \leq \bE\Big[\big\|J_t(Z_t)\big\|\,\ind\big\{p_t(Z_t) < c_\eta \eta_t \big\}\ind\big\{Z_t\in\cB\big\} \Big] + \bE\Big[\big\|J_t(Z_t)\big\|\,\ind\big\{Z_t\notin\cB\big\} \Big].
\end{align*}

By \eqref{eq:jcb-UB-log}, we can derive
\begin{align}
	\big\|J_t(Z_t)\big\|\leq \frac1t+\frac4t \log \frac{1}{(2\pi t)^{d/2}p_t(x)} \leq \frac5t \log \frac{1}{(2\pi t)^{d/2}p_t(x)},
\end{align}
where the last step is true true because $(2\pi t)^{d/2}p_t(x)\leq (c_\eta \log n)/ n \leq 1/e$ when $p_t(x) < c_\eta \eta_t$ and $n$ is sufficiently large.
Hence, applying the same argument for \eqref{eq:prob-small-score-term1}, one can derive
\begin{align}
	\bE\Big[\big\|J_t(Z_t)\big\|\,\ind\big\{p_t(Z_t) < c_\eta \eta_t \big\}\ind\big\{Z_t\in\cB\big\} \Big] &  =\int_{\cB} \ind\big\{p_t(x) < c_\eta \eta_t\big\}\big\|J_t(x)\big\|\, p_t(x) \diff x \notag\\
	&  \lesssim  \int_{\cB} \ind\big\{p_t(x) < c_\eta \eta_t\big\} p_t(x) \frac1t \log \frac{1}{(2\pi t)^{d/2}p_t(x)} \diff x \notag\\
	&  \lesssim \Big(\frac{16}{\pi}\Big)^{d/2} \frac1{nt} \bigg(1+\frac{\sigma^d}{t^{d/2}}\bigg) (\log n)^{d/2+1}
	\label{eq:prob-small-jcb-term1},
\end{align}
where the last step follows from \eqref{eq:prob-small-score-term1}.

In addition, invoking the Cauchy-Schwartz inequality leads to
\begin{align}
	\bE\Big[\big\|J_t(Z_t)\big\|\,\ind\big\{Z_t\notin\cB\big\} \Big] \leq \sqrt{\bE\big[\|J_t(Z_t)\|^2\big]}\sqrt{\bP\{Z_t\notin \cB\}} \lesssim \frac{d}{t} \frac{\sqrt{d}}{n} = \frac{d^{3/2}}{nt}
	\label{eq:prob-small-jcb-term2}
\end{align}
where the second inequality results from \eqref{eq:Jacobian-lp-norm} in Lemma~\ref{lemma:score-lp-norm} and \eqref{eq:X-t-notin-B}.

Putting \eqref{eq:prob-small-jcb-term1} and \eqref{eq:prob-small-jcb-term2} finishes the proof of Claim \eqref{eq:prob-small-jacobian}.

\subsection{Proof of Lemma~\ref{lemma:MSE-Jacobian-high-density}}\label{sub:proof_of_lemma_ref_lemma_mse-Jacobian}

Fix an arbitrary $x\in\cF_t$. As shown in the proof of Theorem~\ref{thm:score-error} in Appendix~\ref{sub:proof_of_lemma_ref_lemma_score-error}, on the event $\cE_t(x)$, one has $\wh{p}_t(x)\geq p_t/2> \eta_t$ (see \eqref{eq:p-hat-LB}). Hence, as shown in \eqref{eq:J-1}, the score estimator equals $\wh{s}_t(x)=\wh{g}_t(x)/\wh{p}_t(x)$ and the Jacobian of $\wh{s}_t(x)$ is equal to
\begin{align*}
	J_{\wh{s}_t}(x) = -\frac1t I_d+ \frac{\wh{H}_t(x)}{\wh{p}_t(x)} - \wh{s}_t(x)\wh{s}_t(x)^\top = -\frac1t I_d+ \frac{\wh{H}_t(x)}{\wh{p}_t(x)} - \frac{\wh{g}_t(x)\wh{g}_t(x)^\top}{\wh{p}_t^2(x)}.
\end{align*}
Combining this with the expression of $J_t(x)$ in \eqref{eq:Jacobian-expression}, one can express
\begin{align}
	J_{\wh{s}_t}(x) - J_{t}(x) &= -\frac1t I_d+ \frac{\wh{H}_t(x)}{\wh{p}_t(x)} - \wh{s}_t(x)\wh{s}_t(x)^\top -J_t(x) \notag\\
	& =\frac{\wh{H}_t(x)}{\wh{p}_t(x)} - \frac{H_t(x)}{p_t(x)}+ \wh{s}_t(x)\wh{s}_t(x)^\top - s_t(x)s_t(x)^\top \notag\\
	& = \underbrace{\frac{\wh{H}_t(x)-H_t(x)}{\wh{p}_t(x)}}_{\defnrev \theta_1} 
	+ \underbrace{\frac{p_t(x)-\wh{p}_t(x)}{\wh{p}_t(x)}\frac{H_t(x)}{p_t(x)}}_{\defnrev \theta_2} 
	+ \underbrace{\big(\wh{s}_t(x)-s_t(x)\big) \big(\wh{s}_t(x)+s_t(x)\big)}_{\defnrev \theta_3} 
	\label{eq:error-Jacobian} 
\end{align}
In what follows, we control the spectral norms of $\theta_1$, $\theta_2$, and $\theta_3$ individually.

\begin{itemize}
	\item 
We start with the term $\theta_1$.
\begin{align}
	\bE\Big[\|\theta_1\|\ind\big\{\cE_t(x)\big\}\Big] &\numpf{i}{\leq} \frac{2}{p_t(x)} \bE\Big[\big\|\wh{H}_t(x)-H_t(x)\big\|\Big] \notag \\
	& \numpf{ii}{\lesssim} \frac1{n(2\pi t)^{d/2}}\frac1{tp_t(x)} + \frac1{\sqrt{n(2\pi t)^{d/2}}}\frac1{t\sqrt{p_{t}(x)}} \notag \\
	& \numpf{iii}{\asymp} \frac1{\sqrt{n(2\pi t)^{d/2}}}\frac1{t\sqrt{p_{t}(x)}} \label{eq:MSE-Jcb-term1}
\end{align}
where (i) is true since $\wh{p}_t(x)\geq p_t(x)/2$; (ii) applies \eqref{eq:Phi-error-exp-UB} in Lemma~\ref{lemma:MSE-density}; (iii) holds as $n(2\pi t)^{d/2}p_t(x) < 1$ when $x\in\cF_t$.

\item Next, let us consider the term $\theta_2$, which can be bounded by
\begin{align}
	\bE\Big[\|\theta_2\|\ind\big\{\cE_t(x)\big\}\Big] &\numpf{i}{\leq} \frac{2\,\|H_t(x)\|}{p^2_t(x)} \sqrt{\bE\big[ |\wh{p}_t(x)-p_t(x)|^2 \big]} \numpf{ii}{\lesssim} \frac{\|H_t(x)\|}{p_t^2(x)} \sqrt{\frac{p_{t}(x)}{n(2\pi t)^{d/2}}} \notag\\
	&= \frac{1}{\sqrt{n(2\pi t)^{d/2}}} \frac{\|H_t(x)\|}{p_t^{3/2}(x)} , \label{eq:MSE-Jcb-term2}
\end{align}
where (i) holds due to $\wh{p}_t(x)\geq p_t(x)/2$ and Jensen's inequality;  (ii) arises from \eqref{eq:MSE-pdf} in Lemma~\ref{lemma:MSE-density}.


\item It remains to control $\theta_3$.
Applying the triangle inequality shows that
\begin{align*}
	\|\theta_3\| 
	& \leq \big\|\wh{s}_t(x)-s_t(x)\big\|_2 \Big(\big\|\wh{s}_t(x)-s_t(x)\big\|_2+2\|s_t(x)\|_2 \Big) \notag \\
	& \lesssim \big\|\wh{s}_t(x)-s_t(x)\big\|_2^2 + \big\|\wh{s}_t(x)-s_t(x)\big\|_2\big\|s_t(x)\big\|_2.
\end{align*}
Thus, one can bound
\begin{align}
	\bE \Big[\| \theta_3\| \ind\big\{\cE_t(x)\big\} \Big] &  \lesssim \bE\Big[\big\|\wh{s}_t(x)-s_t(x)\big\|_2^2\ind\big\{\cE_t(x)\big\}\Big] + \bE\Big[\big\|\wh{s}_t(x)-s_t(x)\big\|_2\ind\big\{\cE_t(x)\big\}\Big]\|s_t(x)\|_2 \notag \\
	&  \numpf{i}{\lesssim} \bE\Big[\big\|\wh{s}_t(x)-s_t(x)\big\|_2^2\ind\big\{\cE_t(x)\big\}\Big] + \sqrt{\bE\Big[\big\|\wh{s}_t(x)-s_t(x)\big\|_2^2\ind\big\{\cE_t(x)\big\}\Big]}\|s_t(x)\|_2 \notag\\
	& \numpf{ii}{\lesssim} \frac{1}{n(2\pi t)^{d/2}} \frac{1}{tp_t(x)} + \frac{1}{n(2\pi t)^{d/2}}  \frac{\|s_t(x)\|_2^2}{p_t(x)} \notag\\
	& \quad + \frac{1}{\sqrt{n(2\pi t)^{d/2}}} \frac{\|s_t(x)\|_2}{\sqrt{tp_t(x)}} + \frac{1}{\sqrt{n(2\pi t)^{d/2}}}  \frac{\|s_t(x)\|_2^2}{\sqrt{p_t(x)}} \notag\\
	& \numpf{iii}{\asymp} \frac{1}{n(2\pi t)^{d/2}} \frac{1}{tp_t(x)}+\frac{1}{\sqrt{n(2\pi t)^{d/2}}} \frac{\|s_t(x)\|_2}{\sqrt{tp_t(x)}} + \frac{1}{\sqrt{n(2\pi t)^{d/2}}}  \frac{\|s_t(x)\|_2^2}{\sqrt{p_t(x)}}
	\label{eq:MSE-Jcb-term3}.
\end{align}
Here, (i) uses the Cauchy-Schwartz inequality; (ii) invokes \eqref{eq:MSE-score} in Lemma~\ref{lemma:MSE-score-high-density}; (iii) is true because $n(2\pi t)^{d/2}p_t(x) < 1$ when $x\in\cF_t$.

\item
Taking \eqref{eq:MSE-Jcb-term1}, \eqref{eq:MSE-Jcb-term2}, and \eqref{eq:MSE-Jcb-term3} collectively, we conclude that
\begin{align*}
&\bE\Big[\|J_{\wh{s}_t}(x) - J_{t}(x)\|\ind\big\{\cE_t(x)\big\} \Big]
\leq \bE\Big[ \big(\| \theta_1\| + \|\theta_2\|+\|\theta_3\| \big)\ind\big\{\cE_t(x)\big\} \Big] \notag\\
& \qquad\lesssim \frac1{\sqrt{n(2\pi t)^{d/2}}}\frac1{t\sqrt{p_{t}(x)}} + \frac{1}{\sqrt{n(2\pi t)^{d/2}}} \frac{\|H_t(x)\|}{p_t^{3/2}(x)} \\ 
& \qquad \quad + \frac{1}{n(2\pi t)^{d/2}} \frac{1}{tp_t(x)} +\frac{1}{\sqrt{n(2\pi t)^{d/2}}} \frac{\|s_t(x)\|_2}{\sqrt{tp_t(x)}} + \frac{1}{\sqrt{n(2\pi t)^{d/2}}}  \frac{\|s_t(x)\|_2^2}{\sqrt{p_t(x)}}\\
	& \qquad \asymp \frac1{\sqrt{n(2\pi t)^{d/2}}} \bigg(\frac1{t\sqrt{p_{t}(x)}} + \frac{\|H_t(x)\|}{p_t^{3/2}(x)} +  \frac{\|s_t(x)\|_2^2}{\sqrt{p_t(x)}}\bigg)
\end{align*}
where the last step uses the fact that $n(2\pi t)^{d/2}p_t(x) < 1$ when $x\in\cF_t$ and the AM-GM inequality that
\begin{align*}
	\frac{1}{t\sqrt{p_t(x)}} + \frac{\|s_t(x)\|_2^2}{\sqrt{p_t(x)}}
	\geq \sqrt{2}\frac{\|s_t(x)\|_2}{\sqrt{tp_t(x)}}.
\end{align*}

This finishes the proof of Lemma~\ref{lemma:MSE-Jacobian-high-density}.

\end{itemize}

%% file: aux-lemma.tex
\section{Auxiliary lemmas}\label{sec:auxiliary_lemmas}
\begin{lemma}\label{lemma:con-ineq-exp}
	Suppose a random variable $X$ satisfies
	\begin{align*}
		\bP\{|X| > t\} \leq 2\exp\left( -\frac38\left\{\frac{t^2}{V}\wedge \frac{t}{B}\right\} \right).
	\end{align*}
 Then the expectation $\bE\big[|X|\big]$ can be bounded by
	\begin{align*}
		\bE\big[|X|\big] \lesssim \sqrt{V} + B.
	\end{align*}
\end{lemma}
\begin{proof}
Denote $c = 3/8$.
We can use the expectation formula for a nonnegative random variable to write
\begin{align*}
\mathbb{E}[|X|] &= \int_{0}^{\infty} \mathbb{P}\{|X| > t\} \diff t \lesssim \int_{0}^{\infty} \exp\biggl(-c\biggl\{\frac{t^2}{V} \wedge \frac{t}{B}\biggr\}\biggr) \diff t \\ 
& = \int_{0}^{V/B} \exp\bigl(-c{t^2}/{V}\bigr) \diff t + \int_{V/B}^{\infty} \exp\bigl(-c{t}/{B}\bigr) \diff t.
\end{align*}

For the first integral, we can derive
\begin{align*}
\int_{0}^{V/B} \exp\bigl(-c{t^2}/{V}\bigr) \diff t &= \sqrt{\frac Vc}\int_{0}^{\sqrt{cV}/B} \exp\bigl(-u^2\bigr) \diff u \leq \sqrt{\frac Vc}\int_{0}^{\infty} \exp\bigl(-u^2\bigr) \diff u= \sqrt{\frac Vc} \frac{\sqrt{\pi}}2 \lesssim \sqrt{V}.
\end{align*}

For the second integral, straightforward calculation yields
\begin{align*}
\int_{V/B}^{\infty} \exp(-c{t}/{B}) \diff t &= \frac Bc \int_{cV/B^2}^{\infty} \exp(-u) \diff u \leq  \frac Bc \int_{0}^{\infty} \exp(-u) \diff u = \frac Bc \lesssim B.
\end{align*}

Combining the two bounds gives the desired bound.
\end{proof}

\begin{lemma}\label{lemma:score-lp-norm}
Given a random vector $Z_0$ in $\bR^d$, consider $Z_t=Z_0+\sqrt{t}\,W$ where $W\sim\cN(0,I_d)$ is a standard Gaussian random vector in $\bR^d$ that is independent of $Z_0$. Let $s_t(x)\defn\gd\log p_{Z_t}(x)$ denote the score function of $Z_t$ and let $J_{t}(x)$ denote the Jacobian of $s_t(x)$.
Then for any integer $p\geq 1$, one has
	\begin{align}
		\bE\Big[\big\|s_t(Z_t)\big\|_2^p\Big] &\leq C_p \bigg(\frac{d}{t}\bigg)^{p/2}, \label{eq:score-lp-norm}\\
		\bE\Big[\big\|J_{t}(Z_t)\big\|^p\Big] & \leq C_p' \bigg(\frac{d}{t}\bigg)^{p} \label{eq:Jacobian-lp-norm},
	\end{align}
	where $C_p,C_p'>0$ are some constants that only depend on $p$. 
\end{lemma}


\begin{proof}
Let us start with Claim \eqref{eq:score-lp-norm}. Recall the expression of $s_t(x)$ in \eqref{eq:score-expression}, we have
	\begin{align*}
	\bE\Big[\big\|s_t(Z_t)\big\|_2^p\Big] &\numpf{i}{=} \frac1{t^p} \bE\Big[ \big\| \bE[Z_t - Z_0 \mid Z_t] \big\|_2^p \Big] \numpf{ii}{\leq}\frac1{t^p}\bE\Big[ \bE\big[\| Z_t - Z_0 \|_2^p \mid Z_t] \Big]\notag\\
	& \numpf{iii}{=} \frac1{t^p} \bE\big[ \| Z_t - Z_0 \|_2^p \big]  \numpf{iv}{\leq} C_p \bigg(\frac{d}{t}\bigg)^{p/2},
	\end{align*}
	for some constant $C_p>0$ that only depends on $p$.
Here, (i) uses the expression of $s_t$ in \eqref{eq:score-expression}; (ii) is due to the convexity of $x\mapsto\|x\|_2^p$ and Jensen's inequality; (iii) holds due to the tower property; (iv) is due to $Z_t-Z_0\sim\cN(0,tI_d)$ and the standard Gaussian property (see \citet[Lemma 8]{li2024provable}). This proves Claim~\eqref{eq:score-lp-norm}.

Next, we move on to consider Claim \eqref{eq:Jacobian-lp-norm}. Recalling the expression of $J_t(x)$ in \eqref{eq:Jacobian}, we have
\begin{align*}
	\bE\Big[\big\|J_t(Z_t)\big\|^p\Big] &\leq \bE\Bigg[\bigg(\frac1t + \frac{1}{t^2} \big\|\bE\big[(Z_t-Z_0)(Z_t-Z_0)^\top \mid Z_t]\big\| + \|s_t(x)\|_2^2 \bigg)^p\Bigg] \\
	& \leq 3^{p-1} \bigg(\frac1{t^p} + \frac{1}{t^{2p}} \bE\Big[ \big\|\bE\big[(Z_t-Z_0)(Z_t-Z_0)^\top \mid Z_t]\big\|^p \Big] + \bE\Big[\|s_t(x)\|_2^{2p}\Big] \bigg)
\end{align*}
Similarly to the above derivation, we can bound
\begin{align*}
	\bE\Big[ \big\|\bE\big[(Z_t-Z_0)(Z_t-Z_0)^\top \mid Z_t]\big\|^p \Big] \leq \bE\Big[ \big\|Z_t-Z_0\big\|^{2p}_2\Big] \leq C_{2p} (td)^p 
\end{align*}
where (i) holds due to the convexity of $x\mapsto\|x\|_2^p$, Jensen's inequality, and the tower property; (ii) uses the fact that $Z_t-Z_0\sim\cN(0,tI_d)$ and the standard Gaussian property.
Combined with \eqref{eq:score-lp-norm}, we conclude that
\begin{align*}
	\bE\Big[\big\|J_{s_t}(Z_t)\big\|^p\Big] 
	\leq 3^{p-1} \Bigg(\frac1{t^p} + C_{2p} \frac{(td)^p}{t^{2p}} + C_{2p} \bigg(\frac dt\bigg)^p \Bigg) \leq C_p' (\frac dt\bigg)^p.
\end{align*}
for some constant $C_p'$ that only depends on $p$.
This completes the proof of Claim \eqref{eq:Jacobian-lp-norm}.
\end{proof}

\begin{lemma}\label{lemma:score-error-X-Z}
	Let $X, Z$ be random vectors in $\bR^d$ and $a\in\bR$ be a constant. Suppose $X\disteq aZ$. Then their score functions and Jacobian matrices satisfy
	\begin{align*}
		s_{X}(x) &= \frac{1}as_Z(x/a),\\ 
		J_{s_{X}}(x) &= \frac{1}{a^2}J_{s_Z}(x/a).
	\end{align*}
\end{lemma}

\begin{proof}

 Straightforward computation shows that $p_{X}(x) = p_{Z}(x/a)/|a|^d$. This gives
\begin{align*}
	s_{X}(x) = \gd\log p_{X}(x) = \gd\log \frac{p_{Z}(x/a)}{|a|^d} = \gd\log p_{Z}(x/a) = \frac{1}as_Z(x/a),
\end{align*}
which further leads to
\begin{align*}
	J_{s_{X}}(x) = \frac{\partial}{\partial x} s_{X}(x) = \frac{\partial}{\partial x} \frac{s_Z(x/a)}a = \frac1{a^2} J_{s_{Z}}(x/a).
\end{align*}


	
\end{proof}